\documentclass[10pt]{article}
\usepackage{geometry}
\usepackage{amsmath,amsthm,amssymb,amscd}
\usepackage{latexsym}
\usepackage{indentfirst}
\usepackage{subfigure}
\usepackage{graphicx}
\usepackage{extarrows}
\usepackage{bm}
\usepackage{multirow}
\usepackage{algorithm}
\usepackage{algpseudocode}
\usepackage{setspace}
\usepackage{stackengine}
\usepackage{braket}
\usepackage{mathtools}
\usepackage{physics}
\usepackage{authblk}
\newtheorem{definition}{Definition}
\newtheorem{theorem}{Theorem}
\newtheorem{lemma}{Lemma}
\newtheorem{remark}{Remark}

\newtheorem{proposition}{Proposition}
\newtheorem{corollary}{Corollary}

\newtheorem{problem}{Problem}

\usepackage{color,xcolor}

\usepackage{fncylab}
\labelformat{enumii}{\theenumi(#1)}
\begin{document}

\title{
  \bf Determination of action model equivalence and simplification of action model
}

\author[1]{Jingwei Li}
\affil[1]{Shanghai Center for Systems Biomedicine, Shanghai Jiao Tong University, Shanghai 200240, China, (Email:ljw2017@sjtu.edu.cn).}

\baselineskip=0pt

%\pacs{}
%
%\keywords{}

\maketitle

\begin{abstract}
\baselineskip=0pt

Kripke models are useful to express static knowledge or belief. On the other hand, action models describe information flow and dynamic knowledge or belief. In this paper, we study two problems: determining action model equivalence and minimizing the event space of an action model under certain structural relationships.

The Kripke model equivalence is perfectly caught by the structural relationship called bisimulation. In this paper, we propose the generalized action emulation perfectly catching the action model equivalence. Previous structural relationships sufficient for the action model equivalence, {\it i.e.} the bisimulation, the propositional action emulation, the action emulation, and the action emulation of canonical action models, can be described by various restricted versions of the generalized action emulation. We summarize four critical properties of the atom set over preconditions, and prove that any formula set satisfying these properties can be used to restrict the generalized action emulation to determine the action model equivalence by an iteration algorithm. We also construct a new formula set with these four properties, which is generally more efficient than the atom set.

The technique of the partition refinement has been used to minimize the world space of a Kripke model under the bisimulation. Applying the partition refinement to action models allows one to minimize their event spaces under the bisimulation. The propositional action emulation is weaker than bismulation but still sufficient for the action model equivalence. We prove that it is PSPACE-complete to minimize the event space of an action model under the propositional action emulation, and provide a PSPACE algorithm for it. Finally, we prove that minimize the event space under the action model equivalence is PSPACE-hard, and propose a computable method based on the canonical formulas of modal logics to solve this problem.
\end{abstract}

\section{Intrtoduction}

Kripke models offer semantic interpretions of modal logics, and two Kripke models are semantically equivalent iff there exists a bisimulation between them \cite{BlackburnCUP2001}. Kripke models can be updated by action models which describe information flow \cite{BaltagS2016}. Two action models are considered equivalent if they have semantically equivalent updating effects over all Kripke models. A trival generalization of the bisimulation from Kripke models to action models offers a sufficient but not necessary condition for the action model equivalence. The propositional action emulation is proposed in \cite{EijckS2012}, which is necessary for the action model bisimulation, while still sufficient but not necessary for the action model equivalence. Moreover, when restricted to propositional action models (those with only propositional preconditions), the propositional action emulation is sufficient and necessary for the equivalence. In \cite{EijckS2012}, the authors show that two action models have the same semantic effect on their canonical Kripke model \cite{BlackburnCUP2001} iff they have the same semantic effect on any Kripke model. This for the first time, gives a sufficient and necessary condition for the action model equivalence. An equivalent restatement without explicitly introducing the canonical Kripke model is proposed as the parametrized action emulation in \cite{EijckS2012}. The atom set over preconditions can be used to transform any action model to its equivalent canonical version \cite{SietsmaEijckJOPL2013}. Moreover, it has been proved that two action models are equivalent iff their canonical versions satisfy the action emulation \cite{SietsmaEijckJOPL2013}, which further simplifies the sufficient and necessary determination of the action model equivalence. We introduce the so-called generalized action emulation, and prove that it is sufficient and necessary for the action model equivalence. The generalized action emulation has several restricted forms, which are equivalent to the bisimulation, the propositional action emulation, the action emulation, and the action emulation of canonical action models (or equivalently, the same semantic effect over canonical Kripke model or the parametrized action emulation) \cite{SietsmaEijckJOPL2013}. After a deep study of the atom set, we find that only four critical properties are responsible for the ability to determine the action model equivalence. In fact, any formula set satisfying the four critical properties can be used to restrict the generalized action emulation to get a sufficient and necessary condition for the action model equivalence. We design an iteration algorithm to determine the action model equivalence based on such formula sets. We also construct a new formula set, which is generally more efficient than the atom set in the iteration algorithm.

The problem of finding the minimal Kripke model bisimulating a given Kripke model has been solved in \cite{EijckCWI2004} by the technique of the partition refinement \cite{PaigeSJOC1987}. It is similar to the problem of minimizing the number of states in a finite automaton \cite{JohnHopcroftAP1971}. On the other hand, how to minimize the event space of an action model under other structural relationships (say the propositional action emulation) or even the action model equivalence are still open problems. Resembling the partition refinement algorithm of Kripke models \cite{PaigeSJOC1987,EijckCWI2004}, one may easily minimize the event space of an action model under the bisimulation. A more difficult problem is to minimize the event space of an action model under the propositional action emulation. In this paper, we prove that this problem is actually PSPACE-complete, and propose a PSPACE algorithm for it. An even more difficult problem is to minimize the event space of an action model under the action model equivalence. We prove that this problem is PSPACE-hard, and propose a computable method for it by the technique of the canonical formulas of modal logics \cite{Moss2007JOPL}.

This paper is arranged as follows. In Section \ref{SecPreRes}, we revise some necessary definitions, general assumptions and classical results related to Kripke and action models. In Section \ref{SecGAE}, we define the generalized action emulation and the iteration algorithm used to determine the action model equivalence. In Section \ref{SecNFS}, we construct the formula set used to speed up the iteration algorithm. In Section \ref{SecAlg}, we propose the PSPACE algorithm to minimize the event space of an action model under propositional action emulation. In Section \ref{SecMTE}, we give the methology to minimize the event space of an action model under the action model equivalence. In Section \ref{SecNP}, we discuss the complexity of the minimizing problems. Finally, conclusions and discussions come in Section \ref{SecConc}.

\section{Preliminary definitions and results}
\label{SecPreRes}

\begin{definition}
Let $A$ be a finite label set and $\mathcal{P}$ be an enumerable proposition set.
\end{definition}

\begin{definition}[Kripke model]
Define a Kripke model $\mathcal{M}\coloneqq(W^\mathcal{M}, Val^\mathcal{M}, \to^\mathcal{M}, W_0^\mathcal{M})$, where $W^\mathcal{M}$ is a world set; $Val^\mathcal{M}:W^\mathcal{M}\mapsto 2^\mathcal{P}$; $\to^\mathcal{M}:A\mapsto 2^{W^\mathcal{M}\times W^\mathcal{M}}$ defines an ordered relation $\to_a^\mathcal{M}\subset W^\mathcal{M}\times W^\mathcal{M}$ for each $a\in A$; $W_0^\mathcal{M}\subset W^\mathcal{M}$ is the set of actual worlds.
\end{definition}

\begin{definition}[modal language]
Define the modal language $\mathcal{L}$ as
\begin{align*}
\phi\Coloneqq p \mid \lnot\phi \mid \phi\lor\phi \mid \lozenge_a\phi,
\end{align*}
where $p\in \mathcal{P}$, $a\in A$. We employ the abbreviations: $\phi\land\psi$ for $\lnot(\phi\lor\psi)$, $\phi\rightarrow\psi$ for $\lnot\phi\lor\psi$, $\square_a\phi$ for $\lnot\lozenge_a\lnot\phi$. $\forall \phi\in\mathcal{L}$, define the depth $\delta(\phi)$ of $\phi$ as the maxmial number of nested $\lozenge$ in $\phi$.
\end{definition}

\begin{definition}
Let $\mathcal{M}$ be a Kripke model. $\forall w\in W$, $\forall p\in \mathcal{P}$, $\forall \phi,\phi_1,\phi_2\in\mathcal{L}$,
\begin{enumerate}
\item $\mathcal{M}\vDash_w p$ iff $p\in Val^\mathcal{M}(w)$.
\item $\mathcal{M}\vDash_w\lnot\phi$ iff $\mathcal{M}\nvDash_w\phi$.
\item $\mathcal{M}\vDash_w\phi_1\lor\phi_2$ iff $\mathcal{M}\vDash_w\phi_1$ or $\mathcal{M}\vDash_w\phi_2$.
\item $\mathcal{M}\vDash_w\lozenge_a\phi$ iff $\exists w'\in \to_a^\mathcal{M}(w,\cdot)$ ({\it i.e.} $w\to_a^\mathcal{M}w'$), $\mathcal{M}\vDash_{w'}\phi$.
\end{enumerate}
\end{definition}

\begin{definition}[action model]
Define an action model $\mathcal{A}\coloneqq (E^\mathcal{A}, Pre^\mathcal{A}, \to^\mathcal{A}, E_0^\mathcal{A})$, where $E^\mathcal{A}$ is an event set; $Pre^\mathcal{A}:E^\mathcal{A}\mapsto \mathcal{L}$; $\to^\mathcal{A}:A\mapsto 2^{E^\mathcal{A}\times E^\mathcal{A}}$ defines an ordered relation $\to_a^\mathcal{A}\subset E^\mathcal{A}\times E^\mathcal{A}$ for each $a\in A$; $E_0^\mathcal{A}\subset E^\mathcal{A}$ is the set of actual events.
\end{definition}

\begin{definition}\label{DefUpdate}
Let $\mathcal{M}$ be a Kripke model and $\mathcal{A}$ be an action model. Define the Kripke model $\mathcal{M}\otimes\mathcal{A}$ as
\begin{enumerate}
\item $W^{\mathcal{M}\otimes\mathcal{A}}\coloneqq \set{(w,x) \mid w\in W^\mathcal{M}, x\in E^\mathcal{A}, \mathcal{M}\vDash_w Pre^\mathcal{A}(x)}$.
\item $(w,x)\to_a^{\mathcal{M}\otimes\mathcal{A}} (v,y)$ iff $w\to_a^\mathcal{M} v$ and $x\to_a^\mathcal{A} y$.
\item $Val^{\mathcal{M}\otimes\mathcal{A}}(w,x)\coloneqq Val^\mathcal{M}(w)$.
\item $W_0^{\mathcal{M}\otimes\mathcal{A}}\coloneqq \set{(w,x)\in W^{\mathcal{M}\otimes\mathcal{A}} \mid w\in W_0^\mathcal{M}, x\in E_0^\mathcal{A}}$.
\end{enumerate}
\end{definition}

\begin{definition}[Kripke bisimulation]
Let $\mathcal{M}$ and $\mathcal{N}$ be Kripke models. A relation $R\subset W^\mathcal{M}\times W^\mathcal{N}$ is a bisimulation if $\forall w\in W^\mathcal{M}$, $\forall v\in W^\mathcal{N}$, $R(w,v)$ implies the following:
\begin{description}
\item[Invariance] $Val^\mathcal{M}(w)=Val^\mathcal{N}(v)$,
\item[Zig] $\forall a\in A$, $\forall w'\in \to_a^\mathcal{M}(w,\cdot)$, $\exists v'\in \to_a^\mathcal{N}(v, \cdot)$, $R(w', v')$,
\item[Zag] $\forall a\in A$, $\forall v'\in \to_a^\mathcal{N}(v,\cdot)$, $\exists w'\in \to_a^\mathcal{M}(w, \cdot)$, $R(w', v')$.
\end{description}
We denote $\mathcal{M}\underline{\leftrightarrow}\mathcal{N}$ iff there exists a bisimulation $R$ between $\mathcal{M}$ and $\mathcal{N}$ such that
\begin{description}
\item[Zig0] $\forall w\in W_0^\mathcal{M}$, $\exists v\in W_0^\mathcal{N}$, $R(w,v)$,
\item[Zag0] $\forall v\in W_0^\mathcal{N}$, $\exists w\in W_0^\mathcal{M}$, $R(w,v)$.
\end{description}
\end{definition}

\begin{definition}[action equivalence]
We say that two actions models $\mathcal{A}$ and $\mathcal{B}$ are equivalent (denoted by $\mathcal{A}\equiv\mathcal{B}$) iff for any Kripke model $\mathcal{M}$, $\mathcal{M}\otimes\mathcal{A}\underline{\leftrightarrow}\mathcal{M}\otimes\mathcal{B}$.
\end{definition}

\begin{definition}[action bisimulation]\label{DefBiSimAct}
Let $\mathcal{A}$ and $\mathcal{B}$ be action models. A relation $S\subset E^\mathcal{A}\times E^\mathcal{B}$ is a bisimulation if $\forall x\in E^\mathcal{A}$, $\forall y\in E^\mathcal{B}$, $S(x,y)$ implies the following:
\begin{description}
\item[Invariance] $Pre^\mathcal{A}(x)\equiv Pre^\mathcal{B}(y)$,
\item[Zig] $\forall a\in A$, $\forall x'\in \to_a^\mathcal{A}(x,\cdot)$, if $Pre^\mathcal{A}(x)\land\lozenge_a Pre^\mathcal{A}(x')\not\equiv\bot$, then $\exists y'\in \to_a^\mathcal{B}(y, \cdot)$, $S(x', y')$,
\item[Zag] $\forall a\in A$, $\forall y'\in \to_a^\mathcal{B}(y,\cdot)$, if $Pre^\mathcal{A}(y)\land\lozenge_a Pre^\mathcal{A}(y')\not\equiv\bot$, then $\exists x'\in \to_a^\mathcal{A}(x, \cdot)$, $S(x', y')$.
\end{description}
We denote $\mathcal{A}\underline{\leftrightarrow}\mathcal{B}$ iff there exists a bisimulation $S$ between $\mathcal{A}$ and $\mathcal{B}$ such that
\begin{description}
\item[Zig0] $\forall x\in E_0^\mathcal{A}$, $\exists y\in E_0^\mathcal{B}$, $S(x,y)$,
\item[Zag0] $\forall y\in E_0^\mathcal{B}$, $\exists x\in E_0^\mathcal{A}$, $S(x,y)$.
\end{description}
\end{definition}

\begin{definition}[propositional action emulation, Definition 18 in \cite{EijckS2012}]\label{DefProEmu}
Let $\mathcal{A}$ and $\mathcal{B}$ be action models. A relation $S\subset E^\mathcal{A}\times E^\mathcal{B}$ is a propositional acition emulation if $\forall x\in E^\mathcal{A}$, $\forall y\in E^\mathcal{B}$, $S(x,y)$ implies the following:
\begin{description}
\item[Consistency] $Pre^\mathcal{A}(x)\land Pre^\mathcal{B}(y)\not\equiv\bot$.
\item[Zig] $\forall a\in A$, $\forall x'\in \to_a^\mathcal{A}(x,\cdot)$,
\begin{align*}
Pre^\mathcal{A}(x')\vDash \bigvee_{y'\in \to_a^\mathcal{B}(y,\cdot)\cap S(x',\cdot)}Pre^\mathcal{B}(y'),
\end{align*}
where $S(x',\cdot)\coloneqq\set{y'\in E^\mathcal{B} \mid S(x',y')}$.
\item[Zag] $\forall a\in A$, $\forall y'\in \to_a^\mathcal{B}(y,\cdot)$,
\begin{align*}
Pre^\mathcal{B}(y')\vDash \bigvee_{x'\in \to_a^\mathcal{A}(x,\cdot)\cap S(\cdot, y')}Pre^\mathcal{A}(x').
\end{align*}
\end{description}
We denote $\mathcal{A}\leftrightarrows_P\mathcal{B}$ iff there exists a propositional action emulation $S$ between $\mathcal{A}$ and $\mathcal{B}$ such that
\begin{description}
\item[Zig0] $\forall x\in E_0^\mathcal{A}$, $Pre^\mathcal{A}(x)\vDash\bigvee_{y\in S(x, \cdot)\cap E_0^\mathcal{B}}Pre^\mathcal{B}(y)$,
\item[Zag0] $\forall y\in E_0^\mathcal{B}$, $Pre^\mathcal{B}(y)\vDash\bigvee_{x\in S(\cdot, y)\cap E_0^\mathcal{A}}Pre^\mathcal{A}(x)$.
\end{description}
\end{definition}

\begin{definition}[action emulation, Definition 14 in \cite{SietsmaEijckJOPL2013}]\label{DefEmu}
Let $\mathcal{A}$ and $\mathcal{B}$ be action models. A relation $S\subset E^\mathcal{A}\times E^\mathcal{B}$ is an acition emulation if $\forall x\in E^\mathcal{A}$, $\forall y\in E^\mathcal{B}$, $S(x,y)$ implies the following:
\begin{description}
\item[Consistency] $Pre^\mathcal{A}(x)\land Pre^\mathcal{B}(y)\not\equiv\bot$.
\item[Zig] $\forall a\in A$, $\forall x'\in \to_a^\mathcal{A}(x,\cdot)$,
\begin{align*}
Pre^\mathcal{A}(x)\land Pre^\mathcal{B}(y)\vDash \square_a\qty(Pre^\mathcal{A}(x')\to\bigvee_{y'\in \to_a^\mathcal{B}(y,\cdot)\cap S(x',\cdot)}Pre^\mathcal{B}(y')).
\end{align*}
\item[Zag] $\forall a\in A$, $\forall y'\in \to_a^\mathcal{B}(y,\cdot)$,
\begin{align*}
Pre^\mathcal{A}(x)\land Pre^\mathcal{B}(y)\vDash \square_a\qty(Pre^\mathcal{B}(y')\to \bigvee_{x'\in \to_a^\mathcal{A}(x,\cdot)\cap S(\cdot, y')}Pre^\mathcal{A}(x')).
\end{align*}
\end{description}
We denote $\mathcal{A}\leftrightarrows\mathcal{B}$ iff there exists an action emulation $S$ between $\mathcal{A}$ and $\mathcal{B}$ such that
\begin{description}
\item[Zig0] $\forall x\in E_0^\mathcal{A}$, $Pre^\mathcal{A}(x)\vDash\bigvee_{y\in S(x, \cdot)\cap E_0^\mathcal{B}}Pre^\mathcal{B}(y)$,
\item[Zag0] $\forall y\in E_0^\mathcal{B}$, $Pre^\mathcal{B}(y)\vDash\bigvee_{x\in S(\cdot, y)\cap E_0^\mathcal{A}}Pre^\mathcal{A}(x)$.
\end{description}
\end{definition}

\begin{definition}[propositional action model, Definition 16 in \cite{EijckS2012}]
An action model is propositional if every precondition formula that occurs in it is a formula of classical propositional logic.
\end{definition}

\begin{definition}[single negation, formula closure, atom \cite{SietsmaEijckJOPL2013}]
$\forall \phi\in \mathcal{L}$, define the single negation of $\phi$ as follows: if $\phi=\lnot\psi$, then the single negation of $\phi$ is $\psi$; otherwise, it is $\lnot\phi$. $\forall \Phi\subset\mathcal{L}$, define the closure $C(\Phi)$ of $\Phi$ as the smallest formula set containing $\Phi$ that is closed under taking subformulas
and single negations, and the atoms $K\circ C(\Phi)$ of $\Phi$ as the set of maximal consistent subsets of $C(\Phi)$. Define $\Gamma\circ K\circ C(\Phi)\coloneqq\set{\bigwedge\Psi \mid \Psi\in K\circ C(\Phi)}$.
\end{definition}

\begin{definition}[canonical Kripke model, Definition 7 in \cite{SietsmaEijckJOPL2013}]
Let $\Phi\subset\mathcal{L}$ and $\set{p_{\Xi} \mid \Xi\in K\circ C\qty(\Phi)}$ be a set of unique propositions not occurring in $\Phi$. Then the canonical Kripke model $\mathcal{M}^c$ over $\Phi$ is defined as
\begin{enumerate}
\item $W^{\mathcal{M}^c}\coloneqq K\circ C\qty(\Phi)$.
\item $\forall \Xi\in W^{\mathcal{M}^c}$, $Val^{\mathcal{M}^c}\qty(\Xi)\coloneqq \qty(\mathcal{P}\cap\Xi)\cup\set{p_{\Xi}}$.
\item $\forall a\in A$, $\forall \Xi,\Xi'\in W^{\mathcal{M}^c}$, $\Xi\to_a^{\mathcal{M}^c} \Xi'$ iff $\bigwedge\Xi\land\lozenge_a\bigwedge\Xi'\not\equiv\bot$.
\item $W_0^{\mathcal{M}^c}\coloneqq K\circ C\qty(\Phi)$.
\end{enumerate}
\end{definition}

\begin{definition}\label{DefDiscre}
Let $\Phi\subset\mathcal{L}$ and $\xi\in\mathcal{L}$. Define $\mathcal{G}\qty(\xi,\Phi)\coloneqq \set{\phi\in\Phi \mid \phi\vDash\xi}$.
\end{definition}

\begin{definition}[parametrized action emulation, Definition 12 in \cite{SietsmaEijckJOPL2013}]
Let $\mathcal{A}$ and $\mathcal{B}$ be action models and
\begin{align*}
\Theta\coloneqq \Gamma\circ K\circ C\qty(Pre^\mathcal{A}\qty(E^\mathcal{A})\cup Pre^\mathcal{B}\qty(E^\mathcal{B})).
\end{align*}
A set of indexed relations $\set{S_\xi \mid \xi\in\Theta}$ with $S_\xi\subset E^\mathcal{A}\times E^\mathcal{B}$ is a parametrized action emulation if $\forall x\in E^\mathcal{A}$, $\forall y\in E^\mathcal{B}$, $\forall\xi\in\Theta$, $S_\xi(x,y)$ implies the following:
\begin{description}
\item[Invariance] $\xi\vDash Pre^\mathcal{A}(x)\land Pre^\mathcal{B}(y)$,
\item[Zig] $\forall a\in A$, $\forall x'\in \to_a^\mathcal{A}(x,\cdot)$, $\forall \xi'\in\mathcal{G}\qty(Pre^\mathcal{A}(x'),\Theta)$, if $\xi\land\lozenge_a \xi'\not\equiv\bot$, then $\exists y'\in \to_a^\mathcal{B}(y, \cdot)$, $S_{\xi'}(x', y')$,
\item[Zag] $\forall a\in A$, $\forall y'\in \to_a^\mathcal{B}(y,\cdot)$, $\forall \xi'\in\mathcal{G}\qty(Pre^\mathcal{B}(y'),\Theta)$, if $\xi\land\lozenge_a \xi'\not\equiv\bot$, then $\exists x'\in \to_a^\mathcal{A}(x, \cdot)$, $S_{\xi'}(x', y')$.
\end{description}
We denote $\mathcal{A}\leftrightarrows_S\mathcal{B}$ iff there exists a parametrized action emulation $\set{S_\xi \mid \xi\in\Theta}$ between $\mathcal{A}$ and $\mathcal{B}$ such that
\begin{description}
\item[Zig0] $\forall x\in E_0^\mathcal{A}$, $\forall \xi\in\mathcal{G}\qty(Pre^\mathcal{A}(x),\Theta)$, $\exists y\in E_0^\mathcal{B}$, $S_\xi(x,y)$,
\item[Zag0] $\forall y\in E_0^\mathcal{B}$, $\forall \xi\in\mathcal{G}\qty(Pre^\mathcal{B}(y),\Theta)$, $\exists x\in E_0^\mathcal{A}$, $S_\xi(x,y)$.
\end{description}
\end{definition}

\begin{definition}[caninical action model, Definition 13 in \cite{SietsmaEijckJOPL2013}]
Let $\mathcal{A}$ be an action model. Define the canonical version $\mathcal{A}^c$ of $\mathcal{A}$ as follows.
\begin{enumerate}
\item $E^{\mathcal{A}^c}\coloneqq\set{(x,\phi) \mid x\in E^\mathcal{A}, \phi\in \Gamma\circ K\circ C\circ Pre^\mathcal{A}\qty(E^\mathcal{A}), \phi\vDash Pre^\mathcal{A}(x)}$.
\item $\forall (x,\phi)\in E^{\mathcal{A}^c}$, $Pre^{\mathcal{A}^c}\qty(x,\phi)\coloneqq\phi$.
\item $\forall (x,\phi),(x,\phi')\in E^{\mathcal{A}^c}$, $\forall a\in A$, $(x,\phi)\to_a^{\mathcal{A}^c}(x',\phi')$ iff $x\to_a^\mathcal{A}x'$ and $\phi\land\lozenge_a\phi'\not\equiv\bot$.
\item $E_0^{\mathcal{A}^c}\coloneqq\set{(x,\phi)\in E^{\mathcal{A}^c}\mid x\in E_0^\mathcal{A}}$.
\end{enumerate}
\end{definition}

\begin{definition}[generated submodel, resembling Definition 2.5 in \cite{BlackburnCUP2001}]\label{DefGenSubMod}
Let $\mathcal{A}$ and $\mathcal{B}$ be action models. We say that $\mathcal{B}$ is a submodel of $\mathcal{A}$ iff 
\begin{enumerate}
\item $E^\mathcal{B}\subset E^\mathcal{A}$;
\item $Pre^\mathcal{B}$, $\to^\mathcal{B}$, $E_0^\mathcal{B}$ are the restrictions of $Pre^\mathcal{A}$, $\to^\mathcal{A}$, $E_0^\mathcal{A}$ on $E^\mathcal{B}$ respectively.
\end{enumerate}
We say that $\mathcal{B}$ is a generated submodel of $\mathcal{A}$ iff
\begin{enumerate}
\item $\mathcal{B}$ is a submodel of $\mathcal{A}$;
\item $E^\mathcal{B}$ is a closed subset of $E^\mathcal{A}$ upon $\to^\mathcal{A}$. That is, $\forall x,y\in E^\mathcal{A}$, $\forall a\in A$, if $x\in E^\mathcal{B}$, $x\to_a^\mathcal{A} y$ and $Pre^\mathcal{A}(x)\land\lozenge_a Pre^\mathcal{A}(y)$, then $y\in E^\mathcal{B}$.
\end{enumerate}
Let $E\subset E^\mathcal{A}$. We say $\mathcal{A}_0$ is the submodel of $\mathcal{A}$ generated by $E$ iff $\mathcal{A}_0$ is the minimal generated submodel of $\mathcal{A}$ such that $E\subset E^{\mathcal{A}_0}$ (denoted by $\mathcal{A}_0=\mathbb{G}(\mathcal{A},E)$).
\end{definition}

We give the following classical results without proofs.
\begin{proposition}
The Kripke bisimulation $\underline{\leftrightarrow}$ is an equivalence relationship (reflexivity, symmetry, transitivity) over Kripke models.
\end{proposition}

\begin{proposition}\label{ProBisProActEquEqu}
The action bisimulation $\underline{\leftrightarrow}$, the propositional action emulation $\leftrightarrows_P$, the action emulation $\leftrightarrows$, the parametrized action emulation $\leftrightarrows_S$, and the action equivalence $\equiv$ are equivalence relationships over action models.
\end{proposition}

\begin{proposition}\label{ProBisImpProImpActImpEqu}
Let $\mathcal{A}$ and $\mathcal{B}$ be action models. Then
\begin{align*}
\mathcal{A}\underline{\leftrightarrow}\mathcal{B}\Rightarrow \mathcal{A}\leftrightarrows_P\mathcal{B}\Rightarrow \mathcal{A}\leftrightarrows\mathcal{B}\Rightarrow \mathcal{A}\leftrightarrows_S\mathcal{B}\Leftrightarrow \mathcal{A}\equiv\mathcal{B}.
\end{align*}
\end{proposition}

\begin{proposition}[Theorem 8 in \cite{SietsmaEijckJOPL2013}]
Let $\mathcal{A}$ and $\mathcal{B}$ be propositional action models. Then $\mathcal{A}\equiv \mathcal{B}$ implies $\mathcal{A}\leftrightarrows_P \mathcal{B}$.
\end{proposition}

\begin{proposition}[Theorem 4 in \cite{SietsmaEijckJOPL2013}]
Let $\mathcal{A}$ be an action model and $\mathcal{A}^c$ be the canonical version of $\mathcal{A}$. Then $\mathcal{A}\equiv\mathcal{A}^c$.
\end{proposition}

\begin{proposition}[Theorems 2, 3, 6 in \cite{SietsmaEijckJOPL2013}]\label{ProEquivImpCanEmu}
Let $\mathcal{A}$ and $\mathcal{B}$ be action models, $\mathcal{M}^c$ be the canonical Kripke model over $Pre^\mathcal{A}\qty(E^\mathcal{A})\cup Pre^\mathcal{B}\qty(E^\mathcal{B})$, and $\mathcal{A}^c$ and $\mathcal{B}^c$ be the canonical versions of $\mathcal{A}$ and $\mathcal{B}$, respectively. Then
\begin{align*}
\mathcal{M}^c\otimes \mathcal{A}\underline{\leftrightarrow}\mathcal{M}^c\otimes \mathcal{B}\Leftrightarrow \mathcal{A}\leftrightarrows_S\mathcal{B}\Leftrightarrow \mathcal{A}^c\leftrightarrows \mathcal{B}^c \Leftrightarrow \mathcal{A}\equiv\mathcal{B}.
\end{align*}
\end{proposition}

\begin{proposition}[generated submodel inducing bisimulation, resembling Proposition 2.19 in \cite{BlackburnCUP2001}]\label{ProGenSubIndBis}
Let $\mathcal{A}$ be an action model. Then $\mathcal{A}\underline{\leftrightarrow}\mathbb{G}\qty(\mathcal{A},E_0^\mathcal{A})$.
\end{proposition}

\begin{proposition}\label{ProBijGKC}
Let $\Phi,\Psi\subset \mathcal{L}$. Define $\iota:K\circ C(\Phi)\times K\circ C(\Psi)\mapsto\mathcal{L}$ as $\iota(\Phi',\Psi')\coloneqq \Phi'\cup\Psi'$ for $(\Phi',\Psi')\in K\circ C(\Phi)\times K\circ C(\Psi)$. Then $\iota$ is a bijection from $\set{(\Phi',\Psi')\in K\circ C(\Phi)\times K\circ C(\Psi) \mid \bigwedge\Phi'\land\bigwedge\Psi'\not\equiv\bot}$ to $K\circ C(\Phi\cup\Psi)$. 
\end{proposition}

\section{Generalized action emulation}
\label{SecGAE}

\begin{definition}[generalized action emulation]\label{DefGenActEmu}
Let $\mathcal{A}$ and $\mathcal{B}$ be action models. A map $\eta:E^\mathcal{A}\times E^\mathcal{B}\mapsto \mathcal{L}$ is a generalized action emulation if $\forall (x,y)\in E^\mathcal{A}\times E^\mathcal{B}$, $\eta(x,y)$ satisfies the following:
\begin{description}
\item[Zig] $\forall a\in A$, $\forall x'\in\to_a^\mathcal{A}(x,\cdot)$,
\begin{align}
\eta(x,y)\vDash \square_a \qty(Pre^\mathcal{A}(x')\to\bigvee_{y'\in\to_a^\mathcal{B}(y,\cdot)}\qty(Pre^\mathcal{B}(y')\land\eta(x',y'))).
\end{align}

\item[Zag] $\forall a\in A$, $\forall y'\in\to_a^\mathcal{B}(y,\cdot)$,
\begin{align}
\eta(x,y)\vDash \square_a \qty(Pre^\mathcal{B}(y')\to\bigvee_{x'\in\to_a^\mathcal{A}(x,\cdot)}\qty(Pre^\mathcal{B}(x')\land\eta(x',y'))).
\end{align}
\end{description}
We denote $\mathcal{A}\leftrightarrows_G\mathcal{B}$ iff there exists a generalized action emulation $\eta$ between $\mathcal{A}$ and $\mathcal{B}$ such that
\begin{description}
\item[Zig0] $\forall x\in E_0^\mathcal{A}$, $Pre^\mathcal{A}(x)\vDash\bigvee_{y\in E_0^\mathcal{B}}\qty(Pre^\mathcal{B}(y)\land\eta(x,y))$,

\item[Zag0] $\forall y\in E_0^\mathcal{B}$, $Pre^\mathcal{B}(y)\vDash\bigvee_{x\in E_0^\mathcal{A}}\qty(Pre^\mathcal{A}(x)\land\eta(x,y))$.
\end{description}
\end{definition}

\begin{remark}\label{RemResToGen}
Let $\mathcal{A}$ and $\mathcal{B}$ be action models, $\mathcal{M}^c$ be the canonical Kripke model over $Pre^\mathcal{A}\qty(E^\mathcal{A})\cup Pre^\mathcal{B}\qty(E^\mathcal{B})$, and $\mathcal{A}^c$ and $\mathcal{B}^c$ be the canonical versions of $\mathcal{A}$ and $\mathcal{B}$, respectively.
\begin{enumerate}
\item $\mathcal{A}\underline{\leftrightarrow}\mathcal{B}$ iff $\exists \sigma(x,y)\subset\Theta(x,y)$ for $(x,y)\in E^\mathcal{A}\times E^\mathcal{B}$ such that $\bigvee\sigma(\cdot,\cdot)$ satisfies $\mathcal{A}\leftrightarrows_G\mathcal{B}$, where
\begin{align*}
\Theta(x,y)\coloneqq\left\{\begin{array}{ll}
\emptyset, & Pre^\mathcal{A}(x)\not\equiv Pre^\mathcal{B}(y),\\
\set{\top}, & Pre^\mathcal{A}(x)\equiv Pre^\mathcal{B}(y).
\end{array}\right.
\end{align*}
\item $\mathcal{A}\leftrightarrows_P\mathcal{B}$ iff $\exists \eta(\cdot,\cdot)\in\set{\bot,\top}$ satisfying $\mathcal{A}\leftrightarrows_G\mathcal{B}$.
\item $\mathcal{A}\leftrightarrows\mathcal{B}$ iff $\exists \eta(x,y)\in\set{\bot,Pre^\mathcal{A}(x)\land Pre^\mathcal{B}(y)}$ for $(x,y)\in E^\mathcal{A}\times E^\mathcal{B}$ satisfying $\mathcal{A}\leftrightarrows_G\mathcal{B}$.
\item $\mathcal{M}^c\otimes\mathcal{A}\underline{\leftrightarrow}\mathcal{M}^c\otimes\mathcal{B}$ iff $\mathcal{A}\leftrightarrows_S\mathcal{B}$ iff $\mathcal{A}^c\leftrightarrows\mathcal{B}^c$ iff $\exists \sigma(\cdot,\cdot)\subset \Gamma\circ K\circ C\qty(Pre^\mathcal{A}\qty(E^\mathcal{A})\cup Pre^\mathcal{B}\qty(E^\mathcal{B}))$ such that $\bigvee\sigma(\cdot,\cdot)$ satisfies $\mathcal{A}\leftrightarrows_G\mathcal{B}$.
\end{enumerate}
\end{remark}

\begin{theorem}\label{TheEmuEquiv}
Let $\mathcal{A}$ and $\mathcal{B}$ be action models. Then $\mathcal{A}\leftrightarrows_G\mathcal{B}$ iff $\mathcal{A}\equiv\mathcal{B}$.
\end{theorem}
\begin{proof}
\begin{description}
\item[Necessity] Let $\mathcal{M}$ be a Kripke model. Let $R\coloneqq\set{((w,x),(w,y))\in W^{\mathcal{M}\otimes\mathcal{A}}\times W^{\mathcal{M}\otimes\mathcal{B}} \mid \mathcal{M}\vDash_w\eta(x,y)}$. We show that $R$ satisfies $\mathcal{M}\otimes\mathcal{A}\underline{\leftrightarrow}\mathcal{M}\otimes\mathcal{A}$ as follows.
\begin{description}
\item[Invariance] $\forall \qty((w,x),(w,y))\in R$, $Val^{\mathcal{M}\otimes\mathcal{A}}(w,x)=Val^\mathcal{M}(w)=Val^{\mathcal{M}\otimes\mathcal{B}}(w,y)$.

\item[Zig] $\forall ((w,x),(w,y))\in R$, $\mathcal{M}\vDash_w \eta(x,y)$. Thus, by the Zig of $\mathcal{A}\leftrightarrows_G\mathcal{B}$, $\forall (w',x')\in\to_a^{\mathcal{M}\otimes\mathcal{A}}\qty((w,x),\cdot)$ (note that this implies $\mathcal{M}\vDash_{w'} Pre^\mathcal{A}(x')$), we have $\mathcal{M}\vDash_{w'} \bigvee_{y'\in\to_a^\mathcal{B}(y,\cdot)}Pre^\mathcal{B}(y')\land\eta(x',y')$. Therefore, $\exists y'\in\to_a^\mathcal{B}(y,\cdot)$, $\mathcal{M}\vDash_{w'}Pre^\mathcal{B}(y')\land\eta(x',y')$, thereby $R((w',x'),(w',y'))$.

\item[Zag] Symmetric to the Zig.

\item[Zig0] $\forall (w,x)\in W_0^{\mathcal{M}\otimes\mathcal{A}}$, because $Pre^\mathcal{A}(x)\vDash \bigvee_{y\in E_0^\mathcal{B}}Pre^\mathcal{B}(y)\land \eta(x,y)$ and $\mathcal{M}\vDash_w Pre^\mathcal{A}(x)$, there exists $y\in E_0^\mathcal{B}$ such that $\mathcal{M}\vDash_w Pre^\mathcal{B}(y)\land\eta(x,y)$. Thus, $R((w,x),(w,y))$.

\item[Zag0] Symmetric to the Zig0.
\end{description}

\item[Sufficiency] Let $\mathcal{A}^c$ and $\mathcal{B}^c$ be the canonical versions of $\mathcal{A}$ and $\mathcal{B}$. By Proposition \ref{ProEquivImpCanEmu}, $\mathcal{A}^c\leftrightarrows\mathcal{B}^c$. By Remark \ref{RemResToGen}, $\mathcal{A}\leftrightarrows_G\mathcal{B}$.
\end{description}
\end{proof}

\begin{remark}
Let $\mathcal{A}$ and $\mathcal{B}$ be action models. By Proposition \ref{ProBisImpProImpActImpEqu}, $\mathcal{A}\leftrightarrows \mathcal{B}$ implies $\mathcal{A}\equiv\mathcal{B}$. The reverse that $\mathcal{A}\equiv\mathcal{B}$ implies $\mathcal{A}\leftrightarrows \mathcal{B}$ is left as an open problem \cite{SietsmaEijckJOPL2013}. We disprove it by the following counterexample with $a\in A$ and $p_1,p_2\in \mathcal{P}$.
\begin{enumerate}
\item $E^\mathcal{A}\coloneqq\set{x_1,x_2,x_3,x_4}$ and $E_0^\mathcal{A}\coloneqq\set{x_1,x_3}$. $E^\mathcal{B}\coloneqq\set{y_1,y_2,y_3,y_4}$ and $E_0^\mathcal{B}\coloneqq\set{y_1,y_3}$.
\item $\to_a^\mathcal{A}\coloneqq\set{(x_1,x_2),(x_3,x_4)}$. $\to_a^\mathcal{B}\coloneqq\set{(y_1,y_2),(y_3,y_4)}$.
\item $Pre^\mathcal{A}(x_1)=Pre^\mathcal{A}(x_3)=Pre^\mathcal{B}(y_1)=Pre^\mathcal{B}(y_3)\coloneqq\square_a p_1\lor\square_a p_2$, $Pre^\mathcal{A}(x_2)\coloneqq\top$, $Pre^\mathcal{A}(x_4)\coloneqq p_1\land p_2$, $Pre^\mathcal{B}(y_2)\coloneqq p_1$, $Pre^\mathcal{B}(y_4)\coloneqq p_2$.
\end{enumerate}
Assume for the sake of contradiction that $\mathcal{A}\leftrightarrows\mathcal{B}$. Then $\exists S\subset E^\mathcal{A}\times E^\mathcal{B}$ satisfying $\mathcal{A}\leftrightarrows\mathcal{B}$. Obviously, $(x_1,y_1),(x_1,y_3)\notin S$ because the Zig of $\mathcal{A}\leftrightarrows\mathcal{B}$ is invalid for both. Since
\begin{align*}
Pre^\mathcal{A}(x_1)=\square_a p_1\lor\square_a p_2\nvDash p_1\lor p_2\equiv Pre^\mathcal{B}(y_2)\lor Pre^\mathcal{B}(y_4),
\end{align*}
the Zig0 of $\mathcal{A}\leftrightarrows\mathcal{B}$ is invalid for $x_1$, conflicts.

Define $\eta:E^\mathcal{A}\times E^\mathcal{B}\mapsto\mathcal{L}$ as follows.
\begin{enumerate}
\item Let $\eta(x_1,y_1)=\eta(x_3,y_4)\coloneqq\square_a p_1$ and $\eta(x_1,y_3)=\eta(x_3,y_1)\coloneqq\square_a p_2$.
\item Let $\eta(x_2,y_2)=\eta(x_2,y_4)=\eta(x_4,y_2)=\eta(x_4,y_4)\coloneqq\top$.
\item Let $\eta(x_1,y_2)=\eta(x_1,y_4)=\eta(x_3,y_2)=\eta(x_3,y_4)=\eta(x_2,y_1)=\eta(x_2,y_3)=\eta(x_4,y_1)=\eta(x_4,y_3)\coloneqq\bot$. 
\end{enumerate}
After a long but standard validation, one may prove that $\eta$ satisfies $\mathcal{A}\leftrightarrows_G\mathcal{B}$, thereby $\mathcal{A}\equiv\mathcal{B}$ by Theorem \ref{TheEmuEquiv}.
\end{remark}

\begin{definition}
Let $\Phi,\Psi\subset\mathcal{L}$.
\begin{enumerate}
\item If $\forall\phi\in\Phi$, $\forall \psi\in \Psi$, either $\psi\vDash\phi$ or $\psi\vDash\lnot\phi$, then we say $\Phi$ is $\Psi$-regular.
\item If $\forall \phi\in\Phi$, $\exists\Psi'\subset\Psi$, $\phi\equiv\bigvee\Psi'$, then we say $\Phi$ is $\Psi$-basis.
\item If $\forall \phi\in\Phi$, $\forall\xi_l^a$ for $1\le l\le L$ and $a\in A$, $\phi\vDash\bigvee_{l=1}^L\bigwedge_{a\in A}\square_a\xi_l^a$ implies $\phi\vDash\bigvee_{l=1}^L\mathcal\bigvee{G}\qty(\bigwedge_{a\in A}\square_a\xi_l^a,\Psi)$, then we say $\Phi$ is $\Psi$-discrete.
\item If $\forall \phi\in\Phi$, $\forall a\in A$, $\forall \xi\in\mathcal{L}$, $\phi\vDash\square_a\xi$ implies $\phi\vDash\square_a\bigvee\mathcal{G}\qty(\xi,\Psi)$, then we say $\Phi$ is $\Psi$-known.
\end{enumerate}
\end{definition}

\begin{definition}
Define $\mathcal{L}_0\subset\mathcal{L}$ such that $\forall \alpha\in\mathcal{L}$, $\alpha\in \mathcal{L}_0$ iff $\forall \phi\in\mathcal{L}$, $\forall a\in A$, if $\alpha\not\equiv\bot$, then $\alpha\vDash \square_a\phi$ implies $\phi\equiv\top$ (or equivalently, $\phi\not\equiv\bot$ implies $\alpha\land\lozenge_a\phi\not\equiv\bot$).
\end{definition}

\begin{lemma}\label{LemGKCsatRBDK}
Let $\mathcal{A}$ and $\mathcal{B}$ be action models and $\Theta\coloneqq\Gamma\circ K\circ C\qty(Pre^\mathcal{A}\qty(E^\mathcal{A})\cup Pre^\mathcal{B}\qty(E^\mathcal{B}))$. Then
\begin{enumerate}
\item\label{ItGKCsatRBDKR} $Pre^\mathcal{A}\qty(E^\mathcal{A})$ and $Pre^\mathcal{B}\qty(E^\mathcal{B})$ are $\Theta$-regular.
\item\label{ItGKCsatRBDKB} $Pre^\mathcal{A}\qty(E^\mathcal{A})$ and $Pre^\mathcal{B}\qty(E^\mathcal{B})$ are $\Theta$-basis.
\item\label{ItGKCsatRBDKD} $\Theta$ is $\Theta$-discrete.
\item\label{ItGKCsatRBDKK} $\Theta$ is $\Theta$-known.
\end{enumerate}
\end{lemma}
\begin{proof}
Cases \ref{ItGKCsatRBDKR} and \ref{ItGKCsatRBDKB} are obvious. We now prove Cases \ref{ItGKCsatRBDKD} and \ref{ItGKCsatRBDKK}. $\forall \theta\in\Theta$, $\exists\Phi\in K\circ C\qty(Pre^\mathcal{A}\qty(E^\mathcal{A})\cup Pre^\mathcal{B}\qty(E^\mathcal{B}))$, such that $\theta=\bigwedge\Phi$.
\begin{enumerate}
\item Let $\Phi_1\coloneqq \qty(\mathcal{P}\cap\Phi)\cup\set{\lnot p\in\Phi \mid p\in \mathcal{P}}\cup\set{\lozenge_a\phi\in\Phi \mid a\in A,\phi\in\mathcal{L}}$.

\item Let $\Phi_2\coloneqq\set{\square_a\phi\in\Phi \mid a\in A,\phi\in\mathcal{L}}$.

\item Let $\Phi_3\coloneqq\set{\phi\land\psi\in\Phi \mid \phi,\psi\in\mathcal{L}}$.

\item Let $\Phi_4\coloneqq\set{\phi\lor\psi\in\Phi \mid \phi,\psi\in\mathcal{L}}$.
\end{enumerate}
Then $\Phi=\bigcup_{i=1}^4\Phi_i$.

Prove that $\bigwedge\Phi\equiv\bigwedge \Phi_1\land\bigwedge\Phi_2$. $\forall \phi\in\Phi$, if $\phi\in \Phi_1\cup\Phi_2$, then $\bigwedge \Phi_1\land\bigwedge\Phi_2\vDash\phi$. Otherwise, if $\phi=\phi_1\land\phi_2\in \Phi_3$, then $\phi_1,\phi_2\in\Phi$. By inducting the complexity over $\land$, both $\bigwedge \Phi_1\land\bigwedge\Phi_2\vDash\phi_1$ and $\bigwedge \Phi_1\land\bigwedge\Phi_2\vDash\phi_2$, thereby $\bigwedge \Phi_1\land\bigwedge\Phi_2\vDash\phi$. Otherwise, $\phi=\phi_1\lor\phi_2\in \Phi_4$, thereby at least one of $\phi_1$ and $\phi_2$ belongs to $\Phi$. Assume without loss of generality that $\phi_1\in\Phi$. Then by inducting the complexity over $\lor$, $\bigwedge \Phi_1\land\bigwedge\Phi_2\vDash\phi_1\vDash\phi$.

$\forall a\in A$, let $\Psi_a\coloneqq\set{\psi \mid \square_a\psi\in\Phi_2}$. If $\bigwedge \Phi_1\land\bigwedge\Phi_2\equiv\theta\vDash\bigvee_{l=1}^L\bigwedge_{a\in A}\square_a\xi_l^a$ for some $L>0$ and $\xi_l^a\in\mathcal{L}$, then since $\bigwedge\Phi_1\in\mathcal{L}_0$, we have $\theta\not\equiv\bot$ implies $\bigwedge\Phi_2\vDash \bigvee_{l=1}^L\bigwedge_{a\in A}\square_a\xi_l^a$, thereby $\exists 1\le l'\le L$, $\forall a\in A$, $\bigwedge\Psi_a\vDash \xi_{l'}^a$. Therefore, $\theta\vDash\bigwedge_{a\in A}\square_a\bigwedge\Psi_a\vDash\bigwedge_{a\in A}\square_a\xi_{l'}^a$. Thus, $\theta=\bigvee\mathcal{G}\qty(\square_a\xi_{l'}^a,\set{\theta})\vDash\bigvee_{l=1}^L\bigvee\mathcal{G}\qty(\square_a\xi_l^a,\Theta)$, which is Case \ref{ItGKCsatRBDKD}.

If $\bigwedge \Phi_1\land\bigwedge\Phi_2\vDash\square_a\xi$ for some $\xi\in\mathcal{L}$, then since $\bigwedge\Phi_1\in\mathcal{L}_0$, we have $\bigwedge\Psi_a\vDash \xi$. Because $\Psi_a\subset C\qty(Pre^\mathcal{A}\qty(E^\mathcal{A})\cup Pre^\mathcal{B}\qty(E^\mathcal{B}))$, we have that $\bigwedge\Psi_a$ is $\Theta$-basis, thereby $\bigwedge\Psi_a\vDash\bigvee\mathcal{G}\qty(\xi,\Theta)$, which implies $\bigwedge \Phi_1\land\bigwedge\Phi_2\vDash\square_a\bigvee\mathcal{G}\qty(\xi,\Theta)$. Since $\theta\equiv\bigwedge\Phi\equiv\bigwedge \Phi_1\land\bigwedge\Phi_2$, we have Case \ref{ItGKCsatRBDKK}.
\end{proof}

\begin{lemma}\label{LemSqXiaSGActEmu}
Let $\mathcal{A}$ and $\mathcal{B}$ be action models with $\mathcal{A}\leftrightarrows_G\mathcal{B}$. Then $\exists \xi_{x,y}^a$ for $(x,y)\in E^\mathcal{A}\times E^\mathcal{B}$ and $a\in A$ such that $\eta(x,y)\coloneqq\bigwedge_{a\in A}\square_a\xi_{x,y}^a$ satisfies $\mathcal{A}\leftrightarrows_G\mathcal{B}$.
\end{lemma}
\begin{proof}
By Definition \ref{DefGenActEmu}, $\exists \eta':E^\mathcal{A}\times E^\mathcal{B}\mapsto\mathcal{L}$ such that $\mathcal{A}\leftrightarrows_G\mathcal{B}$. $\forall (x,y)\in E^\mathcal{A}\times E^\mathcal{B}$, $\forall a\in A$, let
\begin{align*}
\xi_{x,y}^a\coloneqq\bigwedge_{x'\in\to_a^\mathcal{A}(x,\cdot)}\qty(Pre^\mathcal{A}(x')\to\bigvee_{y'\in\to_a^\mathcal{B}(y,\cdot)}\qty(Pre^\mathcal{B}(y')\land\eta'(x',y')))\land\bigwedge_{y'\in\to_a^\mathcal{B}(y,\cdot)}\qty(Pre^\mathcal{B}(y')\to\bigvee_{x'\in\to_a^\mathcal{A}(x,\cdot)}\qty(Pre^\mathcal{A}(x')\land\eta'(x',y'))).
\end{align*}
Prove that $\eta(x,y)=\bigwedge_{a\in A}\square_a\xi_{x,y}^a$ for $(x,y)\in E^\mathcal{A}\times E^\mathcal{B}$ satisfies $\mathcal{A}\leftrightarrows_G\mathcal{B}$.
\begin{description}
\item[Zig] $\forall (x,y)\in E^\mathcal{A}\times E^\mathcal{B}$, by the Zig of $\mathcal{A}\leftrightarrows_G\mathcal{B}$ over $\eta'$,
\begin{align}\label{EqEtapImpEta}
\eta'(x,y)\vDash\bigwedge_{a\in A}\square_a\xi_{x,y}^a\equiv\eta(x,y).
\end{align}
Thus, $\forall a\in A$, $\forall (x,y)\in E^\mathcal{A}\times E^\mathcal{B}$,
\begin{align*}
&\xi_{x,y}^a=\bigwedge_{x'\in\to_a^\mathcal{A}(x,\cdot)}\qty(Pre^\mathcal{A}(x')\to\bigvee_{y'\in\to_a^\mathcal{B}(y,\cdot)}\qty(Pre^\mathcal{B}(y')\land\eta'(x',y')))\land\bigwedge_{y'\in\to_a^\mathcal{B}(y,\cdot)}\qty(Pre^\mathcal{B}(y')\to\bigvee_{x'\in\to_a^\mathcal{A}(x,\cdot)}\qty(Pre^\mathcal{A}(x')\land\eta'(x',y')))\\
&\vDash \bigwedge_{x'\in\to_a^\mathcal{A}(x,\cdot)}\qty(Pre^\mathcal{A}(x')\to\bigvee_{y'\in\to_a^\mathcal{B}(y,\cdot)}\qty(Pre^\mathcal{B}(y')\land\eta(x',y')))\land\bigwedge_{y'\in\to_a^\mathcal{B}(y,\cdot)}\qty(Pre^\mathcal{B}(y')\to\bigvee_{x'\in\to_a^\mathcal{A}(x,\cdot)}\qty(Pre^\mathcal{A}(x')\land\eta(x',y'))),
\end{align*}
which implies the Zig of $\mathcal{A}\leftrightarrows_G\mathcal{B}$ over $\eta$.

\item[Zag] Symmetric to the Zig.

\item[Zig0] This is obviously by the Zig0 of $\mathcal{A}\leftrightarrows_G\mathcal{B}$ over $\eta'$ and Eq. \eqref{EqEtapImpEta}.

\item[Zag0] Symmetric to the Zig0.
\end{description}
\end{proof}

\begin{lemma}\label{LemCoverStable}
Let $\mathcal{A}$ and $\mathcal{B}$ be action models, $F\subset\mathcal{L}$ such that $Pre^\mathcal{A}\qty(E^\mathcal{A})$ and $Pre^\mathcal{B}\qty(E^\mathcal{B})$ are $F$-basis and $F$-regular, $\Theta\subset\mathcal{L}$ such that $F$ is $\Theta$-discrete. Let $\xi_{x,y}^a\in\mathcal{L}$ for $(x,y)\in E^\mathcal{A}\times E^\mathcal{B}$ and $a\in A$ such that $\eta(x,y)\coloneqq\bigwedge_{a\in A}\square_a\xi_{x,y}^a$ satisfies the Zig0 and Zag0 of $\mathcal{A}\leftrightarrows_G\mathcal{B}$. Then $\eta'(x,y)\coloneqq\bigvee\mathcal{G}\qty(\eta(x,y),\Theta)$ for $(x,y)\in E^\mathcal{A}\times E^\mathcal{B}$ satisfies the Zig0 and Zag0 of $\mathcal{A}\leftrightarrows_G\mathcal{B}$ as well.
\end{lemma}
\begin{proof}
We only prove the Zig0 since the Zag0 is symmetric. $\forall x\in E_0^\mathcal{A}$, $\forall f\in F$ such that $Pre^\mathcal{A}(x)\land f\not\equiv\bot$, since $Pre^\mathcal{A}(x)$ is $F$-regular, we have $f\vDash Pre^\mathcal{A}(x)$. By the Zig0 of $\mathcal{A}\leftrightarrows_G\mathcal{B}$ over $\eta$,
\begin{align*}
f\vDash \bigvee_{y\in E_0^\mathcal{B}}\qty(Pre^\mathcal{B}(y)\land\eta(x,y)).
\end{align*}
Because $Pre^\mathcal{B}\qty(E_0^\mathcal{B})$ are $F$ regular, we have
\begin{align*}
f\vDash \bigvee_{y\in E_0^\mathcal{B}:f\vDash Pre^\mathcal{B}(y)}\eta(x,y).
\end{align*}
Because $F$ is $\Theta$-discrete, we have
\begin{align*}
f\vDash \bigvee_{y\in E_0^\mathcal{B}:f\vDash Pre^\mathcal{B}(y)}\bigvee\mathcal{G}\qty(\eta(x,y),\Theta)\equiv \bigvee_{y\in E_0^\mathcal{B}:f\vDash Pre^\mathcal{B}(y)}\qty(Pre^\mathcal{B}(y)\land\eta'(x,y))\vDash \bigvee_{y\in E_0^\mathcal{B}}\qty(Pre^\mathcal{B}(y)\land\eta'(x,y)).
\end{align*}
Because $Pre^\mathcal{A}(x)$ is $F$-basis and $f\in F$ satisfies $Pre^\mathcal{A}(x)\land f\not\equiv\bot$, we have
\begin{align*}
Pre^\mathcal{A}(x)\vDash \bigvee_{y\in E_0^\mathcal{B}}\qty(Pre^\mathcal{B}(y)\land\eta'(x,y)).
\end{align*}
\end{proof}

\begin{lemma}\label{LemIterStable}
Let $\mathcal{A}$ and $\mathcal{B}$ be action models, $F\subset\mathcal{L}$ such that $Pre^\mathcal{A}\qty(E^\mathcal{A})$ and $Pre^\mathcal{B}\qty(E^\mathcal{B})$ are $F$-regular, $\Theta\subset\mathcal{L}$ such that $F$ is $\Theta$-discrete and $\Theta$ is $F$-known. Let $\xi_{x,y}^a\in\mathcal{L}$ for $(x,y)\in E^\mathcal{A}\times E^\mathcal{B}$ and $a\in A$ such that $\eta(x,y)\coloneqq\bigwedge_{a\in A}\square_a\xi_{x,y}^a$ satisfies the Zig and Zag of $\mathcal{A}\leftrightarrows_G\mathcal{B}$. Then $\eta'(x,y)\coloneqq\bigvee\mathcal{G}\qty(\eta(x,y),\Theta)$ for $(x,y)\in E^\mathcal{A}\times E^\mathcal{B}$ satisfies the Zig and Zag of $\mathcal{A}\leftrightarrows_G\mathcal{B}$ as well.
\end{lemma}
\begin{proof}
We only prove the Zig since the Zag is symmetric. $\forall (x,y)\in E^\mathcal{A}\times E^\mathcal{B}$, $\forall \theta\in\mathcal{G}\qty(\eta(x,y),\Theta)$, $\forall a\in A$, $\forall x'\in\to_a^\mathcal{A}(x,\cdot)$, by the Zig of $\mathcal{A}\leftrightarrows_G\mathcal{B}$ over $\eta$,
\begin{align*}
\theta\vDash \eta(x,y)\vDash\square_a\qty(Pre^\mathcal{A}(x')\to\bigvee_{y'\in\to_a^\mathcal{B}(y,\cdot)}\qty(Pre^\mathcal{B}(y')\land\eta(x',y'))).
\end{align*}
Because $\theta$ is $F$-known, we have
\begin{align*}
\theta\vDash \eta(x,y)\vDash\square_a\bigvee\mathcal{G}\qty(Pre^\mathcal{A}(x')\to\bigvee_{y'\in\to_a^\mathcal{B}(y,\cdot)}\qty(Pre^\mathcal{B}(y')\land\eta(x',y')), F).
\end{align*}
$\forall f\in \mathcal{G}\qty(Pre^\mathcal{A}(x')\to\bigvee_{y'\in\to_a^\mathcal{B}(y,\cdot)}\qty(Pre^\mathcal{B}(y')\land\eta(x',y')), F)$ such that $f\land Pre^\mathcal{A}(x')\not\equiv\bot$, since $Pre^\mathcal{A}(x')$ is $F$-regular, we have $f\vDash Pre^\mathcal{A}(x')$, thereby $f\vDash \bigvee_{y'\in\to_a^\mathcal{B}(y,\cdot)}\qty(Pre^\mathcal{B}(y')\land\eta(x',y'))$. Since $Pre^\mathcal{B}\qty(E^\mathcal{B})$ are $F$-regular, we have $f\vDash \bigvee_{y'\in\to_a^\mathcal{B}(y,\cdot):f\vDash Pre^\mathcal{B}(y')}\eta(x',y')$. Because $f$ is $\Theta$-discrete, we have
\begin{align*}
f\vDash \bigvee_{y'\in\to_a^\mathcal{B}(y,\cdot):f\vDash Pre^\mathcal{B}(y')}\bigvee\mathcal{G}\qty(\eta(x',y'),\Theta)\equiv\bigvee_{y'\in\to_a^\mathcal{B}(y,\cdot):f\vDash Pre^\mathcal{B}(y')}\qty(Pre^\mathcal{B}(y')\land\eta'(x',y'))\vDash\bigvee_{y'\in\to_a^\mathcal{B}(y,\cdot)}\qty(Pre^\mathcal{B}(y')\land\eta'(x',y')).
\end{align*}
In summary,
\begin{align*}
\bigvee\mathcal{G}\qty(Pre^\mathcal{A}(x')\to\bigvee_{y'\in\to_a^\mathcal{B}(y,\cdot)}\qty(Pre^\mathcal{B}(y')\land\eta(x',y')), F)\vDash Pre^\mathcal{A}(x')\to\bigvee_{y'\in\to_a^\mathcal{B}(y,\cdot)}\qty(Pre^\mathcal{B}(y')\land\eta'(x',y')),
\end{align*}
thereby
\begin{align*}
\theta\vDash\square_a\qty(Pre^\mathcal{A}(x')\to\bigvee_{y'\in\to_a^\mathcal{B}(y,\cdot)}Pre^\mathcal{B}(y')\land\eta'(x',y')).
\end{align*}
Since $\theta\in\mathcal{G}\qty(\eta(x,y),\Theta)$ is arbitrary, we have the Zig of $\mathcal{A}\leftrightarrows_G\mathcal{B}$ for $\eta'$
\end{proof}

\begin{theorem}\label{TheGAEImpFourCriGAE}
Let $\mathcal{A}$ and $\mathcal{B}$ be action models with $\mathcal{A}\leftrightarrows_G\mathcal{B}$, $F\subset\mathcal{L}$ such that $Pre^\mathcal{A}\qty(E^\mathcal{A})$ and $Pre^\mathcal{B}\qty(E^\mathcal{B})$ are $F$-basis and $F$-regular, $\Theta\subset\mathcal{L}$ such that $F$ is $\Theta$-discrete and $\Theta$ is $F$-known. Then $\exists\sigma(\cdot,\cdot)\subset\Theta$ such that $\bigvee\sigma(\cdot,\cdot)$ satisfies $\mathcal{A}\leftrightarrows_G\mathcal{B}$.
\end{theorem}
\begin{proof}
By Lemma \ref{LemSqXiaSGActEmu}, $\exists \xi_{x,y}^a$ for $(x,y)\in E^\mathcal{A}\times E^\mathcal{B}$ and $a\in A$ such that $\eta(x,y)\coloneqq\bigwedge_{a\in A}\square_a\xi_{x,y}^a$ satisfies $\mathcal{A}\leftrightarrows_G\mathcal{B}$. By Lemmas \ref{LemCoverStable} and \ref{LemIterStable}, $\sigma(x,y)\coloneqq\mathcal{G}\qty(\eta(x,y),\Theta)$ for $(x,y)\in E^\mathcal{A}\times E^\mathcal{B}$ satisfies $\mathcal{A}\leftrightarrows_G\mathcal{B}$.
\end{proof}

\begin{remark}\label{RemAlterNeceProof}
The necessity of Theorem \ref{TheEmuEquiv} has an alternative proof. By Lemma \ref{LemGKCsatRBDK}, Theorem \ref{TheGAEImpFourCriGAE}, and Remark \ref{RemResToGen}, the canonical versions $\mathcal{A}^c$ and $\mathcal{B}^c$ of $\mathcal{A}$ and $\mathcal{B}$ satisfy $\mathcal{A}^c\leftrightarrows\mathcal{B}^c$. Thus, by Proposition \ref{ProEquivImpCanEmu}, $\mathcal{A}\equiv\mathcal{B}$.
\end{remark}

\begin{algorithm}
\caption{Determine structural relationships between action models}
\label{DEOTMB}
\begin{algorithmic}[1]
\Require $\mathcal{A}$, $\mathcal{B}$, $\Theta$.
\For{$(x,y)\in E^\mathcal{A}\times E^\mathcal{B}$}
\State $\sigma_0(x,y)\gets\Theta(x,y)$.
\EndFor
\While{true}
\If{$\exists x\in E_0^\mathcal{A}$, $Pre^\mathcal{A}(x)\nvDash\bigvee_{y\in E_0^\mathcal{B}}Pre^\mathcal{B}(y)\land\bigvee\sigma_0(x,y)$}\label{LineFalsexBetter}
\State \Return false.
\EndIf
\If{$\exists y\in E_0^\mathcal{B}$, $Pre^\mathcal{B}(y)\nvDash\bigvee_{x\in E_0^\mathcal{A}}Pre^\mathcal{A}(x)\land\bigvee\sigma_0(x,y)$}\label{LineFalseyBetter}
\State \Return false.
\EndIf
\For{$(x,y)\in E^\mathcal{A}\times E^\mathcal{B}$}
\For{$a\in A$}
\State
\begin{align*}
&\lambda_{x,y}^a\gets \bigwedge_{x'\in\to_a^\mathcal{A}(x,\cdot)}\qty(Pre^\mathcal{A}(x')\to\bigvee_{y'\in\to_a^\mathcal{B}(y,\cdot)}Pre^\mathcal{B}(y')\land\bigvee\sigma_0(x',y'))\\
&\land\bigwedge_{y'\in\to_a^\mathcal{B}(y,\cdot)}\qty(Pre^\mathcal{B}(y')\to\bigvee_{x'\in\to_a^\mathcal{A}(x,\cdot)}Pre^\mathcal{A}(x')\land\bigvee\sigma_0(x',y')).
\end{align*}\label{LineZigZag}
\EndFor
\State $\sigma_1(x,y)\gets\mathcal{G}\qty(\bigwedge_{a\in A}\square_a\lambda_{x,y}^a,\sigma_0(x,y))$.\label{LineUpSigma}
\EndFor
\If{$\sigma_0=\sigma_1$}\label{LineTrueCondBetter}
\State \Return true.\label{LineTrueBetter}
\EndIf
\State $\sigma_0\gets\sigma_1$.
\EndWhile
\end{algorithmic}
\end{algorithm}

\begin{proposition}\label{ProFniteRetTrue2Equiv}
Let $\mathcal{A}$ and $\mathcal{B}$ be action models and $\Theta(x,y)\subset\mathcal{L}$ for $(x,y)\in E^\mathcal{A}\times E^\mathcal{B}$. Then we have the following.
\begin{enumerate}
\item If $\abs{\Theta(\cdot,\cdot)}<+\infty$, then Algorithm \ref{DEOTMB} returns in finite steps.
\item If Algorithm \ref{DEOTMB} returns true, then $\bigvee\sigma_0(\cdot,\cdot)$ satisfies $\mathcal{A}\leftrightarrows_G\mathcal{B}$ ({\it i.e.} $\mathcal{A}\equiv\mathcal{B}$).
\end{enumerate}
\end{proposition}
\begin{proof}
\begin{enumerate}
\item To avoid returning at Line \algref{DEOTMB}{LineTrueBetter}, one needs $\sigma_1\neq\sigma_0$ at Line \algref{DEOTMB}{LineTrueCondBetter}, which implies $\exists (x,y)\in E^\mathcal{A}\times E^\mathcal{B}$ such that $\sigma_1(x,y)=\mathcal{G}\qty(\bigwedge_{a\in A}\square_a\lambda_{x,y}^a,\sigma_0(x,y))\subsetneq\sigma_0(x,y)$. This can only last for finite steps since $\sigma_0(\cdot,\cdot)=\Theta(\cdot,\cdot)$ is initially finite.

\item This is because Lines \algref{DEOTMB}{LineFalsexBetter} and \algref{DEOTMB}{LineFalsexBetter} justify the Zig0 and Zag0, and Line \algref{DEOTMB}{LineZigZag} justifies the Zig and Zag.
\end{enumerate}
\end{proof}

\begin{lemma}\label{LemRetFiniRetTru}
Let $\mathcal{A}$ and $\mathcal{B}$ be action models and $\sigma(x,y)\subset\Theta(x,y)\subset\mathcal{L}$ for $(x,y)\in E^\mathcal{A}\times E^\mathcal{B}$. Let $\bigvee\sigma(\cdot,\cdot)$ satisfy $\mathcal{A}\leftrightarrows_G\mathcal{B}$. Then if Algorithm \ref{DEOTMB} returns, it must return true.
\end{lemma}
\begin{proof}
Initially, $\sigma(x,y)\subset\sigma_0(x,y)=\Theta(x,y)$. $\forall (x,y)\in E^\mathcal{A}\times E^\mathcal{B}$, $\forall a\in A$, let
\begin{align*}
&\widehat{\lambda}_{x,y}^a\coloneqq \bigwedge_{x'\in\to_a^\mathcal{A}(x,\cdot)}\qty(Pre^\mathcal{A}(x')\to\bigvee_{y'\in\to_a^\mathcal{B}(y,\cdot)}Pre^\mathcal{B}(y')\land\bigvee\sigma(x',y'))\\
&\land\bigwedge_{y'\in\to_a^\mathcal{B}(y,\cdot)}\qty(Pre^\mathcal{B}(y')\to\bigvee_{x'\in\to_a^\mathcal{A}(x,\cdot)}Pre^\mathcal{A}(x')\land\bigvee\sigma(x',y')).
\end{align*}
By the Zig and Zag of $\mathcal{A}\leftrightarrows_G\mathcal{B}$ over $\bigvee\sigma(\cdot,\cdot)$, we have
\begin{align*}
\sigma(x,y)=\mathcal{G}\qty(\bigwedge_{a\in A}\square_a\widehat{\lambda}_{x,y}^a,\sigma(x,y))\subset \mathcal{G}\qty(\bigwedge_{a\in A}\square_a\lambda_{x,y}^a,\sigma_0(x,y))=\sigma_1(x,y).
\end{align*}
Thus, $\sigma(x,y)\subset\sigma_0(x,y)$ for all steps of Algorithm \ref{DEOTMB}. By the Zig0 and Zag0 of $\mathcal{A}\leftrightarrows_G\mathcal{B}$ over $\bigvee\sigma(\cdot,\cdot)$, we have that Lines \algref{DEOTMB}{LineFalsexBetter} and \algref{DEOTMB}{LineFalseyBetter} cannot be satisfied. Thus, if Algorithm \ref{DEOTMB} returns in finite steps, then the only possibility is to return true at Line \algref{DEOTMB}{LineTrueBetter}.
\end{proof}

\begin{proposition}\label{ProGenSetNece}
Let $\mathcal{A}$ and $\mathcal{B}$ be action models, $\Theta(x,y)\coloneqq\widehat{\Theta}\subset\mathcal{L}$ for $(x,y)\in E^\mathcal{A}\times E^\mathcal{B}$, $F\subset\mathcal{L}$. If $Pre^\mathcal{A}\qty(E^\mathcal{A})$ and $Pre^\mathcal{B}\qty(E^\mathcal{B})$ are $F$-regular and $F$-basis, $F$ is $\widehat{\Theta}$-discrete, and $\widehat{\Theta}$ is $F$-known, then $\mathcal{A}\leftrightarrows_G\mathcal{B}$ ({\it i.e.} $\mathcal{A}\equiv\mathcal{B}$) implies that Algorithm \ref{DEOTMB} returns true in finite steps.
\end{proposition}
\begin{proof}
Since $\abs{\widehat{\Theta}}<+\infty$, we have by Proposition \ref{ProFniteRetTrue2Equiv} that Algorithm \ref{DEOTMB} must return in finite steps. By Theorem \ref{TheGAEImpFourCriGAE} and Lemma \ref{LemRetFiniRetTru}, Algorithm \ref{DEOTMB} returns true.
\end{proof}

\begin{corollary}\label{CorAlorReturnBisim}
Let $\mathcal{A}$ and $\mathcal{B}$ be action models and
\begin{align*}
\Theta(x,y)=\left\{\begin{array}{ll}
\set{\top}, & Pre^\mathcal{A}(x)\equiv Pre^\mathcal{B}(y),\\
\emptyset, & Pre^\mathcal{A}(x)\not\equiv Pre^\mathcal{B}(y),
\end{array}\right.
\end{align*}
for $(x,y)\in E^\mathcal{A}\times E^\mathcal{B}$. Then Algorithm \ref{DEOTMB} must return in finite steps, and it returns true iff $\mathcal{A}\underline{\leftrightarrow}\mathcal{B}$.
\end{corollary}
\begin{proof}
\begin{enumerate}
\item By Proposition \ref{ProFniteRetTrue2Equiv}, Algorithm \ref{DEOTMB} returns in finite steps.

\item If Algorithm \ref{DEOTMB} returns true, then by Proposition \ref{ProFniteRetTrue2Equiv}, $\bigvee\sigma_0(\cdot,\cdot)$ satisfies $\mathcal{A}\leftrightarrows_G\mathcal{B}$. By Remark \ref{RemResToGen}, $\mathcal{A}\underline{\leftrightarrow}\mathcal{B}$.

\item If $\mathcal{A}\underline{\leftrightarrow}\mathcal{B}$, then by Remark \ref{RemResToGen}, $\exists\sigma(\cdot,\cdot)\subset\Theta(\cdot,\cdot)$ such that $\bigvee\sigma(\cdot,\cdot)$ satisfies $\mathcal{A}\leftrightarrows_G\mathcal{B}$. Thus, by Lemma \ref{LemRetFiniRetTru}, Algorithm \ref{DEOTMB} returns true.
\end{enumerate}
\end{proof}

\begin{corollary}
Let $\mathcal{A}$ and $\mathcal{B}$ be action models and $\Theta(x,y)=\set{\top}$ for $(x,y)\in E^\mathcal{A}\times E^\mathcal{B}$. Then Algorithm \ref{DEOTMB} must return in finite steps, and it returns true iff $\mathcal{A}\leftrightarrows_P\mathcal{B}$.
\end{corollary}
\begin{proof}
Similar as Corollary \ref{CorAlorReturnBisim}.
\end{proof}

\begin{corollary}
Let $\mathcal{A}$ and $\mathcal{B}$ be action models and $\Theta(x,y)=\set{Pre^\mathcal{A}(x)\land Pre^\mathcal{B}(y)}$ for $(x,y)\in E^\mathcal{A}\times E^\mathcal{B}$. Then Algorithm \ref{DEOTMB} must return in finite steps, and it returns true iff $\mathcal{A}\leftrightarrows\mathcal{B}$.
\end{corollary}
\begin{proof}
Similar as Corollary \ref{CorAlorReturnBisim}.
\end{proof}

\begin{corollary}\label{CoroGKCIterEquiv}
Let $\mathcal{A}$ and $\mathcal{B}$ be action models and $\Theta(x,y)\coloneqq\Gamma\circ K \circ C\qty(Pre^\mathcal{A}\qty(E^\mathcal{A})\cup Pre^\mathcal{B}\qty(E^\mathcal{B}))$ for $(x,y)\in E^\mathcal{A}\times E^\mathcal{B}$. Then $\mathcal{A}\leftrightarrows_G\mathcal{B}$ ({\it i.e.} $\mathcal{A}\equiv\mathcal{B}$) implies that Algorithm \ref{DEOTMB} returns true in finite steps.
\end{corollary}
\begin{proof}
This is obvious by Lemma \ref{LemGKCsatRBDK} and Proposition \ref{ProGenSetNece}.
\end{proof}

\section{A new formula set for the efficient determination of the action model equivalence}
\label{SecNFS}

\begin{definition}\label{DefCDN}
Define $\widetilde{\mathcal{L}}\subset \mathcal{L}$ such that $\forall\widetilde{\xi}\in\mathcal{L}$, $\widetilde{\xi}\in\widetilde{\mathcal{L}}$ iff
\begin{align}\label{EqTilde}
\widetilde{\xi}=\bigvee_{m=1}^M\qty(\alpha_m\land\bigwedge_{a\in A}\square_a\phi_m^a),
\end{align}
where $M,N_m\ge 0$, $\alpha_m\in\mathcal{L}_0$, $\phi_m^a\in\mathcal{L}$, $\bot\not\equiv\alpha_m\land\bigwedge_{a\in A}\square_a\phi_m^a$.

Let $\widetilde{\xi}\in\widetilde{\mathcal{L}}$ as in Eq. \eqref{EqTilde} and $\widetilde{\Phi}\subset\widetilde{\mathcal{L}}$. Define
\begin{align*}
&\norm{\widetilde{\xi}}\coloneqq M,\\
&\alpha_m\qty(\widetilde{\xi})\coloneqq\alpha_m,\\
&D_m\qty(\widetilde{\xi})\coloneqq \alpha_m\land\beta_m\land\bigwedge_{a\in A}\square_a\phi_m^a,\\
&D\qty(\widetilde{\xi})\coloneqq\set{D_m\qty(\widetilde{\xi}) \mid 1\le m\le \norm{\widetilde{\xi}}},\\
&D_m^a\qty(\widetilde{\xi})\coloneqq \phi_m^a,\\
&D^a\qty(\widetilde{\xi})\coloneqq\set{D_m^a\qty(\widetilde{\xi}) \mid 1\le m\le\norm{\widetilde{\xi}}},\\
&D_m^a\qty(\widetilde{\Phi})\coloneqq\set{D_m^a\qty(\widetilde{\phi}) \mid \widetilde{\phi}\in\widetilde{\Phi}},\\
&D_m^A\qty(\widetilde{\xi})\coloneqq\set{D_m^a\qty(\widetilde{\xi}) \mid a\in A}.
\end{align*}

Define $\mathbb{M}\qty(\widetilde{\xi})\subset\set{m \mid 1\le m\le \norm{\widetilde{\xi}}}$ such that $m\in \mathbb{M}\qty(\widetilde{\xi})$ iff $\forall 1\le m'\le \norm{\widetilde{\xi}}$, $D_m^A\qty(\widetilde{\xi})\vDash D_{m'}^A\qty(\widetilde{\xi})$ ($D_m^a\qty(\widetilde{\xi})\vDash D_{m'}^a\qty(\widetilde{\xi})$ for all $a\in A$) implies $D_{m'}^A\qty(\widetilde{\xi})\vDash D_m^A\qty(\widetilde{\xi})$. Further define
\begin{align*}
&D_{[]}\qty(\widetilde{\xi})\coloneqq\set{D_m\qty(\widetilde{\xi}) \mid m\in\mathbb{M}\qty(\widetilde{\xi})},\\
&D_{[]}^a\qty(\widetilde{\xi})\coloneqq\set{D^a_m\qty(\widetilde{\xi}) \mid m\in\mathbb{M}\qty(\widetilde{\xi})}.
\end{align*}

Define $\widehat{\mathcal{L}}\coloneqq\set{\widehat{\phi}\in\widetilde{\mathcal{L}} \mid \forall 1\le m\le\norm{\widehat{\phi}}, \forall a\in A, D_m^a\qty(\widehat{\phi})\in\widehat{\mathcal{L}}}$. Let $\phi\in\mathcal{L}$. Define $\widetilde{\mathcal{L}}(\phi)\coloneqq\set{\widetilde{\phi}\in\widetilde{\mathcal{L}} \mid \widetilde{\phi}\equiv\phi}$, $\widehat{\mathcal{L}}(\phi)\coloneqq\set{\widehat{\phi}\in\widehat{\mathcal{L}} \mid \widehat{\phi}\equiv\phi}$, and $\widehat{\mathcal{L}}_0(\phi)\coloneqq\set{\widehat{\phi}\in\widehat{\mathcal{L}}(\phi) \mid \delta\qty(\widehat{\phi})\le\delta\qty(\phi), \forall i\ge 0,\qty(D^A)^i\qty(\widehat{\phi})\Subset C(\phi)\cup\set{\top}}$, where $\qty(D^A)^i\qty(\widehat{\phi})\Subset C(\phi)\cup\set{\top}$ means that $\forall\psi_1\in\qty(D^A)^i\qty(\widehat{\phi})$, $\exists\psi_2\in C(\phi)\cup\set{\top}$, $\psi_1\equiv\psi_2$.
\end{definition}

\begin{proposition}\label{ProSubClose}
Let $\phi\in\mathcal{L}$. Then $\widehat{\mathcal{L}}_0(\phi)\neq\emptyset$.
\end{proposition}
\begin{proof}
Note that the result is obvious for $\phi\in\mathcal{L}_0$ since $\widehat{\phi}=\phi\land\bigwedge_{a\in A}\square_a\top\in\widehat{\mathcal{L}}(\phi)$ satisfies $\delta\qty(\widehat{\phi})=\delta\qty(\phi)$ and $D_m^a\qty(\widehat{\phi})=\top$. For $\phi\notin\mathcal{L}_0$, let
\begin{align*}
\widetilde{\phi}\coloneqq \bigvee_{m=1}^M\qty(\alpha_m\land\bigwedge_{a\in A}\square_a\phi_m^a)
\end{align*}
be any disjunctive normal form of the basic formulas ({\it i.e.} propositional formulas, $\square_a\psi$ and $\lozenge_a\psi$ with $a\in A$ and $\psi\in\mathcal{L}$) of $\phi$. Then $\alpha_m\in\mathcal{L}_0$ and $\forall 1\le m\le M$, $\forall a\in A$, $\phi_m^a\in C(\phi)$. By inducting on the formula depth, $\exists\widehat{\phi}_m^a\in\widehat{\mathcal{L}}_0\qty(\phi_m^a)\neq\emptyset$, $\delta\qty(\widehat{\phi}_m^a)\le\delta\qty(\phi_m^a)$ and $\forall i\ge 0$, $\qty(D^A)^i\qty(\widehat{\phi}_m^a)\Subset C(\phi_m^a)\cup\set{\top}$. Let
\begin{align*}
\widehat{\phi}\coloneqq \bigvee_{m=1}^M\qty(\alpha_m\land\bigwedge_{a\in A}\square_a\widehat{\phi}_m^a)\in\widehat{\mathcal{L}}\qty(\phi).
\end{align*}
Obviously, $\delta\qty(\widehat{\phi})\le\delta\qty(\phi)$ and $\qty(D^A)^0\qty(\widehat{\phi})=\set{\widehat{\phi}}\Subset C(\phi)$ since $\widehat{\phi}\equiv\phi$. Since $1\le m\le M$ and $a\in A$ are arbitrary, and $\phi_m^a\in C(\phi)$ implies $C(\phi_m^a)\subset C(\phi)$, we have
\begin{align*}
\bigcup_{m=1}^M\bigcup_{a\in A}\qty(D^A)^i\qty(\widehat{\phi}_m^a)=\bigcup_{m=1}^M\bigcup_{a\in A}\qty(D^A)^i\circ D_m^a\qty(\widehat{\phi})=\qty(D^A)^{i+1}\qty(\widehat{\phi})\Subset \bigcup_{m=1}^M\bigcup_{a\in A}C(\phi_m^a)\cup\set{\top}\subset C(\phi)\cup\set{\top}.
\end{align*}
Thus, $\widehat{\phi}\in\widehat{\mathcal{L}}_0\qty(\phi)$.
\end{proof}

\begin{definition}\label{DefOtimes}
Let $\widehat{\phi},\widehat{\psi}\in\widehat{\mathcal{L}}$. Define $\zeta\qty(\widehat{\phi},\widehat{\psi})\coloneqq\set{(m,n) \mid 1\le m\le\norm{\widehat{\phi}}, 1\le n\le\norm{\widehat{\psi}}, D_m\qty(\widehat{\phi})\land D_n\qty(\widehat{\psi})\not\equiv\bot}$. Define
\begin{align*}
\widehat{\phi}\otimes\widehat{\psi} = \bigvee_{(m,n)\in\zeta\qty(\widehat{\phi},\widehat{\psi})}\qty(\alpha_m\qty(\widehat{\phi})\land\alpha_n\qty(\widehat{\psi})\land\bigwedge_{a\in A}\square_a D_m^a\qty(\widehat{\phi})\otimes D_n^a\qty(\widehat{\psi})).
\end{align*}
Let $\widehat{\Phi},\widehat{\Psi}\subset\widehat{L}$. Define $\widehat{\Phi}\otimes\widehat{\Psi}\coloneqq\set{\widehat{\phi}'\otimes\widehat{\psi}' \mid \widehat{\phi}'\in\widehat{\Phi},\widehat{\psi}'\in\widehat{\Psi}}$.
\end{definition}

\begin{proposition}\label{ProOtimesEquivLand}
Let $\phi,\psi\in\mathcal{L}$, $\widehat{\phi}\in\widehat{\mathcal{L}}\qty(\phi)$, $\widehat{\psi}\in\widehat{\mathcal{L}}\qty(\psi)$. Then $\widehat{\phi}\otimes\widehat{\psi}\in \widehat{\mathcal{L}}(\phi\land\psi)$.
\end{proposition}
\begin{proof}
Note that the result is obvious for $\widehat{\phi},\widehat{\psi}\in\mathcal{L}_0$ since
\begin{align*}
\phi\land\psi\equiv\widehat{\phi}\land\widehat{\psi}=\widehat{\phi}\otimes\widehat{\psi}\in\widehat{\mathcal{L}}.
\end{align*}
Otherwise,
\begin{align*}
&\phi\land\psi\equiv\widehat{\phi}\land\widehat{\psi}\equiv\bigvee_{m=1}^\norm{\widehat{\phi}}\bigvee_{n=1}^\norm{\widehat{\psi}}\qty(\alpha_m\qty(\widehat{\phi})\land\alpha_n\qty(\widehat{\psi})\land\bigwedge_{a\in A}\square_a D_m^a\qty(\widehat{\phi})\land D_n^a\qty(\widehat{\psi}))\\
&\equiv\bigvee_{(m,n)\in\zeta\qty(\widehat{\phi},\widehat{\psi})}\qty(\alpha_m\qty(\widehat{\phi})\land\alpha_n\qty(\widehat{\psi})\land\bigwedge_{a\in A}\square_a D_m^a\qty(\widehat{\phi})\land D_n^a\qty(\widehat{\psi})).
\end{align*}
By inducting on the formula depth, $\forall (m,n)\in\zeta\qty(\widehat{\phi},\widehat{\psi})$, $\forall a\in A$, $D_m^a\qty(\widehat{\phi})\otimes D_n^a\qty(\widehat{\psi})\in\widehat{\mathcal{L}}\qty(D_m^a\qty(\widehat{\phi})\land D_n^a\qty(\widehat{\psi}))$, thereby $\widehat{\phi}\otimes\widehat{\psi}\in \widehat{\mathcal{L}}(\phi\land\psi)$.
\end{proof}

\begin{proposition}\label{ProlandSub}
Let $\widehat{\phi},\widehat{\psi}\in\widehat{\mathcal{L}}$ and $a\in A$. Then $D\qty(\widehat{\phi}\otimes\widehat{\psi})\subset D\qty(\widehat{\phi})\otimes D\qty(\widehat{\psi})$.
\end{proposition}
\begin{proof}
By Definition \ref{DefOtimes},
\begin{align*}
&D\qty(\widehat{\phi}\otimes\widehat{\psi})=\set{D_m\qty(\widehat{\phi})\otimes D_n\qty(\widehat{\psi}) \mid (m,n)\in\zeta\qty(\widehat{\phi},\widehat{\psi})}\subset\set{D_m\qty(\widehat{\phi})\otimes D_n\qty(\widehat{\psi}) \mid 1\le m\le \norm{\widehat{\phi}},1\le n\le \norm{\widehat{\psi}}}\\
&=\set{D_m\qty(\widehat{\phi}) \mid 1\le m\le \norm{\widehat{\phi}}}\otimes \set{D_n\qty(\widehat{\psi}) \mid 1\le n\le \norm{\widehat{\psi}}}=D\qty(\widehat{\phi})\otimes D\qty(\widehat{\psi}).
\end{align*}
\end{proof}

\begin{definition}\label{DefFH}
Let $\widehat{\Phi}\coloneqq\set{\set{\widehat{\phi}^+_m, \widehat{\phi}^-_m} \mid 1\le m\le M}\in\widehat{\mathcal{L}}$. $\forall i\ge 0$, define
\begin{align*}
&F_i\qty(\widehat{\Phi})\coloneqq\left\{\begin{array}{ll}
\bigotimes_{m=1}^M\set{\widehat{\phi}_m^+,\widehat{\phi}_m^-}, & i=0,\\
D^A_{[]}\circ F_{i-1}\qty(\widehat{\Phi})\otimes F_0\qty(\widehat{\Phi}), & i>0.
\end{array}\right.\\
&H_i\qty(\widehat{\Phi})\coloneqq\bigcup_{a_i\in A}\bigcup_{a_{i-1}\in A}\cdots\bigcup_{a_1\in A}\qty(D^{a_i}\circ\cdots\circ D^{a_1}\circ F_0\qty(\widehat{\Phi})\otimes D^{a_i}\circ\cdots\circ D^{a_2}\circ F_0\qty(\widehat{\Phi})\otimes \cdots\otimes D^{a_i}\circ F_0\qty(\widehat{\Phi})).
\end{align*}
\end{definition}

\begin{proposition}\label{ProHiDep}
Let $\widehat{\Phi}\coloneqq\set{\set{\widehat{\phi}^+_m,\widehat{\phi}^-_m} \mid 1\le m\le M}\in\widehat{\mathcal{L}}$. Then $\forall i>\delta\qty(\widehat{\Phi})$, $H_i\qty(\widehat{\Phi})=H_{\delta\qty(\widehat{\Phi})}\qty(\widehat{\Phi})$.
\end{proposition}
\begin{proof}
This is obvious by noting that $\forall i>\delta\circ F_0\qty(\widehat{\Phi})=\delta\qty(\widehat{\Phi})$, $\qty(D^A)^i\circ F_0\qty(\widehat{\Phi})=\set{\top}$.
\end{proof}

\begin{proposition}\label{ProFiInHi}
Let $\widehat{\Phi}\coloneqq\set{\set{\widehat{\phi}^+_m,\widehat{\phi}^-_m} \mid 1\le m\le M}\in\widehat{\mathcal{L}}$. Then $\forall i\ge 0$, $D_{[]}^A\circ F_i\qty(\widehat{\Phi})\subset H_{i+1}\qty(\widehat{\Phi})$.
\end{proposition}
\begin{proof}
By Definition \ref{DefFH}, $D_{[]}^A\circ F_0\qty(\widehat{\Phi})\subset D^A\circ F_0\qty(\widehat{\Phi})=H_1\qty(\widehat{\Phi})$. For $i>0$, we have by the induction and Proposition \ref{ProlandSub} that
\begin{align*}
&D_{[]}^A\circ F_i\qty(\widehat{\Phi})=D_{[]}^A\qty(D^A_{[]}\circ F_{i-1}\qty(\widehat{\Phi})\otimes F_0\qty(\widehat{\Phi}))\subset D^A\qty(H_i\qty(\widehat{\Phi})\otimes F_0\qty(\widehat{\Phi}))\\
&=\bigcup_{a_{i+1}\in A}D^{a_{i+1}}\qty(F_0\qty(\widehat{\Phi})\otimes\bigcup_{a_i\in A}\cdots\bigcup_{a_1\in A}\qty(D^{a_i}\circ\cdots\circ D^{a_1}\circ F_0\qty(\widehat{\Phi})\otimes \cdots\otimes D^{a_i}\circ F_0\qty(\widehat{\Phi})))\\
&\subset\bigcup_{a_{i+1}\in A}\cdots\bigcup_{a_1\in A}\qty(D^{a_{i+1}}\circ\cdots\circ D^{a_1}\circ F_0\qty(\widehat{\Phi})\otimes \cdots\otimes D^{a_{i+1}}\circ F_0\qty(\widehat{\Phi}))=H_{i+1}\qty(\widehat{\Phi}).
\end{align*}
\end{proof}

\begin{lemma}\label{LemBoundF}
Let $\Phi\subset\mathcal{L}$ and $\widehat{\Phi}\coloneqq\set{\set{\widehat{\phi}^+,\widehat{\phi}^-} \mid \phi\in\Phi, \widehat{\phi}^+\in\widehat{\mathcal{L}}_0\qty(\phi), \widehat{\phi}^-\in\widehat{\mathcal{L}}_0\qty(\lnot\phi)}\subset\widehat{\mathcal{L}}$. Moreover, $\forall 0\le i<\delta\qty(\Phi)$, $\forall \psi\in\qty(D^A)^i\qty(\widehat{\Phi})$, $\norm{\psi}\le k$. Then $\forall 1\le i\le \delta\qty(\Phi)$,
\begin{align*}
&\abs{H_i\qty(\widehat{\Phi})}\le \qty(2k^{(i+1)/2})^{\abs{\Phi}i}\abs{A}^i,\\
&\abs{\bigcup_{i=1}^\infty H_i\qty(\widehat{\Phi})}\le\frac{(2k)^\abs{\Phi}\abs{A}}{(2k)^\abs{\Phi}\abs{A}-1} \qty(2k^{(\delta(\Phi)+1)/2})^{\abs{\Phi}\delta(\Phi)}\abs{A}^{\delta(\Phi)}.
\end{align*}
\end{lemma}
\begin{proof}
By Proposition \ref{ProlandSub}, $\forall i\ge 1$,
\begin{align*}
D^{a_i}\circ D^{a_{i-1}}\circ\cdots\circ D^{a_1}\circ G_0\qty(\widehat{\Phi})\subset\bigotimes_{\set{\widehat{\phi}^+,\widehat{\phi}^-}\in\widehat{\Phi}} D^{a_i}\circ D^{a_{i-1}}\circ\cdots\circ D^{a_1}\qty(\set{\widehat{\phi}^+,\widehat{\phi}^-}).
\end{align*}
Therefore,
\begin{align*}
&\abs{D^{a_i}\circ D^{a_{i-1}}\circ\cdots\circ D^{a_1}\circ G_0\qty(\widehat{\Phi})}\le \abs{\bigotimes_{\set{\widehat{\phi}^+,\widehat{\phi}^-}\in\widehat{\Phi}} D^{a_i}\circ D^{a_{i-1}}\circ\cdots\circ D^{a_1}\qty(\set{\widehat{\phi}^+,\widehat{\phi}^-})}\\
&\le\prod_{\set{\widehat{\phi}^+,\widehat{\phi}^-}\in\widehat{\Phi}}\abs{D^{a_i}\circ D^{a_{i-1}}\circ\cdots\circ D^{a_1}\qty(\set{\widehat{\phi}^+,\widehat{\phi}^-})}\le\prod_{\set{\widehat{\phi}^+,\widehat{\phi}^-}\in\widehat{\Phi}} \qty(k^i\abs{\set{\widehat{\phi}^+,\widehat{\phi}^-}})=\qty(2k^i)^\abs{\Phi}.
\end{align*}
Thus, $\forall 1\le i\le \delta(\Phi)$,
\begin{align}\label{EqHnEsti}
\abs{H_i\qty(\widehat{\Phi})}\le\abs{A}^i\qty(\qty(2k^i)^\abs{\Phi}\qty(2k^{i-1})^\abs{\Phi}\qty(2k^{i-2})^\abs{\Phi}\cdots \qty(2k)^\abs{\Phi})=\qty(2k^{(i+1)/2})^{\abs{\Phi}i}\abs{A}^i.
\end{align}
By Proposition \ref{ProHiDep},
\begin{align*}
&\abs{\bigcup_{i=1}^\infty H_i\qty(\widehat{\Phi})}=\abs{\bigcup_{i=1}^{\delta\qty(\widehat{\Phi})} H_i\qty(\widehat{\Phi})}\le \abs{\bigcup_{i=1}^{\delta\qty(\Phi)} H_i\qty(\widehat{\Phi})} \le\sum_{i=1}^{\delta(\Phi)}\abs{H_i\qty(\widehat{\Phi})}\le\sum_{i=1}^{\delta(\Phi)} \qty(2k^{(i+1)/2})^{\abs{\Phi}i}\abs{A}^i\\
&=\qty(2k^{(\delta(\Phi)+1)/2})^{\abs{\Phi}\delta(\Phi)}\abs{A}^{\delta(\Phi)}\sum_{i=1}^{\delta(\Phi)}\frac{1}{2^{\abs{\Phi}(\delta(\Phi)-i)}k^{\abs{\Phi}\sum_{j=i+1}^{\delta(\Phi)} j}\abs{A}^{\delta(\Phi)-i}}\le \qty(2k^{(\delta(\Phi)+1)/2})^{\abs{\Phi}\delta(\Phi)}\abs{A}^{\delta(\Phi)}\sum_{i=1}^{\delta(\Phi)}\frac{1}{2^{\abs{\Phi}(\delta(\Phi)-i)}k^{\abs{\Phi}(\delta(\Phi)-i)}\abs{A}^{\delta(\Phi)-i}}\\
&=\qty(2k^{(\delta(\Phi)+1)/2})^{\abs{\Phi}\delta(\Phi)}\abs{A}^{\delta(\Phi)}\sum_{i=0}^{\delta(\Phi)-1}\frac{1}{2^{\abs{\Phi}i}k^{\abs{\Phi}i}\abs{A}^i}=\qty(2k^{(\delta(\Phi)+1)/2})^{\abs{\Phi}\delta(\Phi)}\abs{A}^{\delta(\Phi)}\sum_{i=0}^\infty\frac{1}{2^{\abs{\Phi}i}k^{\abs{\Phi}i}\abs{A}^i}\\
&=\frac{(2k)^\abs{\Phi}\abs{A}}{(2k)^\abs{\Phi}\abs{A}-1} \qty(2k^{(\delta(\Phi)+1)/2})^{\abs{\Phi}\delta(\Phi)}\abs{A}^{\delta(\Phi)}.
\end{align*}
\end{proof}

\begin{definition}
Let $\widehat{\mathcal{F}}\subset\widehat{\mathcal{L}}$. Define $\kappa\qty(\widehat{\mathcal{F}})\coloneqq\set{\bigwedge_{a\in A}\square_a D_m^a\qty(\widehat{\psi}) \mid \widehat{\psi}\in\widehat{\mathcal{F}}, m\in\mathbb{M}\qty(\widehat{\psi})}$.
\end{definition}

\begin{proposition}\label{ProThreeProperties}
Let $\Phi\subset\mathcal{L}$ and $\widehat{\Phi}\coloneqq\set{\set{\widehat{\phi}^+,\widehat{\phi}^-} \mid \phi\in\Phi, \widehat{\phi}^+\in\widehat{\mathcal{L}}_0\qty(\phi), \widehat{\phi}^-\in\widehat{\mathcal{L}}_0\qty(\lnot\phi)}\subset\widehat{\mathcal{L}}$. Then
\begin{enumerate}
\item $\Phi$ is $\bigcup_{i=0}^\infty F_i\qty(\widehat{\Phi})$-regular.

\item $\Phi$ is $\bigcup_{i=0}^\infty F_i\qty(\widehat{\Phi})$-basis.

\item $\bigcup_{i=0}^\infty F_i\qty(\widehat{\Phi})$ is $\kappa\qty(\bigcup_{i=0}^\infty F_i\qty(\widehat{\Phi}))$-discrete.

\item $\kappa\qty(\bigcup_{i=0}^\infty F_i\qty(\widehat{\Phi}))$ is $\bigcup_{i=0}^\infty F_i\qty(\widehat{\Phi})$-known.
\end{enumerate}
\end{proposition}
\begin{proof}
\begin{enumerate}
\item By Definition \ref{DefFH}, $\forall i> 0$, $F_i\qty(\widehat{\Phi})=D^A_{[]}\circ F_{i-1}\qty(\widehat{\Phi})\otimes F_0\qty(\widehat{\Phi})$. By Proposition \ref{ProOtimesEquivLand},
\begin{align*}
F_0\qty(\widehat{\Phi})=\bigotimes_{\set{\widehat{\phi}^+,\widehat{\phi}^-}\in\widehat{\Phi}}\set{\widehat{\phi}^+,\widehat{\phi}^-}\subset\widehat{\mathcal{L}}\qty(\bigwedge_{\phi\in\Phi}\set{\phi_m,\lnot\phi_m}).
\end{align*}

\item Since $\Phi$ is $\bigcup_{\phi\in\Phi}\set{\phi,\lnot\phi}$-basis, we have by Proposition \ref{ProOtimesEquivLand} that $\Phi$ is $F_0\qty(\widehat{\Phi})$-basis, thereby $\bigcup_{i=0}^\infty F_i\qty(\widehat{\Phi})$-basis.

\item Let $\widehat{\psi}\in \bigcup_{i=0}^\infty F_i\qty(\widehat{\Phi})$ and $\xi_l^a\in\mathcal{L}$ for $1\le l\le L$ and $a\in A$ such that $\widehat{\psi}\vDash\bigvee_{l=1}^L\bigwedge_{a\in A}\square_a\xi_l^a$. $\forall 1\le n\le \norm{\widehat{\psi}}$, $\exists n'\in \mathbb{M}\qty(\widehat{\psi})$, $\forall a\in A$, $D_n^a\qty(\widehat{\psi})\vDash D_{n'}^a\qty(\widehat{\psi})$. Since $\alpha_{n'}\qty(\widehat{\psi})\in\mathcal{L}_0$ and $D_{n'}\qty(\widehat{\psi})\vDash \bigvee_{l=1}^L\bigwedge_{a\in A}\square_a\xi_l^a$, we have that $\exists 1\le l'\le L$, $\forall a\in A$, $D^a_{n'}\qty(\widehat{\psi})\vDash \xi_{l'}^a$. This implies that $D_n\qty(\widehat{\psi})\vDash \mathcal{G}\qty(\bigwedge_{a\in A}\square_a\xi_{l'}^a, \bigwedge_{a\in A}\square_a D_n^a\qty(\widehat{\psi}))$. Since $\bigwedge_{a\in A}\square_a D_n^a\qty(\widehat{\psi})\in \kappa\qty(\bigcup_{i=0}^\infty F_i\qty(\widehat{\Phi}))$, we have $D_n\qty(\widehat{\psi})\vDash \mathcal{G}\qty(\bigwedge_{a\in A}\square_a\xi_{l'}^a, \kappa\qty(\bigcup_{i=0}^\infty F_i\qty(\widehat{\Phi})))$, thereby
\begin{align*}
D_n\qty(\widehat{\psi})\vDash \bigvee_{l=1}^L\mathcal{G}\qty(\bigwedge_{a\in A}\square_a\xi_{l}^a, \kappa\qty(\bigcup_{i=0}^\infty F_i\qty(\widehat{\Phi}))).
\end{align*}
Since $1\le n\le \norm{\widehat{\psi}}$ is arbitrary, we have $\widehat{\psi}\vDash \bigvee_{l=1}^L\mathcal{G}\qty(\bigwedge_{a\in A}\square_a\xi_{l}^a, \kappa\qty(\bigcup_{i=0}^\infty F_i\qty(\widehat{\Phi})))$.

\item Let $\widehat{\psi}\in D_{[]}^A\circ\kappa\qty(\bigcup_{i=0}^\infty F_i\qty(\widehat{\Phi}))\equiv D_{[]}^A\qty(\bigcup_{i=0}^\infty F_i\qty(\widehat{\Phi}))$. By Definition \ref{DefFH},
\begin{align*}
&\widehat{\Psi}\coloneqq F_0\qty(\widehat{\Phi})\otimes \widehat{\psi}\subset F_0\qty(\widehat{\Phi})\otimes D_{[]}^A\qty(\bigcup_{i=0}^\infty F_i\qty(\widehat{\Phi}))=F_0\qty(\widehat{\Phi})\otimes \bigcup_{i=0}^\infty D_{[]}^A\circ F_i\qty(\widehat{\Phi})\\
&=\bigcup_{i=0}^\infty F_0\qty(\widehat{\Phi})\otimes D_{[]}^A\circ F_i\qty(\widehat{\Phi})=\bigcup_{i=0}^\infty F_{i+1}\qty(\widehat{\Phi})\subset\bigcup_{i=0}^\infty F_i\qty(\widehat{\Phi}).
\end{align*}
By Proposition \ref{ProOtimesEquivLand},
\begin{align*}
\bigvee\widehat{\Psi}=\bigvee\qty(F_0\qty(\widehat{\Phi})\otimes \widehat{\psi})\equiv\bigvee\qty(F_0\qty(\widehat{\Phi})\land \widehat{\psi})\equiv \bigvee\qty(F_0\qty(\widehat{\Phi}))\land \widehat{\psi}\equiv\top\land \widehat{\psi}\equiv \widehat{\psi}.
\end{align*}
\end{enumerate}
\end{proof}

\begin{corollary}
Let $\mathcal{A}$ and $\mathcal{B}$ be action models, $\Phi\coloneqq Pre^\mathcal{A}\qty(E^\mathcal{A})\cup Pre^\mathcal{B}\qty(E^\mathcal{B})$, $\widehat{\Phi}\coloneqq\set{\set{\widehat{\phi}^+, \widehat{\phi}^-} \mid \phi\in\Phi,\widehat{\phi}^+\in\widehat{\mathcal{L}}_0\qty(\phi), \widehat{\phi}^-\in\widehat{\mathcal{L}}_0\qty(\lnot\phi)}$, $\Theta(x,y)=\kappa\qty(\bigcup_{i=0}^\infty F_i\qty(\widehat{\Phi}))$ for $(x,y)\in E^\mathcal{A}\times E^\mathcal{B}$. Then Algorithm \ref{DEOTMB} must return in finite steps, and it returns true iff $\mathcal{A}\leftrightarrows_G\mathcal{B}$ ({\it i.e.} $\mathcal{A}\equiv\mathcal{B}$).
\end{corollary}
\begin{proof}
This is obvious by Propositions \ref{ProFniteRetTrue2Equiv}, \ref{ProGenSetNece} and \ref{ProThreeProperties}.
\end{proof}

\begin{proposition}\label{ProFPhiLandCPhi}
Let $\Phi\coloneqq\set{\phi_m \mid 1\le m\le M}\subset\mathcal{L}$, $\widehat{\Phi}\coloneqq\set{\set{\widehat{\phi}^+_m,\widehat{\phi}^-_m} \mid 1\le m\le M, \widehat{\phi}^+_m\in\widehat{\mathcal{L}}_0\qty(\phi_m), \widehat{\phi}^-_m\in\widehat{\mathcal{L}}_0\qty(\lnot\phi_m)}\subset\widehat{\mathcal{L}}$, $i\ge 0$, $\widehat{\xi}\in H_i\qty(\widehat{\Phi})$. Then $\exists \Psi\subset C(\Phi)$, $\widehat{\xi}\equiv\bigwedge\Psi$.
\end{proposition}
\begin{proof}
By Definition \ref{DefFH}, $\exists\widehat{\xi}^j\in \qty(D^A)^j\circ F_0\qty(\widehat{\Phi})$ for $1\le j\le i$ such that $\widehat{\xi}=\bigotimes_{j=1}^i\widehat{\xi}^j$. $\forall 1\le j\le i$, by Proposition \ref{ProlandSub},
\begin{align*}
\widehat{\xi}^j\in\qty(D^A)^j\circ F_0\qty(\widehat{\Phi})\subset\bigotimes_{m=1}^M \qty(D^A)^j\qty(\set{\widehat{\phi}^+_m,\widehat{\phi}^-_m}).
\end{align*}
Thus, $\exists\widehat{\xi}_m^j\in \qty(D^A)^j\qty(\set{\widehat{\phi}^+_m,\widehat{\phi}^-_m})$ for $1\le m\le M$ such that $\widehat{\xi}^j=\bigotimes_{m=1}^M\widehat{\xi}_m^j$. Because $\widehat{\phi}^+_m\in\widehat{\mathcal{L}}_0\qty(\phi_m)$ and $\widehat{\phi}^-_m\in\widehat{\mathcal{L}}_0\qty(\lnot\phi_m)$, we have $\widehat{\xi}_m^j\Subset C(\phi_m)\cup\set{\top}\subset C(\Phi)\cup\set{\top}$, thereby $\exists\xi_m^j\in C(\Phi)\cup\set{\top}$, $\widehat{\xi}_m^j\equiv\xi_m^j$. By Proposition \ref{ProOtimesEquivLand},
\begin{align*}
\widehat{\xi}=\bigotimes_{j=1}^i\widehat{\xi}^j=\bigotimes_{j=1}^i\bigotimes_{m=1}^M\widehat{\xi}_m^j\equiv\bigwedge_{j=1}^i\bigwedge_{m=1}^M\xi_m^j.
\end{align*}
We have the result since $\top=\bigwedge\emptyset$.
\end{proof}

\begin{remark}
Upon the setting of Lemma \ref{LemBoundF}, $\forall \set{\widehat{\phi}^+, \widehat{\phi}^-}\in\widehat{\Phi}$, $\forall 0\le n\le \delta(\Phi)$, $\qty(D^A)^n\qty(\set{\widehat{\phi}^+, \widehat{\phi}^-})\subset C(\Phi)$. Thus, in the worst case,
\begin{align*}
&\abs{C(\Phi)}\ge \sum_{\set{\widehat{\phi}^+, \widehat{\phi}^-}\in\widehat{\Phi}}\sum_{n=0}^{\delta(\Phi)}\abs{\qty(D^A)^n\qty(\set{\widehat{\phi}^+, \widehat{\phi}^-})}\ge\sum_{\set{\widehat{\phi}^+, \widehat{\phi}^-}\in\widehat{\Phi}}\sum_{n=0}^{\delta(\Phi)}2(k\abs{A})^n = 2\abs{\Phi}\sum_{n=0}^{\delta(\Phi)}\qty(k\abs{A})^n,\\
&\abs{\Gamma\circ K\circ C(\Phi)} = \abs{K\circ C(\Phi)} = 2^{\abs{C(\Phi)}/2} = 2^{\abs{\Phi}\sum_{n=0}^{\delta(\Phi)}\qty(k\abs{A})^n}.
\end{align*}

Let $\Theta(\cdot,\cdot)=\kappa\qty(\bigcup_{i=0}^\infty F_i\qty(\widehat{\Phi}))$ in Algorithm \ref{DEOTMB}. To obtain $\mathcal{G}\qty(\bigwedge_{a\in A}\square_a\lambda_{x,y}^a,\Theta(x,y))$ at Line \algref{DEOTMB}{LineUpSigma} ($\sigma_0(\cdot,\cdot)=\Theta(\cdot,\cdot)$ initially), one must determine whether $\widehat{\xi}\vDash\lambda_{x,y}^a$ for each $a\in A$ and $\widehat{\xi}\in D_{[]}^a\qty(\bigcup_{i=0}^\infty F_i\qty(\widehat{\Phi}))$. By Proposition \ref{ProFiInHi},
\begin{align*}
D_{[]}^a\qty(\bigcup_{i=0}^\infty F_i\qty(\widehat{\Phi}))=\bigcup_{i=0}^\infty D_{[]}^a\circ F_i\qty(\widehat{\Phi})\subset \bigcup_{i=0}^\infty H_{i+1}\qty(\widehat{\Phi})=\bigcup_{i=1}^\infty H_i\qty(\widehat{\Phi}).
\end{align*}
Thus, the total validity checks are upper-bounded by $\abs{\bigcup_{i=1}^\infty H_i\qty(\widehat{\Phi})}$. By Lemma \ref{LemBoundF},
\begin{align*}
\abs{\bigcup_{i=1}^\infty H_i\qty(\widehat{\Phi})}\le\frac{(2k)^\abs{\Phi}\abs{A}}{(2k)^\abs{\Phi}\abs{A}-1} \qty(2k^{(\delta(\Phi)+1)/2})^{\abs{\Phi}\delta(\Phi)}\abs{A}^{\delta(\Phi)}.
\end{align*}

One may find that both $k$ and $\abs{A}$ are the exponents in $\abs{\Gamma\circ K\circ C(\Phi)}$ but not $\abs{\bigcup_{i=1}^\infty H_i\qty(\widehat{\Phi})}$. Moreover, $\delta(\Phi)$ is the exponent of an exponent in $\abs{\Gamma\circ K\circ C(\Phi)}$ but not $\abs{\bigcup_{i=1}^\infty H_i\qty(\widehat{\Phi})}$. Generally, one may expect that $\abs{\Gamma\circ K\circ C(\Phi)}$ is much much larger than $\abs{\bigcup_{i=1}^\infty H_i\qty(\widehat{\Phi})}$. This implies that each step of Algorithm \ref{DEOTMB} with $\Theta(\cdot,\cdot)=\kappa\qty(\bigcup_{i=0}^\infty F_i\qty(\widehat{\Phi}))$ is mush faster than that with $\Theta(\cdot,\cdot)=\Gamma\circ K\circ C(\Phi)$. Moreover, Propositions \ref{ProFPhiLandCPhi} says that $\forall a\in A$, $\forall\widehat{\xi}\in D_{[]}^a\qty(\bigcup_{i=0}^\infty F_i\qty(\widehat{\Phi}))\subset\bigcup_{i=1}^\infty H_i\qty(\widehat{\Phi})$, $\widehat{\xi}\equiv\bigwedge\Xi$ for some $\Xi\subset C(\Phi)$. Note that $\forall \Psi\in K\circ C(\Phi)$, $\Psi$ is the maximal consistent subset of $C(\Phi)$. Thus, one may further expected that the expensive PSPACE (see Theorem 6.47 of \cite{BlackburnCUP2001}) validity check of $\bigwedge\Xi\vDash\lambda_{x,y}^a$ over Kripke models is cheaper than that of $\bigwedge\Psi\vDash\bigwedge_{a\in A}\square_a\lambda_{x,y}^a$.
\end{remark}

\section{Minimize the event space of an action model under the propositional action emulation}
\label{SecAlg}

\begin{definition}
Let $F\subset\mathcal{L}^M$. $G\subset\mathcal{L}^M$ is called a formula basis of $F$ iff $\forall f\in F$, $\exists G_0\subset G$, $f\equiv\bigvee G_0$.
\end{definition}

\begin{algorithm}
\caption{Minimize the event space under the propositional action emulation}
\label{PartitionRefinement}
\begin{algorithmic}[1]
\Require $\mathcal{A}$.
\State $\mathcal{A}_0\gets\mathbb{G}(\mathcal{A},E_0^\mathcal{A})$.
\State $\sim_0^{\mathcal{A}_0}\gets E^{\mathcal{A}_0}\times E^{\mathcal{A}_0}$.
\State $\Theta_0^{\mathcal{A}_0}\gets \set{\sim_0^{\mathcal{A}_0}(x,\cdot) \mid x\in E^{\mathcal{A}_0}}$.
\For{$I=1,2,3,\cdots$}
\State $\sim_I^{\mathcal{A}_0}\gets\set{(x_1,x_2)\in \sim_{I-1}^{\mathcal{A}_0} \mid \forall \theta\in\Theta_{I-1}^{\mathcal{A}_0}, \forall a\in A, \bigvee_{x\in \to^{\mathcal{A}_0}_a(x_1, \cdot)\cap\theta}Pre^{\mathcal{A}_0}(x)\equiv \bigvee_{x\in\to^{\mathcal{A}_0}_a(x_2, \cdot)\cap\theta}Pre^{\mathcal{A}_0}(x)}$
\State $\Theta_I^{\mathcal{A}_0}\gets\set{\sim_I^{\mathcal{A}_0}(x,\cdot) \mid x\in E^{\mathcal{A}_0}}$.
\If{$\Theta_{I-1}^{\mathcal{A}_0} = \Theta_I^{\mathcal{A}_0}$}
\State Break.
\EndIf
\EndFor
\For{$\theta\in\Theta_I^{\mathcal{A}_0}$}
\State $F(\theta)\gets\set{\bigvee_{y\in\to^{\mathcal{A}_0}_a(\theta_0, \cdot)\cap\theta}Pre^{\mathcal{A}_0}(y) \mid a\in A,\theta_0\in\Theta_I^{\mathcal{A}_0},\bigvee_{x\in\theta_0}Pre^{\mathcal{A}_0}(x)\not\equiv\bot}\cup\set{\bigvee_{y\in\theta\cap E_0^{\mathcal{A}_0}}Pre^{\mathcal{A}_0}(y)}$.
\State Get the minimal formula basis $G(\theta)$ of $F(\theta)$.\label{PRSFB}
\EndFor
\State $E^\mathcal{B}\gets\set{(\theta, g) \mid \theta\in\Theta_I^{\mathcal{A}_0}, g\in G(\theta)}$.
\For{$a\in A$}
\State $\to_a^\mathcal{B}\gets\set{((\theta,g),(\theta',g'))\in E^\mathcal{B}\times E^\mathcal{B} \mid g'\vDash\bigvee_{y\in\to^{\mathcal{A}_0}_a(\theta, \cdot)\cap\theta'}Pre^{\mathcal{A}_0}(y)}$.
\EndFor
\For{$(\theta,g)\in E^\mathcal{B}$}
\State $Pre^\mathcal{B}(\theta,g)\gets g$.
\EndFor
\State $E_0^\mathcal{B}\gets \set{(\theta, g)\in E^\mathcal{B} \mid g\vDash\bigvee_{x\in\theta\cap E_0^{\mathcal{A}_0}}Pre^{\mathcal{A}_0}(x)}$.
\State $\mathcal{B}\gets\qty(E^\mathcal{B},Pre^\mathcal{B},\to^\mathcal{B},E_0^\mathcal{B})$.
\State \Return $\mathcal{B}$.
\end{algorithmic}
\end{algorithm}

\begin{theorem}\label{emulation}
Let $\mathcal{B}$ be the action model returned by Algorithm \ref{PartitionRefinement}. Then $\mathcal{A}_0\leftrightarrows_P\mathcal{B}$, where $\mathcal{A}_0\coloneqq\mathbb{G}\qty(\mathcal{A},E_0^\mathcal{A})$.
\end{theorem}
\begin{proof}
Let $S\coloneqq\set{(x,(\theta,g))\in E^{\mathcal{A}_0}\otimes E^\mathcal{B} \mid x\in\theta, Pre^{\mathcal{A}_0}(x)\land g\not\equiv\bot}$. We prove that $S$ satisfies $\mathcal{A}_0\leftrightarrows_P\mathcal{B}$.
\begin{description}
\item[Consistency] $\forall (x,(\theta,g))\in S$, $Pre^{\mathcal{A}_0}(x)\land Pre^\mathcal{B}\qty(\theta, g)\equiv Pre^{\mathcal{A}_0}(x)\land g\not\equiv\bot$.

\item[Zig] $\forall \qty(x_0, (\theta_0, g_0))\in S$, $x_0\in\theta_0$. $\forall a\in A$, $\forall x_1\in \to_a^{\mathcal{A}_0}(x_0, \cdot)\subset\to_a^{\mathcal{A}_0}(\theta_0, \cdot)$, $\exists\theta_1\in\Theta_I^{\mathcal{A}_0}$, $x_1\in\theta_1$. Because $x_0\in\theta_0$ and $Pre^{\mathcal{A}_0}(x_0)\land g_0\not\equiv\bot$, we have $\bigvee_{y\in\theta_0}Pre^{\mathcal{A}_0}(y)\not\equiv\bot$, thereby $\bigvee_{y\in\to^{\mathcal{A}_0}_a(\theta_0, \cdot)\cap\theta_1}Pre^{\mathcal{A}_0}(y)\in F(\theta_1)$. Because $G(\theta_1)$ is the minimal formula basis of $F(\theta_1)$, there exists $G_1\subset G(\theta_1)$,
\begin{align*}
Pre^{\mathcal{A}_0}(x_1)\vDash \bigvee_{y\in\to^{\mathcal{A}_0}_a(\theta_0, \cdot)\cap\theta_1}Pre^{\mathcal{A}_0}(y)\equiv\bigvee G_1 \equiv \bigvee_{g\in G_1}Pre^\mathcal{B}\qty(\theta_1, g). 
\end{align*}
$\forall g\in G_1$ such that $\qty(x_1, (\theta_1, g))\notin S$, $Pre^{\mathcal{A}_0}(x_1)\land g\equiv\bot$. Thus,
\begin{align*}
Pre^{\mathcal{A}_0}(x_1)\vDash \bigvee_{g\in G_1: S\qty(x_1, (\theta_1, g))}Pre^\mathcal{B}\qty(\theta_1, g).
\end{align*}
$\forall g\in G_1$, since $g\vDash \bigvee_{y\in\to^{\mathcal{A}_0}_a(\theta_0, \cdot)\cap\theta_1}Pre^{\mathcal{A}_0}(y)$, we have $(\theta_0, g_0) \to_a^\mathcal{B} (\theta_1, g)$. Thus,
\begin{align*}
Pre^{\mathcal{A}_0}(x_1)\vDash \bigvee_{(\theta,g)\in \to_a^\mathcal{B}\qty((\theta_0, g_0), \cdot)\cap S\qty(x_1, \cdot)}Pre^\mathcal{B}(\theta,g).
\end{align*}

\item[Zag] $\forall \qty(x_0, (\theta_0, g_0))\in S$, $\forall (\theta_1, g_1)\to_a^\mathcal{B}\qty((\theta_0, g_0),\cdot)$,
\begin{align*}
g_1\vDash \bigvee_{x\in\to_a^{\mathcal{A}_0}(\theta_0, \cdot)\cap\theta_1}Pre^{\mathcal{A}_0}(x).
\end{align*}
Because $x_0\in\theta_0$ and $\Theta_{I-1}^{\mathcal{A}_0}=\Theta_I^{\mathcal{A}_0}$, we have
\begin{align*}
g_1\vDash \bigvee_{x\in\to_a^{\mathcal{A}_0}(\theta_0, \cdot)\cap\theta_1}Pre^{\mathcal{A}_0}(x)\equiv\bigvee_{x\in\to_a^{\mathcal{A}_0}(x_0, \cdot)\cap\theta_1}Pre^{\mathcal{A}_0}(x).
\end{align*}
$\forall x\in\theta_1$ such that $x\notin S(\cdot, \theta_1, g_1)$, $Pre^{\mathcal{A}_0}(x)\land g_1\equiv\bot$. Thus,
\begin{align*}
Pre^\mathcal{B}\qty(\theta_1,g_1)\equiv g_1\vDash \bigvee_{x\in\to_a^{\mathcal{A}_0}(x_0, \cdot)\cap S(\cdot,\theta_1, g_1)}Pre^{\mathcal{A}_0}(x) .
\end{align*}

\item[Zig0] $\forall x\in E_0^{\mathcal{A}_0}$, $\exists \theta\in\Theta_I^{\mathcal{A}_0}$ such that $x\in\theta$. Because $\bigvee_{y\in\theta\cap E_0^{\mathcal{A}_0}}Pre^{\mathcal{A}_0}(y)\in F(\theta)$ and $G(\theta)$ is the minimal formula basis of $F(\theta)$, we have that $G_0(\theta)\coloneqq\set{g\in G(\theta) \mid g\vDash \bigvee_{y\in\theta\cap E_0^{\mathcal{A}_0}}Pre^{\mathcal{A}_0}(y)}$ satisfies
\begin{align*}
Pre^{\mathcal{A}_0}(x)\vDash\bigvee_{y\in\theta\cap E_0^{\mathcal{A}_0}}Pre^{\mathcal{A}_0}(y)\equiv\bigvee G_0(\theta)\equiv \bigvee_{g\in G_0(\theta)}Pre^\mathcal{B}(\theta,g).
\end{align*}
$\forall g\in G_0(\theta)$ such that $(\theta,g)\notin S(x,\cdot)$, we have $Pre^{\mathcal{A}_0}(x)\land g\equiv\bot$ since $x\in\theta$, and $(\theta,g)\in E_0^\mathcal{B}$ since $g\vDash \bigvee_{y\in\theta\cap E_0^{\mathcal{A}_0}}Pre^{\mathcal{A}_0}(y)$. Therefore,
\begin{align*}
Pre^{\mathcal{A}_0}(x)\vDash\bigvee_{(\theta',g')\in E_0^\mathcal{B}\cap S(x,\cdot)}Pre^\mathcal{B}(\theta',g').
\end{align*}

\item[Zag0] $\forall (\theta,g)\in E_0^\mathcal{B}$,
\begin{align*}
Pre^\mathcal{B}(\theta,g)\equiv g\vDash \bigvee_{x\in\theta\cap E_0^{\mathcal{A}_0}}Pre^{\mathcal{A}_0}(x).
\end{align*}
$\forall x\in \theta\setminus S(\cdot, (\theta,g))$, we have $Pre^{\mathcal{A}_0}(x)\land g\equiv\bot$. Therefore,
\begin{align*}
Pre^\mathcal{B}(\theta,g) \vDash\bigvee_{x\in\theta\cap E_0^{\mathcal{A}_0}\cap S(\cdot,(\theta,g))}Pre^{\mathcal{A}_0}(x).
\end{align*}
\end{description}
\end{proof}

\begin{lemma}\label{LemRea2All}
Let $\mathcal{A}$ and $\mathcal{B}$ be action models. If $\mathcal{A}=\mathbb{G}(\mathcal{A},E_0^\mathcal{A})$ and $\mathcal{A}\leftrightarrows_P\mathcal{B}$, then $\forall x\in E^\mathcal{A}$,
\begin{align}
Pre^\mathcal{A}(x)\vDash \bigvee_{y\in S(x,\cdot)}Pre^\mathcal{B}(y).
\end{align}
\end{lemma}
\begin{proof}
Let $S\subset E^\mathcal{A}\times E^\mathcal{B}$ be the propositional action emulation satisfying $\mathcal{A}\leftrightarrows_P\mathcal{B}$. Because $\mathcal{A}=\mathbb{G}(\mathcal{A},E_0^\mathcal{A})$, we have $\forall x\in  E^\mathcal{A}$, there exists a path 
\begin{align*}
E_0^\mathcal{A}\ni x_0 \to^\mathcal{A}_{a_0} x_1 \to^\mathcal{A}_{a_1} x_2 \to^\mathcal{A}_{a_2} x_3 \to^\mathcal{A}_{a_3} \cdots \to^\mathcal{A}_{a_{i-1}} x_i \to^\mathcal{A}_{a_i} \cdots \to^\mathcal{A}_{a_{I-1}} x_I \equiv x,
\end{align*}
such that $Pre^\mathcal{A}(x_i)\land\lozenge_a Pre^\mathcal{A}(x_{i+1})\not\equiv\bot$ for all $0\le i < i+1\le I$, which implies $Pre^\mathcal{A}(x_i)\not\equiv\bot$ for all $0\le i\le I$. By the Zig at $x_0$,
\begin{align*}
Pre^\mathcal{A}(x_0)\vDash \bigvee_{y\in S(x_0,\cdot)\cap E_0^\mathcal{B}}Pre^\mathcal{B}(y).
\end{align*}
$Pre^\mathcal{A}(x_0)\neq\bot$ implies that $S(x_0,\cdot)\neq\emptyset$. Let $y_0\in S(x_0,\cdot)$. By the Zig of $S(x_0, y_0)$,
\begin{align*}
Pre^\mathcal{A}(x_1)\vDash \bigvee_{y\in S(x_1,\cdot)\cap \to^\mathcal{B}_{a_0}(y_0,\cdot)}Pre^\mathcal{B}(y).
\end{align*}
$Pre^\mathcal{A}(x_1)\neq\bot$ implies $S(x_1,\cdot)\neq\emptyset$. Let $y_1\in S(x_1,\cdot)$. Continuing the above process, one finally reaches
\begin{align*}
Pre^\mathcal{A}(x_I)\vDash \bigvee_{y\in S(x_I,\cdot)\cap \to^\mathcal{B}_{a_{I-1}}(y_{I-1},\cdot)}Pre^\mathcal{B}(y).
\end{align*}
\end{proof}

\begin{lemma}\label{LemProEmuUnDiv}
Let $\mathcal{A}$ and $\mathcal{B}$ be action models with the propositional action emulation $S\subset E^\mathcal{A}\times E^\mathcal{B}$ satisfying $\mathcal{A}\leftrightarrows_P\mathcal{B}$. Then $\forall x_1,x_2\in E^\mathcal{A}$, $\forall y\in E^\mathcal{B}$, if $S(x_1,y)$ and $S(x_2,y)$, then $\forall I\ge 0$, $x_1\sim_I^\mathcal{A}x_2$.
\end{lemma}
\begin{proof}
Trivally, $\forall x_1,x_2\in E^\mathcal{A}$, $x_1\sim_0^\mathcal{A}x_2$. Assume by induction that for some $I\ge 0$, $\forall x_1,x_2\in E^\mathcal{A}$, $\forall y\in E^\mathcal{B}$, $S(x_1,y)$ and $S(x_2,y)$ imply $x_1\sim_I^\mathcal{A}x_2$. We prove that $x_1\sim_{I+1}^\mathcal{A}x_2$. $\forall \theta\in\Theta_I^\mathcal{A}$, $\forall a\in A$, $\forall x'\in\theta\cap \to_a^\mathcal{A}(x_1,\cdot)$, $\forall y'\in \to_a^\mathcal{B}(y,\cdot)\cap S(x',\cdot)$, $\forall x''\in\to_a^\mathcal{A}(x_2,\cdot)\cap S(\cdot,y')$, we have by the induction assumption that $x'\sim_I^\mathcal{A}x''$, thereby $x''\in\theta$. Thus, by the Zig at $(x_1,y)$ and the Zag at $(x_2,y)$,
\begin{align*}
Pre^\mathcal{A}(x')\vDash \bigvee_{y'\in \to_a^\mathcal{B}(y,\cdot)\cap S(x',\cdot)}Pre^\mathcal{B}(y')\vDash \bigvee_{y'\in \to_a^\mathcal{B}(y,\cdot)\cap S(x',\cdot)}\bigvee_{x''\in\to_a^\mathcal{A}(x_2,\cdot)\cap S(\cdot,y')}Pre^\mathcal{A}(x'')\vDash \bigvee_{x''\in \theta\cap \to_a^\mathcal{A}(x_2,\cdot)}Pre^\mathcal{A}(x'').
\end{align*}
Since $x'\in \theta\cap \to_a^\mathcal{A}(x_1,\cdot)$ is arbitrary, we have
\begin{align*}
\bigvee_{x'\in \theta\cap \to_a^\mathcal{A}(x_1,\cdot)}Pre^\mathcal{A}(x')\vDash\bigvee_{x''\in \theta\cap \to_a^\mathcal{A}(x_2,\cdot)}Pre^\mathcal{A}(x'').
\end{align*}
By symmetry,
\begin{align*}
\bigvee_{x'\in \theta\cap \to_a^\mathcal{A}(x_1,\cdot)}Pre^\mathcal{A}(x')\equiv\bigvee_{x''\in \theta\cap \to_a^\mathcal{A}(x_2,\cdot)}Pre^\mathcal{A}(x'').
\end{align*}
Since $\theta\in\Theta_I^\mathcal{A}$ and $a\in A$ are arbitrary, we have $x_1\sim_{I+1}^\mathcal{A}x_2$.
\end{proof}

\begin{theorem}\label{simplest}
Let $\mathcal{B}$ be the action model returned by Algorithm \ref{PartitionRefinement}. Then $\mathcal{A}\leftrightarrows_P\mathcal{B}$, and for any action model $\mathcal{C}\leftrightarrows_P\mathcal{A}$, $\abs{E^\mathcal{B}}\le\abs{E^\mathcal{C}}$.
\end{theorem}
\begin{proof}
Let $\mathcal{A}_0\coloneqq\mathbb{G}(\mathcal{A},E_0^\mathcal{A})$. By Theorem \ref{emulation} and Propositions \ref{ProBisProActEquEqu}, \ref{ProBisImpProImpActImpEqu}, \ref{ProGenSubIndBis}, $\mathcal{A}\leftrightarrows_P\mathcal{A}_0\leftrightarrows_P\mathcal{B}$. Assume for the sake of contradiction that $\exists\mathcal{C}\leftrightarrows_P\mathcal{A}$ such that $\abs{E^\mathcal{C}} < \abs{E^\mathcal{B}}$. By Propositions \ref{ProBisProActEquEqu}, \ref{ProBisImpProImpActImpEqu}, \ref{ProGenSubIndBis}, $\mathcal{A}_0\leftrightarrows_P\mathcal{A}\leftrightarrows_P\mathcal{C}$. Let $S\subset E^{\mathcal{A}_0}\times E^\mathcal{C}$ be the propositional action emulation satisfying $\mathcal{A}_0\leftrightarrows_P\mathcal{C}$.

$\forall \theta_1,\theta_2\in\Theta_I^{\mathcal{A}_0}$ with $\theta_1\neq\theta_2$, we have $S(\theta_1,\cdot)\cap S(\theta_2,\cdot)=\emptyset$ since otherwise, $\exists y\in S(\theta_1,\cdot)\cap S(\theta_2,\cdot)$, $\exists x_1\in\theta_1$, $\exists x_2\in\theta_2$, $S(x_1,y)$ and $S(x_2,y)$, which implies $\theta_1=\theta_2$ by Lemma \ref{LemProEmuUnDiv}, conflicts. Thus, by $\bigcup_{\theta\in\Theta_I^{\mathcal{A}_0}}S(\theta, \cdot)\subset E^\mathcal{C}$,
\begin{align*}
\sum_{\theta\in\Theta_I^{\mathcal{A}_0}}\abs{S(\theta, \cdot)}=\abs{E^\mathcal{C}} < \abs{E^\mathcal{B}} = \sum_{\theta\in\Theta_I^{\mathcal{A}_0}} \abs{G(\theta)}.
\end{align*}
Therefore, $\exists\theta\in \Theta_I^{\mathcal{A}_0}$, $S(\theta, \cdot)< G(\theta)$. By the Zig0 of $\mathcal{A}_0\leftrightarrows_P\mathcal{C}$,
\begin{align*}
\bigvee_{x\in E_0^{\mathcal{A}_0}\cap\theta}Pre^{\mathcal{A}_0}(x)\vDash \bigvee_{y\in E_0^\mathcal{C}\cap S(\theta,\cdot)}Pre^\mathcal{C}(y).
\end{align*}
It can be proved that
\begin{align}\label{EqCanRep}
\bigvee_{x\in E_0^{\mathcal{A}_0}\cap\theta}Pre^{\mathcal{A}_0}(x)\equiv\bigvee_{y\in E_0^\mathcal{C}\cap S(\theta,\cdot)}Pre^\mathcal{C}(y)
\end{align}
since otherwise, $\exists y\in E_0^\mathcal{C}\cap S(\theta,\cdot)$,
\begin{align*}
Pre^\mathcal{C}(y)\nvDash \bigvee_{x\in E_0^{\mathcal{A}_0}\cap\theta}Pre^{\mathcal{A}_0}(x).
\end{align*}
However, by the Zag0 of $\mathcal{A}_0\leftrightarrows_P\mathcal{C}$,
\begin{align*}
Pre^\mathcal{C}(y)\vDash \bigvee_{x\in E_0^{\mathcal{A}_0}\cap S(\cdot,y)}Pre^{\mathcal{A}_0}(x).
\end{align*}
Therefore, $S(\cdot,y)\setminus\theta\neq\emptyset$. Thus, $\exists\theta'\in\Theta_I^{\mathcal{A}_0}$ such that $\theta'\neq\theta$ and $y\in S(\theta,\cdot)\cap S(\theta',\cdot)$, conflicts. Because $G(\theta)$ is the minimal formula basis of $F(\theta)$, we have $\abs{S(\theta, \cdot)} < \abs{G(\theta)}$ implies that $\exists f\in F(\theta)$, $\forall E\subset S(\theta, \cdot)$, $f\not\equiv\bigvee_{y\in E}Pre^\mathcal{C}(y)$. By Eqs. \ref{EqCanRep}, $f\not\equiv\bigvee_{x\in E_0^{\mathcal{A}_0}\cap\theta}Pre^{\mathcal{A}_0}(x)$. Thus, $\exists a\in A$, $\exists\theta_0\in\Theta_I^{\mathcal{A}_0}$, $\exists x_0\in\theta_0$ such that $Pre^{\mathcal{A}_0}(x_0)\not\equiv\bot$, and $f\equiv \bigvee_{x\in\to_a^{\mathcal{A}_0}(\theta_0, \cdot)\cap\theta}Pre^{\mathcal{A}_0}(x)$. Because $x_0\in\theta_0$ and $\Theta_{I-1}^{\mathcal{A}_0}=\Theta_I^{\mathcal{A}_0}$, we have
\begin{align*}
f\equiv\bigvee_{x\in\to_a^{\mathcal{A}_0}(\theta_0, \cdot)\cap\theta}Pre^{\mathcal{A}_0}(x)\equiv\bigvee_{x\in\to_a^{\mathcal{A}_0}(x_0, \cdot)\cap\theta}Pre^{\mathcal{A}_0}(x).
\end{align*}
By Lemma \ref{LemRea2All}, $Pre^{\mathcal{A}_0}(x_0)\vDash \bigvee_{y\in S(x_0,\cdot)}Pre^\mathcal{C}(y)$, thereby $Pre^{\mathcal{A}_0}(x_0)\not\equiv\bot$ implies $S(x_0,\cdot)\neq\emptyset$. Let $y_0\in S(x_0, \cdot)$. By the Zig of $S(x_0, y_0)$,
\begin{align*}
f\equiv\bigvee_{x\in\to_a^{\mathcal{A}_0}(x_0, \cdot)\cap\theta}Pre^{\mathcal{A}_0}(x)\vDash\bigvee_{y\in S(\theta, \cdot)\cap \to_a^\mathcal{C}(y_0, \cdot)}Pre^\mathcal{C}(y).
\end{align*}
Since $\forall E\subset S(\theta, \cdot)$, $f\not\equiv\bigvee_{y\in E}Pre^\mathcal{C}(y)$, we have
\begin{align*}
\bigvee_{y\in S(\theta, \cdot)\cap \to_a^\mathcal{C}(y_0, \cdot)}Pre^\mathcal{C}(y)\nvDash f\equiv \bigvee_{x\in\to_a^{\mathcal{A}_0}(x_0, \cdot)\cap\theta}Pre^{\mathcal{A}_0}(x).
\end{align*}
Thus, $\exists y_1\in S(\theta, \cdot)\cap \to_a^\mathcal{C}(y_0, \cdot)$,
\begin{align*}
Pre^\mathcal{C}(y_1)\nvDash \bigvee_{x\in\to_a^{\mathcal{A}_0}(x_0, \cdot)\cap\theta}Pre^{\mathcal{A}_0}(x).
\end{align*}
By the Zag of $S(x_0, y_0)$,
\begin{align*}
Pre^\mathcal{C}(y_1)\vDash \bigvee_{x\in S(\cdot, y_1)\cap \to_a^{\mathcal{A}_0}(x_0, \cdot)}Pre^{\mathcal{A}_0}(x).
\end{align*}
Thus, $S(\cdot, y_1)\setminus\theta\neq\emptyset$. Therefore, $\exists'\in\Theta_I^{\mathcal{A}_0}$ such that $\theta'\neq\theta$ and $y_1\in S(\theta,\cdot)\cap S(\theta',\cdot)$, conflicts.
\end{proof}

\section{Minimize the event space of an action model under the action model equivalence}
\label{SecMTE}

\begin{definition}[canonical formula, Definitions 5 and 6 in \cite{FangLiu2019AI}]\label{DefCMAndCF}
Let $a\in A$ and $\Phi\subset\mathcal{L}$. The cover modality is defined as follows:
\begin{align*}
\nabla_a\Phi\coloneqq\square_a\qty(\bigvee\Phi)\land\qty(\bigwedge_{\phi\in\Phi}\lozenge_a\phi).
\end{align*}
Let $P\subset \mathcal{P}$ be finite. Define the set $\mathcal{E}_k^P$ of canonical formulas with depth $k$ and propositions $P$ as
\begin{align*}
\mathcal{E}_k^P\coloneqq\left\{\begin{array}{ll}
\set{\top}, & k=-1,\\
\set{\bigwedge P'\land \bigwedge_{p\in P\setminus P'}\lnot p \mid P'\subset P}, & k=0,\\
\set{\psi\land\bigwedge_{a\in A}\nabla_a\Phi_a \mid \psi\in \mathcal{E}_0^P, \Phi_a\subset \mathcal{E}_{k-1}^P}, & k>0.
\end{array}\right.
\end{align*}
\end{definition}

\begin{definition}
Let $\Phi\subset\mathcal{L}$. Define $\mathcal{P}\qty(\Phi)$ as the set of propositions in $\Phi$.
\end{definition}

\begin{proposition}[\cite{Moss2007JOPL}]\label{ProRegBasCF}
Let $k\ge 0$ and $\Phi\subset\set{\phi \mid \phi\in\mathcal{L},\delta(\phi)\le k}$ with $\abs{\Phi}<+\infty$. Then $\mathcal{E}_k^{\mathcal{P}(\Phi)}$ is $\Phi$-regular and $\Phi$-basis.
\end{proposition}

\begin{proposition}\label{ProDisInvCF}
Let $k\ge 0$ and $P\subset\mathcal{P}$ with $\abs{P}<+\infty$. Then
\begin{enumerate}
\item $\forall j\ge k$, $\mathcal{E}_k^P$ is $\mathcal{E}_j^P$-discrete.
\item $\forall j\ge k-1$, $\mathcal{E}_k^P$ is $\mathcal{E}_j^P$-known.
\end{enumerate}
\end{proposition}
\begin{proof}
Let $\phi\coloneqq\psi\land\bigwedge_{a\in A}\nabla_a\Phi_a\in\mathcal{E}_k^P$ with $\psi\in \mathcal{E}_0^P$ and $\Phi_a\subset\mathcal{E}_{k-1}^P$, $\xi_l^a\in\mathcal{L}$ for $1\le l\le L$ and $a\in A$. If $\phi\vDash\bigvee_{l=1}^L\bigwedge_{a\in A}\square_a\xi_l^a$, then since $\psi\in\mathcal{E}_0^P\subset\mathcal{L}_0$, we have that $\exists 1\le l'\le L$, $\phi\vDash\bigwedge_{a\in A}\square_a\qty(\bigvee\Phi_a) \vDash\bigwedge_{a\in A}\square_a\xi_{l'}^a$. Thus, $\forall j\ge k$, $\phi\vDash\mathcal{G}\qty(\bigwedge_{a\in A}\square_a\xi_{l'}^a, \mathcal{E}_j^P)$ since by Proposition \ref{ProRegBasCF}, $\delta(\phi)=k$ implies $\phi$ is $\mathcal{E}_j^P$-basis. Therefore, $\phi\vDash\bigvee_{l=1}^L\mathcal{G}\qty(\bigwedge_{a\in A}\square_a\xi_l^a, \mathcal{E}_j^P)$, {\it i.e.} $\mathcal{E}_k^P$ is $\mathcal{E}_j^P$-discrete.

Note that $D_{[]}^A\qty(\phi)=\set{\bigvee\Phi_a \mid a\in A}$.
\begin{enumerate}
\item If $k=0$, then $\Phi_a=\set{\top}$, thereby $D_{[]}^A\qty(\phi)=\set{\bigvee\Phi_a \mid a\in A}=\set{\top}=\mathcal{E}_{-1}^P$, {\it i.e.} $\mathcal{E}_0^P$ is $\mathcal{E}_{-1}^P$-known.
\item If $k>0$, then $\forall a\in A$, $\Phi_a\subset\mathcal{E}_{k-1}^P$ satisfies $\delta\qty(\Phi_a)=k-1$, thereby $\forall j\ge k-1$, $\mathcal{E}_j^P$ is $\Phi_a$-basis by Proposition \ref{ProRegBasCF}. Thus, $\exists \Psi\subset\mathcal{E}_j^P$, $\bigvee\Psi\equiv\bigvee\Phi_a$, {\it i.e.} $\mathcal{E}_k^P$ is $\mathcal{E}_j^P$-invariant.
\end{enumerate}
\end{proof}

\begin{corollary}
Let $\mathcal{A}$ and $\mathcal{B}$ be action models, $P\coloneqq\mathcal{P}\circ Pre^\mathcal{A}\qty(E^\mathcal{A})\cup\mathcal{P}\circ Pre^\mathcal{B}\qty(E^\mathcal{B})$, $k\ge 0$, $\Theta(x,y)=\mathcal{E}_k^P$ for $(x,y)\in E^\mathcal{A}\times E^\mathcal{B}$. If $\delta\circ Pre^\mathcal{A}\qty(E^\mathcal{A})\le k$ and $\delta\circ Pre^\mathcal{B}\qty(E^\mathcal{B})\le k$, then Algorithm \ref{DEOTMB} must return in finite steps, and it returns true iff $\mathcal{A}\leftrightarrows_G\mathcal{B}$ ({\it i.e.} $\mathcal{A}\equiv\mathcal{B}$).
\end{corollary}
\begin{proof}
This is obvious by Propositions \ref{ProFniteRetTrue2Equiv}, \ref{ProGenSetNece}, \ref{ProRegBasCF} and \ref{ProDisInvCF}.
\end{proof}

\begin{definition}\label{DefEquviRelCanForm}
Let $\mathcal{A}$ be an action model. $\forall x\in E^\mathcal{A}$, $\forall a\in A$, define
\begin{align*}
&\mathcal{R}_a^\mathcal{A}\qty(x)\coloneqq\set{y \mid y\in E^\mathcal{A},Pre^\mathcal{A}(x)\land\lozenge_a Pre^\mathcal{A}(y)\not\equiv\bot},\\
&\mathcal{Q}_a^\mathcal{A}\qty(x)\coloneqq\set{y \mid y\in\to_a^\mathcal{A}(x,\cdot),Pre^\mathcal{A}(x)\land\lozenge_a Pre^\mathcal{A}(y)\not\equiv\bot}.
\end{align*}
\end{definition}

\begin{lemma}\label{LemSepCovActMod}
Let $\mathcal{A}$ and $\mathcal{B}$ be action models. Let $[Pre](y)\subset E^\mathcal{A}$ for $y\in E^\mathcal{B}$ such that
\begin{enumerate}
\item $E_0^\mathcal{A}=\bigcup_{y\in E^\mathcal{B}_0}\qty[Pre](y)$.

\item $\forall y\in E^\mathcal{B}$, $\forall a\in A$,
\begin{align*}
&\bigcup_{x\in[Pre](y)}\mathcal{Q}_a^\mathcal{A}(x)\subset \bigcup_{y'\in \to_a^\mathcal{B}(y,\cdot)}[Pre](y'),\\
&\qty(\bigcup_{x\in[Pre](y)}\qty(\mathcal{R}_a^\mathcal{A}(x)\setminus\mathcal{Q}_a^\mathcal{A}(x)))\cap\qty(\bigcup_{y'\in \to_a^\mathcal{B}(y,\cdot)}[Pre](y'))=\emptyset.
\end{align*}

\item $\forall y\in E^\mathcal{B}$, $Pre^\mathcal{B}(y)=\bigvee_{x\in[Pre](y)}Pre^\mathcal{A}(x)$.
\end{enumerate}
Then $\mathcal{A}\equiv\mathcal{B}$.
\end{lemma}
\begin{proof}
$\forall (x,y)\in E^\mathcal{A}\times E^\mathcal{B}$, let
\begin{align*}
\eta(x,y)\coloneqq\left\{\begin{array}{ll}
Pre^\mathcal{A}(x), & x\in [Pre](y),\\
\bot, & x\notin [Pre](y).
\end{array}\right.
\end{align*}
We now prove that $\eta$ satisfies $\mathcal{A}\leftrightarrows_G\mathcal{B}$, thereby $\mathcal{A}\equiv\mathcal{B}$.
\begin{description}
\item[Zig0] Let $x_0\in E_0^\mathcal{A}=\bigcup_{y\in E^\mathcal{B}_0}[Pre](y)$. Then $\exists y_0\in E^\mathcal{B}_0$, $x_0\in [Pre](y_0)$, thereby $\eta(x_0,y_0)=Pre^\mathcal{A}(x_0)$. Thus,
\begin{align*}
Pre^\mathcal{A}(x_0)\equiv Pre^\mathcal{A}(x_0)\land\bigvee_{x\in[Pre](y_0)}Pre^\mathcal{A}(x) = \eta(x_0,y_0)\land Pre^\mathcal{B}(y_0) \vDash \bigvee_{y\in E_0^\mathcal{B}}\qty(\eta(x_0,y)\land Pre^\mathcal{B}(y)).
\end{align*}

\item[Zag0] Let $y_0\in E_0^\mathcal{B}$. Since $[Pre](y_0)\subset \bigcup_{y\in E^\mathcal{B}_0}\qty[Pre](y)=E_0^\mathcal{A}$, we have
\begin{align*}
Pre^\mathcal{B}(y_0)=\bigvee_{x\in[Pre](y_0)}Pre^\mathcal{A}(x)\equiv\bigvee_{x\in[Pre](y_0)}\qty(Pre^\mathcal{A}(x)\land\eta(x,y_0))\vDash \bigvee_{x\in E_0^\mathcal{A}}\qty(Pre^\mathcal{A}(x)\land\eta(x,y_0)).
\end{align*}

\item[Zig] Let $(x,y)\in E^\mathcal{A}\times E^\mathcal{B}$. Assume without loss of generality that $x\in [Pre](y)$ since otherwise, $\eta(x,y)=\bot$ is trival. Let $a\in A$ and $x'\in\to_a^\mathcal{A}(x,\cdot)$. Assume without loss of generality that $Pre^\mathcal{A}(x)\land\lozenge_a Pre^\mathcal{A}(x')\not\equiv\bot$ since otherwise, $\eta(x,y)= Pre^\mathcal{A}(x)\vDash \square_a\lnot Pre^\mathcal{A}(x')$ is trival. Then $x'\in\mathcal{Q}_a^\mathcal{A}(x)\subset\bigcup_{y'\in\to_a^\mathcal{B}(y,\cdot)}[Pre](y')$. Thus, $\exists y''\in \to_a^\mathcal{B}(y,\cdot)$, $x'\in [Pre](y'')$, which implies
\begin{align*}
Pre^\mathcal{A}(x')\equiv Pre^\mathcal{B}(y'')\land\eta(x',y'')\vDash \bigvee_{y'\in\to_a^\mathcal{B}(y,\cdot)} Pre^\mathcal{B}(y')\land\eta(x',y').
\end{align*}
Therefore,
\begin{align*}
\eta(x,y)\vDash \top\equiv\square_a\top\equiv\square_a\qty(Pre^\mathcal{A}(x')\to\bigvee_{y'\in\to_a^\mathcal{B}(y,\cdot)} Pre^\mathcal{B}(y')\land\eta(x',y')).
\end{align*}

\item[Zag] Let $(x,y)\in E^\mathcal{A}\times E^\mathcal{B}$ Assume without loss of generality that $x\in[Pre](y)$ since otherwise, $\eta(x,y)=\bot$ is trival. Let $a\in A$, $y'\in\to_a^\mathcal{B}(y,\cdot)$ and $x''\in [Pre](y')$. If $Pre^\mathcal{A}(x)\land\lozenge_a Pre^\mathcal{A}(x'')\equiv\bot$, then
\begin{align*}
\eta(x,y)= Pre^\mathcal{A}(x)\vDash \square_a\lnot Pre^\mathcal{A}(x'').
\end{align*}
Otherwise, $x''\in\mathcal{R}_a^\mathcal{A}(x)$. Because $\qty(\mathcal{R}_a^\mathcal{A}(x)\setminus\mathcal{Q}_a^\mathcal{A}(x))\cap [Pre](y')=\emptyset$, we have $x''\in \mathcal{Q}_a^\mathcal{A}(x)$, thereby $x''\in\to_a^\mathcal{A}(x)$. Since $x''\in[Pre](y')$, we have
\begin{align*}
Pre^\mathcal{A}(x'')\equiv Pre^\mathcal{A}(x'')\land\eta(x'',y')\vDash\bigvee_{x'\in\to_a^\mathcal{A}(x,\cdot)}\qty(Pre^\mathcal{A}(x')\land\eta(x',y')).
\end{align*}
Therefore,
\begin{align*}
\eta(x,y)=Pre^\mathcal{A}(x)\vDash\top\equiv\square_a\top\equiv\square_a\qty(Pre^\mathcal{A}(x'')\to \bigvee_{x'\in\to_a^\mathcal{A}(x,\cdot)}\qty(Pre^\mathcal{A}(x')\land\eta(x',y'))).
\end{align*}
Since $x''\in[Pre](y')$ is arbitrary, we have
\begin{align*}
\eta(x,y)\vDash\square_a\qty(Pre^\mathcal{B}(y')\to \bigvee_{x'\in\to_a^\mathcal{A}(x,\cdot)}\qty(Pre^\mathcal{A}(x')\land\eta(x',y'))).
\end{align*}
\end{description}
\end{proof}

\begin{definition}[bisimulation refinement]\label{DefBiSimEquivParti}
Let $\mathcal{A}$ be an action model. Let $\Omega_0^\mathcal{A}$ be the partition of $E^\mathcal{A}$ based on the equivalence relationship $\sim_0^\mathcal{A}$ such that $\forall x,y\in E^\mathcal{A}$, $x\sim_0^\mathcal{A}y$ iff $Pre^\mathcal{A}(x)\equiv Pre^\mathcal{B}(y)$. $\forall k\ge 0$, define the equivalence relationship $\sim_{k+1}^\mathcal{A}$ over $E^\mathcal{A}$ such that $\forall x,y\in E^\mathcal{A}$, $x\sim_{k+1}^\mathcal{A}y$ iff $x\sim_k^\mathcal{A}y$ and $\forall E'\in\Omega_k^\mathcal{A}$, $\forall a\in A$, $\mathcal{Q}_a^\mathcal{A}(x)\cap E'\neq\emptyset$ iff $\mathcal{Q}_a^\mathcal{A}(y)\cap E'\neq\emptyset$.

$\forall k\ge 0$, define the action model $[\mathcal{A}]_k$ that
\begin{enumerate}
\item $E^{[\mathcal{A}]_k}\coloneqq\Omega_k^\mathcal{A}$. $E_0^{[\mathcal{A}]_k}\coloneqq\set{E\in\Omega_k^\mathcal{A} \mid E\cap E_0^\mathcal{A}\neq\emptyset}$.
\item $\forall E\in\Omega_k^\mathcal{A}$, $Pre^{[\mathcal{A}]_k}(E)=Pre^\mathcal{A}(x)$, where $x\in E$ can be arbitrarily chosen since $\Omega_k^\mathcal{A}$ is denser than $\Omega_0^\mathcal{A}$.
\item $\forall E_1,E_2\in\Omega_k^\mathcal{A}$, $\forall a\in A$, $E_1\to_a^{[\mathcal{A}]_k}E_2$ iff $\mathcal{Q}_a^\mathcal{A}(E_1)\cap E_2\neq\emptyset$.
\end{enumerate}
\end{definition}

\begin{proposition}\label{ProBisimRefBisim}
Let $\mathcal{A}$ be an action model. Then $\exists k\ge 0$ such that $\Omega_k^\mathcal{A}=\Omega_{k+1}^\mathcal{A}$ and $\mathcal{A}\underline{\leftrightarrow}[\mathcal{A}]_k$.
\end{proposition}
\begin{proof}
Because $\Omega_k^\mathcal{A}$ for $k\ge 0$ form gradually denser partitions of $E^\mathcal{A}$, we have that $\Omega_k^\mathcal{A}$ converges since $\abs{E^\mathcal{A}}<+\infty$.

Let $(x, E_k)\in S\coloneqq\set{(x', E'_k)\in E^\mathcal{A}\times \Omega_k^\mathcal{A} \mid x'\in E'_k}$.
\begin{description}
\item[Invariance] $Pre^\mathcal{A}(x)\equiv Pre^{[\mathcal{A}]_k}(E_k)$ by the definition of $Pre^{[\mathcal{A}]_k}$.

\item[Zig] $\forall a\in A$, $\forall x'\in\mathcal{Q}_a^\mathcal{A}(x)$, since $\Omega_k^\mathcal{A}$ is a partition of $E^\mathcal{A}$, $\exists E_k'\in\Omega_k^\mathcal{A}$, $x'\in E_k'$, thereby $S\qty(x', E_k')$ and $E_k\to_a^{[\mathcal{A}]_k}E_k'$.

\item[Zag] $\forall a\in A$, $\forall E_k'\in\to_a^{[\mathcal{A}_k]}\qty(E_k,\cdot)$, $\mathcal{Q}_a^\mathcal{A}\qty(E_k)\cap E_k'\neq\emptyset$. Because $x\in E_k\in \Omega_k^\mathcal{A}=\Omega_{k+1}^\mathcal{A}$, we have that $\exists x'\in\mathcal{Q}_a^\mathcal{A}\qty(x)\cap E_k'\neq\emptyset$, thereby $x\to_a^\mathcal{A}x'$ and $S\qty(x', E_k')$.

\item[Zig0] $\forall x\in E_0^\mathcal{A}$, $\exists E_k\in\Omega_k^\mathcal{A}$, $x\in E_k$ (thereby $S\qty(x,E_k)$), which implies $E_k\cap E_0^\mathcal{A}\neq\emptyset$, thereby $E_k\in E_0^{[\mathcal{A}]_k}$.

\item[Zag0] $\forall E_k\in E_k\in E_0^{[\mathcal{A}]_k}$, $E_k\cap E_0^\mathcal{A}\neq\emptyset$. Thus, $\exists x\in E_k\cap E_0^\mathcal{A}$, thereby $S\qty(x,E_k)$.
\end{description}
\end{proof}

\begin{lemma}\label{LemUnDivdable}
Let $\mathcal{A}$ and $\mathcal{B}$ be action models such that $Pre^\mathcal{A}(E^\mathcal{A})$ and $Pre^\mathcal{B}(E^\mathcal{B})$ are $Pre^\mathcal{A}(E^\mathcal{A})$-regular. Let $\eta\qty(E^\mathcal{A}\times E^\mathcal{B})\subset\mathcal{L}$ satisfy $\mathcal{A}\leftrightarrows_G\mathcal{B}$ and be $Pre^\mathcal{A}(E^\mathcal{A})$-regular. Then $\forall x_1,x_2\in E^\mathcal{A}$ with $Pre^\mathcal{A}(x_1)\equiv Pre^\mathcal{A}(x_2)$, $\exists y\in E^\mathcal{B}$ such that $Pre^\mathcal{A}(x_1)\vDash \eta(x_1,y)$ and $Pre^\mathcal{A}(x_2)\vDash \eta(x_2,y)$ implies that $\forall k\ge 0$, $x_1\sim_k^\mathcal{A}x_2$.
\end{lemma}
\begin{proof}
Because $Pre^\mathcal{A}(x_1)\equiv Pre^\mathcal{A}(x_2)$, we have the claim for $k=0$. $\forall k\ge 0$, assume by induction that the claim is true for $k$. Now we prove that it is also true for $k+1$. 

Let $x_1,x_2\in E_k$ and $y\in E^\mathcal{B}$ satisfy the claim. $\forall E_k'\in\Omega_k^\mathcal{A}$. $\forall a\in A$, if $\exists x_1'\in E_k'$, $x_1\to_a^\mathcal{A}x_1'$, then by the Zig of $\mathcal{A}\leftrightarrows_G\mathcal{B}$,
\begin{align*}
Pre^\mathcal{A}(x_1)\vDash \eta(x_1,y)\vDash \square_a \qty(Pre^\mathcal{A}(x_1')\to \bigvee_{y'\in\to_a^\mathcal{B}(y,\cdot)}\qty(Pre^\mathcal{B}(y')\land\eta(x_1',y'))).
\end{align*}
If $Pre^\mathcal{A}(x_1)\land\lozenge_a Pre^\mathcal{A}(x_1')\not\equiv\bot$, then
\begin{align*}
Pre^\mathcal{A}(x_1')\land \bigvee_{y'\in\to_a^\mathcal{B}(y,\cdot)}\qty(Pre^\mathcal{B}(y')\land\eta(x_1',y'))\not\equiv\bot.
\end{align*}
Thus, $\exists y'\in\to_a^\mathcal{B}(y,\cdot)$, $Pre^\mathcal{A}(x_1')\land Pre^\mathcal{B}(y')\land\eta(x_1',y')\not\equiv\bot$. Because $Pre^\mathcal{B}(y')$ and $\eta(x_1',y')$ are $Pre^\mathcal{A}(E^\mathcal{A})$-regular, we have
\begin{align}\label{Eqx1ypimp}
Pre^\mathcal{A}(x_1')\vDash Pre^\mathcal{B}(y')\land\eta(x_1',y').
\end{align}
By the Zig of $\mathcal{A}\leftrightarrows_G\mathcal{B}$,
\begin{align*}
Pre^\mathcal{A}(x_2)\vDash\eta(x_2,y)\vDash \square_a \qty(Pre^\mathcal{B}(y')\to \bigvee_{x'\in\to_a^\mathcal{A}(x_2,\cdot)}\qty(Pre^\mathcal{A}(x')\land\eta(x',y'))).
\end{align*}
Because $Pre^\mathcal{A}(x_1)\equiv Pre^\mathcal{A}(x_2)$ and $Pre^\mathcal{A}(x_1)\land\lozenge_a Pre^\mathcal{A}(x_1')\not\equiv\bot$, we have $Pre^\mathcal{A}(x_2)\land\lozenge_a \qty(Pre^\mathcal{A}(x_1')\land Pre^\mathcal{B}(y'))\not\equiv\bot$, thereby
\begin{align*}
Pre^\mathcal{A}(x_1')\land Pre^\mathcal{B}(y')\land \bigvee_{x'\in\to_a^\mathcal{A}(x_2,\cdot)}\qty(Pre^\mathcal{A}(x')\land\eta(x',y'))\not\equiv\bot.
\end{align*}
Thus, $\exists x_2'\in\to_a^\mathcal{A}(x_2,\cdot)$, $Pre^\mathcal{A}(x_1')\land Pre^\mathcal{B}(y')\land Pre^\mathcal{A}(x_2')\land\eta(x_2',y')\not\equiv\bot$. Because $Pre^\mathcal{A}(x_1')$, $Pre^\mathcal{B}(y')$ and $\eta(x_2',y')$ are all $Pre^\mathcal{A}\qty(E^\mathcal{A})$-regular, we have $Pre^\mathcal{A}(x_2')\vDash Pre^\mathcal{A}(x_1')$ and
\begin{align}\label{Eqx2ypimp}
Pre^\mathcal{A}(x_2')\vDash Pre^\mathcal{B}(y')\land\eta(x_2',y').
\end{align}
Since $Pre^\mathcal{A}(x_2')$ is also $Pre^\mathcal{A}\qty(E^\mathcal{A})$-regular, we have that $Pre^\mathcal{A}(x_1')\nvDash \lnot Pre^\mathcal{A}(x_2')$ implies $Pre^\mathcal{A}(x_1')\vDash Pre^\mathcal{A}(x_2')$, thereby $Pre^\mathcal{A}(x_1')\equiv Pre^\mathcal{A}(x_2')$, which further implies $Pre^\mathcal{A}(x_2)\land\lozenge_a Pre^\mathcal{A}(x_2')\not\equiv\bot$ by $Pre^\mathcal{A}(x_1)\land\lozenge_a Pre^\mathcal{A}(x_1')\not\equiv\bot$. Together with Eqs. \eqref{Eqx1ypimp} and \eqref{Eqx2ypimp}, we have by induction that $x_2\in E_k'$.

By symmetry, if $\exists x_2'\in E_k'$ such that $x_2\to_a^\mathcal{A}x_2'$ and $Pre^\mathcal{A}(x_2)\land\lozenge_a Pre^\mathcal{A}(x_2')\not\equiv\bot$, then $\exists x_1'\in E_k'$ such that $x_1\to_a^\mathcal{A}x_1'$ and $Pre^\mathcal{A}(x_1')\equiv Pre^\mathcal{A}(x_2')$, which implies $Pre^\mathcal{A}(x_1)\land\lozenge_a Pre^\mathcal{A}(x_1')\not\equiv\bot$ by $Pre^\mathcal{A}(x_2)\land\lozenge_a Pre^\mathcal{A}(x_2')\not\equiv\bot$. By Definition \ref{DefBiSimEquivParti}, $x_1\sim_{k+1}^\mathcal{A}x_2$.
\end{proof}

\begin{lemma}\label{LemEquivToSepCover}
Let $\mathcal{A}$ and $\mathcal{B}$ be action models such that $Pre^\mathcal{A}(E^\mathcal{A})$ and $Pre^\mathcal{B}(E^\mathcal{B})$ are $Pre^\mathcal{A}(E^\mathcal{A})$-regular. $\forall x_1,x_2\in E^\mathcal{A}$, $\exists k\ge 0$, $x_1\nsim_k^\mathcal{A}x_2$. Let $\eta\qty(E^\mathcal{A}\times E^\mathcal{B})$ satisfy $\mathcal{A}\leftrightarrows_G\mathcal{B}$ and be $Pre^\mathcal{A}(E^\mathcal{A})$-regular. Then
\begin{align*}
[Pre](y)\coloneqq\set{x\in E^\mathcal{A} \mid Pre^\mathcal{A}(x)\vDash Pre^\mathcal{B}(y)\land \eta(x,y)}
\end{align*}
for $y\in E^\mathcal{B}$ satisfy the followings.
\begin{enumerate}
\item $E_0^\mathcal{A}=\bigcup_{y\in E^\mathcal{B}_0}\qty[Pre](y)$.

\item $\forall y\in E^\mathcal{B}$, $\forall a\in A$,
\begin{align*}
&\bigcup_{x\in[Pre](y)}\mathcal{Q}_a^\mathcal{A}(x)\subset \bigcup_{y'\in \to_a^\mathcal{B}(y,\cdot)}[Pre](y'),\\
&\qty(\bigcup_{x\in[Pre](y)}\qty(\mathcal{R}_a^\mathcal{A}(x)\setminus\mathcal{Q}_a^\mathcal{A}(x)))\cap\qty(\bigcup_{y'\in \to_a^\mathcal{B}(y,\cdot)}[Pre](y'))=\emptyset.
\end{align*}
\end{enumerate}
\end{lemma}
\begin{proof}
\begin{enumerate}
\item By the Zig0 of $\mathcal{A}\leftrightarrows_G\mathcal{B}$, $\forall x_0\in E_0^\mathcal{A}$, $Pre^\mathcal{A}(x_0)\vDash \bigvee_{y\in E_0^\mathcal{B}}\qty(Pre^\mathcal{B}(y)\land\eta(x_0,y))$. Thus, $\exists y_0\in E_0^\mathcal{B}$, $Pre^\mathcal{A}(x_0)\land Pre^\mathcal{B}(y_0)\land\eta(x_0,y_0)\not\equiv\bot$. Because both $Pre^\mathcal{B}(y_0)$ and $\eta(x_0,y_0)$ are $Pre^\mathcal{A}\qty(E^\mathcal{A})$-regular, we have $Pre^\mathcal{A}(x_0)\vDash Pre^\mathcal{B}(y_0)\land\eta(x_0,y_0)$, thereby $x_0\in [Pre](y_0)\subset\bigcup_{y\in E_0^\mathcal{B}}[Pre](y)$. Since $x_0\in E_0^\mathcal{A}$ is arbitrary, we have $E_0^\mathcal{A}\subset\bigcup_{y\in E_0^\mathcal{B}}[Pre](y)$.

By the Zag0 of $\mathcal{A}\leftrightarrows_G\mathcal{B}$, $\forall y_0\in E_0^\mathcal{B}$, $Pre^\mathcal{B}(y_0)\vDash \bigvee_{x\in E_0^\mathcal{A}}\qty(Pre^\mathcal{A}(x)\land\eta(x,y_0))$. $\forall x'\in [Pre](y_0)$, $Pre^\mathcal{A}(x')\vDash \eta(x',y_0)\land Pre^\mathcal{B}(y_0)$, thereby $Pre^\mathcal{A}(x')\vDash\bigvee_{x\in E_0^\mathcal{A}}\qty(Pre^\mathcal{A}(x)\land\eta(x,y_0))$. Thus, $\exists x_0\in E_0^\mathcal{A}$, $Pre^\mathcal{A}(x')\land Pre^\mathcal{A}(x_0)\land\eta(x_0,y_0)\not\equiv\bot$. Because both $Pre^\mathcal{A}(x_0)$ and $\eta(x_0,y_0)$ are $Pre^\mathcal{A}\qty(E^\mathcal{A})$-regular, we have $Pre^\mathcal{A}(x')\vDash Pre^\mathcal{A}(x_0)\land\eta(x_0,y_0)$. Since $Pre^\mathcal{A}(x')$ is $Pre^\mathcal{A}\qty(E^\mathcal{A})$-regular and $Pre^\mathcal{A}(x')\vDash Pre^\mathcal{A}(x_0)$ implies $Pre^\mathcal{A}(x_0)\nvDash\lnot Pre^\mathcal{A}(x')$, we have $Pre^\mathcal{A}(x_0)\vDash Pre^\mathcal{A}(x')$, thereby $Pre^\mathcal{A}(x_0)\equiv Pre^\mathcal{A}(x')$. If $x'\neq x_0$, then $\exists k\ge 0$, $x'\nsim_k^\mathcal{A}x_0$, which conflicts Lemma \ref{LemUnDivdable}. Thus, $x'=x_0\in E_0^\mathcal{A}$. Since $x'\in [Pre](y_0)$ and $y_0\in E_0^\mathcal{B}$ are both arbitrary, we have $\bigcup_{y_0\in E_0^\mathcal{B}}[Pre](y_0)\subset E_0^\mathcal{A}$.

\item $\forall y\in E^\mathcal{B}$, $\forall a\in A$, $\forall x\in[Pre](y)$, $\forall x'\in\mathcal{Q}_a^\mathcal{A}(x)$, by the Zig of $\mathcal{A}\leftrightarrows_G\mathcal{B}$,
\begin{align*}
Pre^\mathcal{A}(x)\vDash Pre^\mathcal{B}(y)\land\eta(x,y)\vDash\square_a\qty(Pre^\mathcal{A}(x')\to\bigvee_{y'\in\to_a^\mathcal{B}(y,\cdot)}\qty(Pre^\mathcal{B}(y')\land\eta(x',y'))).
\end{align*}
Because $Pre^\mathcal{A}(x)\land\lozenge_a Pre^\mathcal{A}(x')\not\equiv\bot$, we have
\begin{align*}
Pre^\mathcal{A}(x')\land\bigvee_{y'\in\to_a^\mathcal{B}(y,\cdot)}\qty(Pre^\mathcal{B}(y')\land\eta(x',y'))\not\equiv\bot.
\end{align*}
Therefore, $\exists y'\in\to_a^\mathcal{B}(y,\cdot)$, $Pre^\mathcal{A}(x')\land Pre^\mathcal{B}(y')\land\eta(x',y')\not\equiv\bot$. Because $Pre^\mathcal{B}(y')$ and $\eta(x',y')$ are $Pre^\mathcal{A}\qty(E^\mathcal{A})$-regular, we have $Pre^\mathcal{A}(x')\vDash Pre^\mathcal{B}(y')\land\eta(x',y')$, thereby $x'\in [Pre](y')\subset \bigcup_{y'\in\to_a^\mathcal{B}(y,\cdot)}[Pre](y')$.

$\forall y\in E^\mathcal{B}$, $\forall x\in [Pre](y)$, $\forall a\in A$, $\forall y'\in\to_a^\mathcal{B}(y,\cdot)$, $\forall x''\in [Pre](y')\cap\mathcal{R}_a^\mathcal{A}(x)$, we have $Pre^\mathcal{A}(x'')\vDash Pre^\mathcal{B}(y')$. By the Zig of $\mathcal{A}\leftrightarrows_G\mathcal{B}$,
\begin{align*}
Pre^\mathcal{A}(x)\vDash Pre^\mathcal{B}(y)\land\eta(x,y)\vDash\square_a\qty(Pre^\mathcal{A}(x'')\to\bigvee_{x'\in\to_a^\mathcal{A}(x,\cdot)}\qty(Pre^\mathcal{A}(x')\land\eta(x',y'))).
\end{align*}
Since $x''\in \mathcal{R}_a^\mathcal{A}(x)$ implies $Pre^\mathcal{A}(x)\land\lozenge_a Pre^\mathcal{A}(x'')\not\equiv\bot$, we have
\begin{align*}
\qty(Pre^\mathcal{A}(x'')\land\bigvee_{x'\in\to_a^\mathcal{A}(x,\cdot)}\qty(Pre^\mathcal{A}(x')\land\eta(x',y')))\not\equiv\bot.
\end{align*}
Thus, $\exists x'\in \to_a^\mathcal{A}(x,\cdot)$, $Pre^\mathcal{A}(x'')\land Pre^\mathcal{A}(x')\land\eta(x',y')\not\equiv\bot$. Because $Pre^\mathcal{A}(x')$ and $\eta(x',y')$ are $Pre^\mathcal{A}\qty(E^\mathcal{A})$-regular, we have $Pre^\mathcal{A}(x'')\vDash Pre^\mathcal{A}(x')\land\eta(x',y')$. Because $Pre^\mathcal{A}(x'')$ is also $Pre^\mathcal{A}\qty(E^\mathcal{A})$-regular and $Pre^\mathcal{A}(x'')\vDash Pre^\mathcal{A}(x')$ implies $Pre^\mathcal{A}(x')\nvDash\lnot Pre^\mathcal{A}(x'')$, we have $Pre^\mathcal{A}(x')\vDash Pre^\mathcal{A}(x'')$, thereby $Pre^\mathcal{A}(x')\equiv Pre^\mathcal{A}(x'')$. Thus, $Pre^\mathcal{A}(x)\land\lozenge_a Pre^\mathcal{A}(x')\not\equiv\bot$ (thereby $x'\in\mathcal{Q}_a^\mathcal{A}(x)$) and $Pre^\mathcal{A}(x')\vDash Pre^\mathcal{A}(y')\land\eta(x',y')$. If $x''\neq x'$, then $\exists k\ge 0$, $x''\nsim_k^\mathcal{A}x'$, which conflicts Lemma \ref{LemUnDivdable}. Thus, $x''=x'\in\mathcal{Q}_a^\mathcal{A}(x)$. Since $x''\in [Pre](y')\cap\mathcal{R}_a^\mathcal{A}(x)$ is arbitrary, we have $\qty(\mathcal{R}_a^\mathcal{A}(x)\setminus\mathcal{Q}_a^\mathcal{A}(x))\cap[Pre](y')=\emptyset$.
\end{enumerate}
\end{proof}

\begin{definition}\label{DefmathcalEkl}
Let $0\le k\le l$ and $P\subset\mathcal{P}$. Define $\mu_{k,l}^P:\mathcal{E}_k^P\mapsto\mathcal{E}_l^P$ such that $\forall \xi=\psi\land\bigwedge_{a\in A}\nabla_a\Phi^k_a \in \mathcal{E}_k^P$, $\mu_{k,l}^P(\xi)=\psi\land\bigwedge_{a\in A}\nabla_a\Phi^l_a$, where $\Phi^l_a\coloneqq\set{\phi\in\mathcal{E}_{l-1}^P \mid \phi\vDash\bigvee\Phi^k_a}$.
\end{definition}

\begin{theorem}\label{ThelkEquivTrans}
Let $\mathcal{A}$ and $\mathcal{B}$ be action models with $0\le k=\delta\circ Pre^\mathcal{A}(E^\mathcal{A})<\delta\circ Pre^\mathcal{B}(E^\mathcal{B})=l$. Let $P\coloneqq\mathcal{P}\circ Pre^\mathcal{A}\qty(E^\mathcal{A})\cup\mathcal{P}\circ Pre^\mathcal{B}\qty(E^\mathcal{B})$. Let $\mathcal{C}$ be an action model defined as following.
\begin{enumerate}
\item $E^\mathcal{C}=E^\mathcal{B}$ and $E_0^\mathcal{C}=E_0^\mathcal{B}$.
\item $\to^\mathcal{C}=\to^\mathcal{B}$.
\item $\forall y\in E^\mathcal{C}=E^\mathcal{B}$, $Pre^\mathcal{C}(y)\coloneqq \bigvee\set{\xi\in\mathcal{E}_k^P \mid \mu_{k,l}^P(\xi)\vDash Pre^\mathcal{B}(y)}$.
\end{enumerate}
If $\mathcal{A}\equiv\mathcal{B}$, then $\mathcal{A}\equiv\mathcal{C}$.
\end{theorem}
\begin{proof}
By Theorem \ref{TheEmuEquiv}, $\mathcal{A}\leftrightarrows_G\mathcal{B}$. By Propositions \ref{ProRegBasCF} and \ref{ProDisInvCF}, and Theorem \ref{TheGAEImpFourCriGAE}, $\exists\sigma_l(\cdot,\cdot)\subset\mathcal{E}_l^P$ such that $\bigvee\sigma_l(\cdot,\cdot)$ satisfies $\mathcal{A}\leftrightarrows_G\mathcal{B}$. Let $\sigma_k^P(\cdot,\cdot)\coloneqq\set{\xi\in\mathcal{E}_k^P \mid \mu_{k,l}^P(\xi)\in\sigma_l(\cdot,\cdot)}$. We now prove that $\bigvee\sigma_k(\cdot,\cdot)$ satisfies $\mathcal{A}\leftrightarrows_G\mathcal{C}$.
\begin{description}
\item[Zig] $\forall x\in E^\mathcal{A}$, $\forall y\in E^\mathcal{B}=E^\mathcal{C}$, $\forall \xi_k\in \sigma_k(x,y)$, we have $\mu_{k,l}^P\qty(\xi_k)\in\sigma_l(x,y)$. Thus, $\forall a\in A$, $\forall x'\in \to_a^\mathcal{A}(x,\cdot)$, by the Zig of $\mathcal{A}\leftrightarrows_G\mathcal{B}$, we have
\begin{align*}
\mu_{k,l}^P\qty(\xi_k)\vDash\square_a\qty(Pre^\mathcal{A}(x')\to\bigvee_{y'\in\to_a^\mathcal{B}(y,\cdot)}\qty(Pre^\mathcal{B}(y')\land\bigvee\sigma_l(x',y'))).
\end{align*}
By Definition \ref{DefmathcalEkl}, $D_1^a\circ \mu_{k,l}^P\qty(\xi_k)\equiv D_1^a(\xi_k)$. Thus,
\begin{align*}
D_1^a(\xi_k)\land Pre^\mathcal{A}(x')\vDash \bigvee_{y'\in\to_a^\mathcal{B}(y,\cdot)}\qty(Pre^\mathcal{B}(y')\land\bigvee\sigma_l(x',y')).
\end{align*}
$\forall \xi\in\mathcal{E}_k^P$ such that $\xi\vDash D_1^a(\xi_k)\land Pre^\mathcal{A}(x')$, we have
\begin{align*}
\mu_{k,l}^P(\xi)\vDash\xi\vDash \bigvee_{y'\in\to_a^\mathcal{B}(y,\cdot)}\qty(Pre^\mathcal{B}(y')\land\bigvee\sigma_l(x',y')).
\end{align*}
Therefore, $\exists y'\in\to_a^\mathcal{B}(y,\cdot)=\to_a^\mathcal{C}(y,\cdot)$, $\mu_{k,l}^P(\xi)\land Pre^\mathcal{B}(y')\land\bigvee\sigma_l(x',y')\not\equiv\bot$. Because $\delta\circ Pre^\mathcal{B}(y')\le l$ and $\sigma_l(x',y')\subset\mathcal{E}_l^P$, we have by Proposition \ref{ProRegBasCF} that $\mu_{k,l}^P(\xi)\vDash Pre^\mathcal{B}(y')$ and $\mu_{k,l}^P(\xi)\in \sigma_l(x',y')$. Therefore, $\xi\vDash Pre^\mathcal{C}(y')$ and $\xi\in\sigma_k(x',y')$. Thus,
\begin{align*}
\xi\vDash \bigvee_{y'\in\to_a^\mathcal{C}(y,\cdot)}\qty(Pre^\mathcal{C}(y')\land\bigvee\sigma_k(x',y')).
\end{align*}
Because $D_1^a(\xi_k)$ and $Pre^\mathcal{A}(x')$ are $\mathcal{E}_k^P$-regular and $\mathcal{E}_k^P$-basis, and $\xi\in\mathcal{E}_k^P$ is arbitrary, we have
\begin{align*}
D_1^a(\xi_k)\land Pre^\mathcal{A}(x')\vDash\bigvee_{y'\in\to_a^\mathcal{C}(y,\cdot)}\qty(Pre^\mathcal{C}(y')\land\bigvee\sigma_k(x',y')).
\end{align*}
Thus,
\begin{align*}
\xi_k\vDash\square_a\qty(Pre^\mathcal{A}(x')\to\bigvee_{y'\in\to_a^\mathcal{C}(y,\cdot)}\qty(Pre^\mathcal{C}(y')\land\bigvee\sigma_k(x',y'))).
\end{align*}
Because $\xi_k\in\sigma_k(x,y)$ is arbitrary, we have
\begin{align*}
\bigvee\sigma_k(x,y)\vDash \square_a\qty(Pre^\mathcal{A}(x')\to\bigvee_{y'\in\to_a^\mathcal{C}(y,\cdot)}\qty(Pre^\mathcal{C}(y')\land\bigvee\sigma_k(x',y'))).
\end{align*}

\item[Zag] $\forall x\in E^\mathcal{A}$, $\forall y\in E^\mathcal{B}=E^\mathcal{C}$, $\forall \xi_k\in\sigma_k(x,y)$, we have $\mu_{k,l}^P\qty(\xi_k)\in\sigma_l(x,y)$. Thus, $\forall a\in A$, $\forall y'\in\to_a^\mathcal{B}(y,\cdot)=\to_a^\mathcal{C}(y,\cdot)$, by the Zag of $\mathcal{A}\leftrightarrows_G\mathcal{B}$, we have
\begin{align*}
\mu_{k,l}^P\qty(\xi_k)\vDash\square_a\qty(Pre^\mathcal{B}(y')\to\bigvee_{x'\in\to_a^\mathcal{A}(x,\cdot)}\qty(Pre^\mathcal{A}(x')\land\bigvee\sigma_l(x',y'))).
\end{align*}
By Definition \ref{DefmathcalEkl}, $D_1^a\circ \mu_{k,l}^P\qty(\xi_k)\equiv D_1^a(\xi_k)$. Thus,
\begin{align*}
D_1^a(\xi_k)\land Pre^\mathcal{B}(y')\vDash\bigvee_{x'\in\to_a^\mathcal{A}(x,\cdot)}\qty(Pre^\mathcal{A}(x')\land\bigvee\sigma_l(x',y')).
\end{align*}
$\forall \xi\in\mathcal{E}_k^P$ such that $\xi\vDash D_1^a(\xi_k)\land Pre^\mathcal{C}(y')$, we have $\mu_{k,l}^P(\xi)\vDash\xi\vDash D_1^a(\xi_k)$ and $\mu_{k,l}^P(\xi)\vDash Pre^\mathcal{B}(y')$, thereby
\begin{align*}
\mu_{k,l}^P(\xi)\vDash\bigvee_{x'\in\to_a^\mathcal{A}(x,\cdot)}\qty(Pre^\mathcal{A}(x')\land\bigvee\sigma_l(x',y')).
\end{align*}
Thus, $\exists x'\in\to_a^\mathcal{A}(x,\cdot)$, $\mu_{k,l}^P(\xi)\land Pre^\mathcal{A}(x')\land\bigvee\sigma_l(x',y')\not\equiv\bot$. Because $\delta\circ Pre^\mathcal{A}(x')\le k$ and $\sigma_l(x',y')\subset\mathcal{E}_l^P$, we have by Proposition \ref{ProRegBasCF} that $\mu_{k,l}^P(\xi)\vDash Pre^\mathcal{A}(x')$ and $\mu_{k,l}^P(\xi)\in\sigma_l(x',y')$. Therefore, $\xi\vDash Pre^\mathcal{A}(x')$ and $\xi\in\sigma_k(x',y')$. Thus,
\begin{align*}
\xi\vDash\bigvee_{x'\in\to_a^\mathcal{A}(x,\cdot)}\qty(Pre^\mathcal{A}(x')\land\bigvee\sigma_k(x',y')).
\end{align*}
Because $D_1^a(\xi_k)$ and $Pre^\mathcal{C}(y')$ are $\mathcal{E}_k^P$-regular and $\mathcal{E}_k^P$-basis, and $\xi\in\mathcal{E}_k^P$ is arbitrary, we have
\begin{align*}
D_1^a(\xi_k)\land Pre^\mathcal{C}(y')\vDash\bigvee_{x'\in\to_a^\mathcal{A}(x,\cdot)}\qty(Pre^\mathcal{A}(x')\land\bigvee\sigma_k(x',y')).
\end{align*}
Thus,
\begin{align*}
\xi_k\vDash\square_a\qty(Pre^\mathcal{C}(y')\to\bigvee_{x'\in\to_a^\mathcal{A}(x,\cdot)}\qty(Pre^\mathcal{A}(x')\land\bigvee\sigma_k(x',y'))).
\end{align*}
Because $\xi_k\in\sigma_k(x,y)$ is arbitrary, we have
\begin{align*}
\bigvee\sigma_k(x,y)\vDash \square_a\qty(Pre^\mathcal{C}(y')\to\bigvee_{x'\in\to_a^\mathcal{A}(x,\cdot)}\qty(Pre^\mathcal{A}(x')\land\bigvee\sigma_k(x',y'))).
\end{align*}

\item[Zig0] $\forall x_0\in E_0^\mathcal{A}$, $\forall \xi_k\in\mathcal{E}_k^P$ such that $\xi_k\vDash Pre^\mathcal{A}(x_0)$, we have $\mu_{k,l}^P\vDash \xi_k\vDash (\xi_k)\vDash Pre^\mathcal{A}(x_0)$. Thus, by the Zig0 of $\mathcal{A}\leftrightarrows_G\mathcal{B}$, we have
\begin{align*}
\mu_{k,l}^P(\xi_k)\vDash\bigvee_{y\in E_0^\mathcal{B}}\qty(Pre^\mathcal{B}(y)\land\bigvee\sigma_l(x_0,y)).
\end{align*}
Thus, $\exists y\in E_0^\mathcal{B}=E_0^\mathcal{C}$, $\mu_{k,l}^P(\xi_k)\land Pre^\mathcal{B}(y)\land\bigvee\sigma_l(x_0,y)\not\equiv\bot$. Because $\delta\circ Pre^\mathcal{B}(y)\le l$ and $\sigma_l(x_0,y)\subset \mathcal{E}_l^P$, we have $\mu_{k,l}^P(\xi_k)\vDash Pre^\mathcal{B}(y)$ and $\mu_{k,l}^P(\xi_k)\in \sigma_l(x_0,y)$, thereby $\xi_k\vDash Pre^\mathcal{C}(y)$ and $\xi_k\in \sigma_k(x_0,y)$. Thus,
\begin{align*}
\xi_k\vDash\bigvee_{y\in E_0^\mathcal{C}}\qty(Pre^\mathcal{C}(y)\land\bigvee\sigma_k(x_0,y)).
\end{align*}
Because $\delta\circ Pre^\mathcal{A}(x_0)\le k$ and $\xi_k\in\mathcal{E}_k^P$ such that $\xi_k\vDash Pre^\mathcal{A}(x_0)$ is arbitrary, we have by Proposition \ref{ProRegBasCF} that
\begin{align*}
Pre^\mathcal{A}(x_0)\vDash \bigvee_{y\in E_0^\mathcal{C}}\qty(Pre^\mathcal{C}(y)\land\bigvee\sigma_k(x_0,y)).
\end{align*}

\item[Zag0] $\forall y_0\in E_0^\mathcal{B}=E_0^\mathcal{C}$, $\forall \xi_k\in\mathcal{E}_k^P$ such that $\xi_k\vDash Pre^\mathcal{C}(y_0)$, we have $\mu_{k,l}^P(\xi_k)\vDash \xi_k\vDash Pre^\mathcal{B}(y_0)$. Thus, by the Zag0 of $\mathcal{A}\leftrightarrows_G\mathcal{B}$, we have
\begin{align*}
\mu_{k,l}^P(\xi_k)\vDash\bigvee_{x\in E_0^\mathcal{A}}\qty(Pre^\mathcal{A}(x)\land\bigvee\sigma_l(x,y_0)).
\end{align*}
Thus, $\exists x\in E_0^\mathcal{A}$, $\mu_{k,l}^P(\xi_k)\land Pre^\mathcal{A}(x)\land\bigvee\sigma_l(x,y_0)\not\equiv\bot$. Because $\delta\circ Pre^\mathcal{A}(x)\le k$ and $\sigma_l(x,y_0)\subset \mathcal{E}_l^P$, we have $\mu_{k,l}^P(\xi_k)\vDash Pre^\mathcal{A}(x)$ and $\mu_{k,l}^P(\xi_k)\in \sigma_l(x,y_0)$, thereby $\xi_k\vDash Pre^\mathcal{A}(x)$ and $\xi_k\in \sigma_k(x,y_0)$. Thus,
\begin{align*}
\xi_k\vDash\bigvee_{x\in E_0^\mathcal{A}}\qty(Pre^\mathcal{A}(x)\land\bigvee\sigma_k(x,y_0)).
\end{align*}
Because $\delta\circ Pre^\mathcal{C}(y_0)\le k$ and $\xi_k\in\mathcal{E}_k^P$ such that $\xi_k\vDash Pre^\mathcal{C}(y_0)$ is arbitrary, we have by Proposition \ref{ProRegBasCF} that
\begin{align*}
Pre^\mathcal{C}(y_0)\vDash \bigvee_{x\in E_0^\mathcal{A}}\qty(Pre^\mathcal{A}(x)\land\bigvee\sigma_k(x,y_0)).
\end{align*}
\end{description}
Finally, we have $\mathcal{A}\equiv\mathcal{C}$ by Theorem \ref{TheEmuEquiv}.
\end{proof}

\begin{definition}[regular action model]\label{DefRegActMod}
Let $\mathcal{A}$ be an action model and $\Phi\subset\mathcal{L}$. Define the regular version $\mathcal{A}^r$ of $\mathcal{A}$ over $\Phi$ as follows.
\begin{enumerate}
\item $E^{\mathcal{A}^r}\coloneqq\set{(x,\phi) \mid x\in E^\mathcal{A}, \phi\in\Phi, \phi\vDash Pre^\mathcal{A}(x)}$.
\item $\forall (x,\phi)\in E^{\mathcal{A}^r}$, $Pre^{\mathcal{A}^r}\qty(x,\phi)\coloneqq\phi$.
\item $\forall (x,\phi),(x,\phi')\in E^{\mathcal{A}^r}$, $\forall a\in A$, $(x,\phi)\to_a^{\mathcal{A}^r}(x',\phi')$ iff $x\to_a^\mathcal{A}x'$ and $\phi\land\lozenge_a\phi'\not\equiv\bot$.
\item $E_0^{\mathcal{A}^r}\coloneqq\set{(x,\phi)\in E^{\mathcal{A}^r} \mid x\in E_0^\mathcal{A}}$.
\end{enumerate}
\end{definition}

\begin{remark}
The canonical version $\mathcal{A}^c$ of an action model $\mathcal{A}$ is the regular version of $\mathcal{A}$ over $\Gamma\circ K\circ C\circ Pre^\mathcal{A}\qty(E^\mathcal{A})$.
\end{remark}

\begin{proposition}[resembling Theorem 4 in \cite{SietsmaEijckJOPL2013}]\label{ProCanActEquiv}
Let $\mathcal{A}$ be an action model and $k\ge 0$ such that $\delta\circ Pre^\mathcal{A}(E^\mathcal{A})\le k$ and $P\coloneqq\mathcal{P}\circ Pre^\mathcal{A}\qty(E^\mathcal{A})$. Let $\mathcal{A}^r$ be the regular version of $\mathcal{A}$ over $\mathcal{E}_k^P$. Then $\mathcal{A}\equiv \mathcal{A}^r$.
\end{proposition}

\begin{theorem}\label{TheCoverNonInterEquivCond}
Let $\mathcal{A}$ be an action model, $k\coloneqq\delta\circ Pre^\mathcal{A}(E^\mathcal{A})$ and $P\coloneqq\mathcal{P}\circ Pre^\mathcal{A}\qty(E^\mathcal{A})$. Let $\mathcal{A}^r$ be the regular version of $\mathcal{A}$ over $\mathcal{E}_k^P$. Let $[\mathcal{A}^r]$ be the limit of the bisimulation refinements of $\mathcal{A}^r$ (see Proposition \ref{ProBisimRefBisim}). Let $\mathcal{B}$ be an action model such that $\exists [Pre](y)\subset E^{[\mathcal{A}^r]}$ for $y\in E^\mathcal{B}$ satisfying the followings.
\begin{enumerate}
\item $E_0^{[\mathcal{A}^r]}=\bigcup_{y\in E^\mathcal{B}_0}\qty[Pre](y)$.

\item $\forall y\in E^\mathcal{B}$, $\forall a\in A$,
\begin{align*}
&\bigcup_{x\in[Pre](y)}\mathcal{Q}_a^{[\mathcal{A}^r]}(x)\subset \bigcup_{y'\in \to_a^\mathcal{B}(y,\cdot)}[Pre](y'),\\
&\qty(\bigcup_{x\in[Pre](y)}\qty(\mathcal{R}_a^{[\mathcal{A}^r]}(x)\setminus\mathcal{Q}_a^{[\mathcal{A}^r]}(x)))\cap\qty(\bigcup_{y'\in \to_a^\mathcal{B}(y,\cdot)}[Pre](y'))=\emptyset.
\end{align*}

\item $\forall y\in E^\mathcal{B}$, $Pre^\mathcal{B}(y)=\bigvee_{x\in[Pre](y)}Pre^{[\mathcal{A}^r]}(x)$.
\end{enumerate}
Then $\mathcal{A}\equiv\mathcal{B}$. If $\mathcal{B}$ has the minimal event space among all action models satisfying the above three conditions, then $\mathcal{B}$ has the minimal event space among all action models equivalent to $\mathcal{A}$.
\end{theorem}
\begin{proof}
By Propositions \ref{ProBisimRefBisim} and \ref{ProCanActEquiv} and Lemma \ref{LemSepCovActMod}, $\mathcal{A}\equiv\mathcal{A}^r\equiv[\mathcal{A}^r]\equiv\mathcal{B}$. Let $\mathcal{C}$ be another action model such that $\mathcal{A}\equiv\mathcal{C}$. Assume without loss of generality that $\mathcal{P}\circ Pre^\mathcal{C}(E^\mathcal{C})=P$. By Theorem \ref{ThelkEquivTrans}, assume without loss of generality that $\delta\circ Pre^\mathcal{C}\qty(E^\mathcal{C})\le k=\delta\circ Pre^\mathcal{A}(E^\mathcal{A})$. By Proposition \ref{ProRegBasCF}, $Pre^\mathcal{C}\qty(E^\mathcal{C})$ is $\mathcal{E}_k^P$-regular. By Definitions \ref{DefBiSimEquivParti} and \ref{DefRegActMod}, $Pre^{[\mathcal{A}^r]}\qty(E^{[\mathcal{A}^r]})\subset Pre^{\mathcal{A}^r}\qty(E^{\mathcal{A}^r})\subset\mathcal{E}_k^P$, thereby $Pre^\mathcal{C}\qty(E^\mathcal{C})$ is $Pre^{[\mathcal{A}^r]}\qty(E^{[\mathcal{A}^r]})$-regular. Since $Pre^{[\mathcal{A}^r]}\qty(E^{[\mathcal{A}^r]})\subset\mathcal{E}_k^P$ implies $\delta\circ Pre^{[\mathcal{A}^r]}\qty(E^{[\mathcal{A}^r]})=k$, we have by Proposition \ref{ProRegBasCF} that $Pre^{[\mathcal{A}^r]}\qty(E^{[\mathcal{A}^r]})$ is $Pre^{[\mathcal{A}^r]}\qty(E^{[\mathcal{A}^r]})$-regular. By Propositions \ref{ProBisimRefBisim} and \ref{ProCanActEquiv}, $\mathcal{C}\equiv \mathcal{A}\equiv\mathcal{A}^r\equiv[\mathcal{A}^r]$. By Theorems \ref{TheEmuEquiv} and \ref{TheGAEImpFourCriGAE}, and Propositions \ref{ProRegBasCF} and \ref{ProDisInvCF}, $\exists \sigma(x,y)\subset\mathcal{E}_k^P$ for $x\in E^{[\mathcal{A}^r]}$ and $y\in E^\mathcal{C}$ such that $\eta(\cdot,\cdot)=\bigvee\sigma(\cdot,\cdot)$ satisfies $\mathcal{A}\leftrightarrows_G\mathcal{C}$. Since $\eta\qty(E^{[\mathcal{A}^r]},E^\mathcal{C})$ are $\mathcal{E}_k^P$-regular (thereby $Pre^{[\mathcal{A}^r]}\qty(E^{[\mathcal{A}^r]})$-regular), we have by Lemma \ref{LemEquivToSepCover} that $\exists [Pre]'(y)$ for $y\in E^\mathcal{C}$ satisfying the followings.
\begin{enumerate}
\item $E_0^{[\mathcal{A}^r]}=\bigcup_{y\in E^\mathcal{C}_0}\qty[Pre]'(y)$.

\item $\forall y\in E^\mathcal{C}$, $\forall a\in A$,
\begin{align*}
&\bigcup_{x\in[Pre]'(y)}\mathcal{Q}_a^{[\mathcal{A}^r]}(x)\subset \bigcup_{y'\in \to_a^\mathcal{C}(y,\cdot)}[Pre]'(y'),\\
&\qty(\bigcup_{x\in[Pre]'(y)}\qty(\mathcal{R}_a^{[\mathcal{A}^r]}(x)\setminus\mathcal{Q}_a^{[\mathcal{A}^r]}(x)))\cap\qty(\bigcup_{y'\in \to_a^\mathcal{C}(y,\cdot)}[Pre]'(y'))=\emptyset.
\end{align*}
\end{enumerate}
If one modifies $Pre^\mathcal{C}$ such that $\forall y\in E^\mathcal{C}$, $Pre^\mathcal{C}(y)=\bigvee_{x\in [Pre]'(y)}Pre^{[\mathcal{A}^r]}(x)$, then the above two conditions are still valid, thereby the modified $\mathcal{C}$ satisfies all three conditions. Thus, $\abs{E^\mathcal{B}}\le\abs{E^\mathcal{C}}$. Since the action model $\mathcal{C}$ such that $\mathcal{A}\equiv\mathcal{C}$ is arbitrary, we have the result.
\end{proof}

\begin{remark}
Under the conditions of Theorem \ref{TheCoverNonInterEquivCond}, the key step to find an action model $\mathcal{B}$ with the minimal event space equivalent to $\mathcal{A}$ is to determine $[Pre](y)\subset E^{\qty[\mathcal{A}^r]}$ for $y\in E^\mathcal{B}$. There are only finite many choices for $[Pre](y)$ since $\abs{E^\mathcal{B}}\le \abs{E^\mathcal{A}}$. Thus, we now have a computable method to minimize the event space of an action model. 
\end{remark}

\section{The complexities of the minimizing problems}
\label{SecNP}

\begin{problem}\label{BlemSet}
Let $F\subset\mathcal{L}^M$. Does there exist a formula basis $G$ of $F$ with $\abs{G}\le K$ for an integer $K>0$?
\end{problem}

\begin{proposition}\label{NPhardFormulaBasis}
Problem \ref{BlemSet} is PSPACE-complete of $\norm{F}\coloneqq \sum_{f\in F}\abs{f}$, where $\abs{f}$ is the string length of the formula $f$.
\end{proposition}
\begin{proof}
\begin{enumerate}
\item By Theorems 6.47 and 6.50 (Ladner’s Theorem) in \cite{BlackburnCUP2001}, the satisfiablity check of $\mathcal{L}$ (over Kripke models) is PSPACE-complete. Thus, there exists a formula $f_0\in\mathcal{L}$ whose satisfiablity check is PSPACE-complete of $\abs{f_0}$. Let $F\coloneqq\set{p, p\land f_0}$ in Problem \ref{BlemSet}, where $p\in \mathcal{P}\setminus\mathcal{P}\qty(f_0)$. Then Problem \ref{BlemSet} for $K=1$ is equivalent to the satisfiablity check of $f_0$, which is PSPACE-complete of $\abs{f_0}$, thereby PSPACE-complete of $\norm{F}=\abs{p}+\abs{p\land f_0}=\abs{f_0}+3$. This proves that Problem \ref{BlemSet} is PSPACE-hard of $\norm{F}$.

\item The case for $K\ge\abs{F}$ is trival since $G=F$ is a solution. $\forall K <\abs{F}$, if there exists a formula basis $G$ of $F$ such that $\abs{G}\le K$, then $\exists\xrightarrow{b}\subset F\times G$, $\forall f\in F$, $f\equiv\bigvee_{g\in\xrightarrow{b}(f,\cdot)}g$. Note that $G'=\set{\bigwedge_{f\in\xrightarrow{b}(\cdot,g)}f \mid g\in G}$ is also a formula basis of $F$ because
\begin{align*}
f\equiv\bigvee_{g\in\xrightarrow{b}(f,\cdot)}g\vDash \bigvee_{g\in\xrightarrow{b}(f,\cdot)}\bigwedge_{f'\in\xrightarrow{b}(\cdot,g)}f'\vDash f.
\end{align*}
Thus, Problem \ref{BlemSet} is equivalent to find $\mathcal{F}\subset 2^F$ such that $\forall f\in F$, $f\equiv\bigvee_{F'\in\mathcal{F}:f\in F'}\bigwedge F'$ and $\abs{\mathcal{F}}\le K$. For each choice of $\mathcal{F}$ and each $f\in F$, we check the satisfiability of $\bigwedge F'\land\bigwedge_{f'\in F\setminus F'}\lnot f'$ (this can be done in polynomial space of $\norm{F}$), where $F'\subset F$ is arbitrary as long as $f\in F'$ and $ F'\notin\mathcal{F}$. If all checks are negative for each $f\in F$ and the corresponding $F'$, {\it i.e.} $\bigwedge F'\land\bigwedge_{f'\in F\setminus F'}\lnot f'\equiv\bot$, then
\begin{align*}
f=\bigvee_{F'\subset F:f\in F'}\qty(\bigwedge F'\land\bigwedge_{f'\in F\setminus F'}\lnot f')=\bigvee_{F'\subset F:f\in F',F'\in\mathcal{F}}\qty(\bigwedge F'\land\bigwedge_{f'\in F\setminus F'}\lnot f')=\bigvee_{F'\in\mathcal{F}:f\in F'}\bigwedge F'.
\end{align*}
Thus, we have found a formula basis $\set{\bigwedge F' \mid F'\in\mathcal{F}}$ of $F$ whose size is less than $K$ . For all possible choices of $\mathcal{F}$, the above check can be done in polynomial space of $\norm{F}$. Thus, Problem \ref{BlemSet} is PSPACE of $\norm{F}$.
\end{enumerate}
\end{proof}

\begin{problem}\label{BlemPropEmu}
Let $\mathcal{A}$ be an action model. Does there exist an action model $\mathcal{B}$ such that $\mathcal{A}\leftrightarrows_P\mathcal{B}$ and $\abs{E^\mathcal{B}}\le K$ for an integer $K>0$?
\end{problem}

\begin{proposition}\label{ProPropEmuNPhard}
Problem \ref{BlemPropEmu} is PSPACE-hard of $\norm{Pre^\mathcal{A}\qty(E^\mathcal{A})}+\norm{\to^\mathcal{A}}$, where $\norm{\to^\mathcal{A}}\coloneqq\sum_{a\in A}\abs{\to_a^\mathcal{A}}$.
\end{proposition}
\begin{proof}
Let $F\subset\mathcal{L}^M$ and $a\in A$. Define an action model $\mathcal{A}$ as following.
\begin{enumerate}
\item $E^\mathcal{A}\coloneqq\set{x_1(f) \mid f\in F}\cup\set{x_0(f) \mid f\in F}$. $E^\mathcal{A}_0\coloneqq\set{x_0(f) \mid f\in F}$.

\item $\forall f\in F$, $Pre^\mathcal{A}\circ x_0(f)\coloneqq\top$ and $Pre^\mathcal{A}\circ x_1(f)\coloneqq f$.

\item $\to_a^\mathcal{A}\coloneqq\set{\qty(x_0(f),x_1(f)) \mid f\in F}$. $\forall a'\in A$ with $a'\neq a$, $\to_{a'}^\mathcal{A}\coloneqq\emptyset$.
\end{enumerate}
Let $\mathcal{B}$ be an action model with $\abs{E^\mathcal{B}}\le K$ and $\mathcal{A}\leftrightarrows_P\mathcal{B}$. Let $S\subset E^\mathcal{A}\times E^\mathcal{B}$ satisfy $\mathcal{A}\leftrightarrows_P\mathcal{B}$. By the Zig0 of  $\mathcal{A}\leftrightarrows_P\mathcal{B}$, $\forall f\in F$, $\exists y_0(f)\in E_0^\mathcal{B}$, $S\qty(x_0(f),y_0(f))$. Because $\to_a^\mathcal{A}\qty(x_0(f),\cdot)=\set{x_1(f)}$, we have by the Zig and Zag of $\mathcal{A}\leftrightarrows_P\mathcal{B}$ that
\begin{align*}
f\equiv Pre^\mathcal{A}\circ x_1(f)\equiv \bigvee_{y\in\to_a^\mathcal{B}\qty(y_0(f),\cdot)\cap S(x_1(f),\cdot)}Pre^\mathcal{B}(y).
\end{align*}
Since $f\in F$ is arbitrary, we have that $Pre^\mathcal{B}\qty(E^\mathcal{B})$ is a formula basis of $F$ with $\abs{Pre^\mathcal{B}\qty(E^\mathcal{B})}=\abs{E^\mathcal{B}}\le K$.

In summary, we have transformed Problem \ref{BlemSet} to Problem \ref{BlemPropEmu} in polynomial time, thereby only using polynomial additional space. Since Problem \ref{BlemSet} is PSPACE-complete by Proposition \ref{NPhardFormulaBasis}, we have that Problem \ref{BlemPropEmu} is PSPACE-hard.
\end{proof}

\begin{proposition}\label{ProPropEmuPSPACE}
Problem \ref{BlemPropEmu} is PSPACE of $\norm{Pre^\mathcal{A}\qty(E^\mathcal{A})}+\norm{\to^\mathcal{A}}$.
\end{proposition}
\begin{proof}
Note that each step of Algorithm \ref{PartitionRefinement}, including the parition refinements involving the satisfiability checks and the search of the minimal formula basis, can be finished in polynomial space of $\norm{Pre^\mathcal{A}\qty(E^\mathcal{A})}+\norm{\to^\mathcal{A}}$.
\end{proof}

\begin{corollary}
Problem \ref{BlemPropEmu} is PSPACE-complete of $\norm{Pre^\mathcal{A}\qty(E^\mathcal{A})}+\norm{\to^\mathcal{A}}$.
\end{corollary}

\begin{problem}\label{BlemEmu}
Let $\mathcal{A}$ be an action model. Does there exist an action model $\mathcal{B}$ such that $\mathcal{A}\equiv\mathcal{B}$ and $\abs{E^\mathcal{B}}\le K$ for an integer $K>0$?
\end{problem}

\begin{proposition}\label{ProEmuNPhard}
Problem \ref{BlemEmu} is PSPACE-hard of $\norm{Pre^\mathcal{A}\qty(E^\mathcal{A})}+\norm{\to^\mathcal{A}}$.
\end{proposition}
\begin{proof}
Let $\mathcal{A}$ be an action model as in Proposition \ref{ProPropEmuNPhard}. Let $\mathcal{B}$ be an action model such that $\mathcal{A}\equiv\mathcal{B}$ and $\abs{E^\mathcal{B}}\le K$. Let $P\coloneqq\mathcal{P}\circ Pre^\mathcal{A}\qty(E^\mathcal{A})\cup\mathcal{P}\circ Pre^\mathcal{A}\qty(E^\mathcal{A})$. Let $k\ge 0$ such that $\delta\circ Pre^\mathcal{A}\qty(E^\mathcal{A})\le k$ and $\delta\circ Pre^\mathcal{B}\qty(E^\mathcal{B})\le k$. By Theorems \ref{TheEmuEquiv} and \ref{TheGAEImpFourCriGAE}, and Propositions \ref{ProRegBasCF} and \ref{ProDisInvCF}, $\exists \sigma(x,y)\subset\mathcal{E}_k^P$ for $(x,y)\in E^\mathcal{A}\times E^\mathcal{B}$ such that $\eta(\cdot,\cdot)=\bigvee\sigma(\cdot,\cdot)$ satisfies $\mathcal{A}\leftrightarrows_G\mathcal{B}$. Let $P'\subset P$ and $f_0\coloneqq\bigwedge P'\land\bigwedge_{p\in P\setminus P'}\lnot p\land\nabla_a\mathcal{E}_{k-1}^P\in \mathcal{E}_k^P$. Then $f_0\vDash Pre^\mathcal{A}\circ x_0(f)$ for all $f\in F$ since $Pre^\mathcal{A}\circ x_0(f)=\top$. By the Zig0 of $\mathcal{A}\leftrightarrows_G\mathcal{B}$, $\forall f\in F$,
\begin{align*}
f_0\vDash\top\equiv Pre^\mathcal{A}\circ x_0(f)\vDash \bigvee_{y\in E_0^B}\qty(Pre^\mathcal{B}(y)\land \bigvee\sigma\qty(x_0(f), y)).
\end{align*}
Thus, $\exists y_0(f)\in E_0^\mathcal{B}$, $f_0\land Pre^\mathcal{B}\circ y_0(f)\land \bigvee\sigma\qty(x_0(f), y_0(f))$. Since $\delta\circ Pre^\mathcal{B}\circ y_0(f)\le k$ and $\sigma\qty(x_0(f), y_0(f))\subset \mathcal{E}_k^P$, we have by Proposition \ref{ProRegBasCF} that $f_0\vDash Pre^\mathcal{B}\circ y_0(f)\land \bigvee\sigma\qty(x_0(f), y_0(f))$. By the Zig of $\mathcal{A}\leftrightarrow_G\mathcal{B}$,
\begin{align*}
f_0\vDash \bigvee\sigma\qty(x_0(f), y_0(f))\vDash\square_a \qty(Pre^\mathcal{A}\circ x_1(f)\to\bigvee_{y\in\to_a^\mathcal{B}\qty(y_0(f),\cdot)}Pre^\mathcal{B}(y)\land\bigvee\sigma\qty(x_1(f),y)).
\end{align*}
Thus,
\begin{align*}
\top\equiv\bigvee \mathcal{E}_{k-1}^P\vDash Pre^\mathcal{A}\circ x_1(f)\to\bigvee_{y\in\to_a^\mathcal{B}\qty(y_0(f),\cdot)}Pre^\mathcal{B}(y)\land\bigvee\sigma\qty(x_1(f),y).
\end{align*}
Therefore,
\begin{align}\label{EqEmuNPhardfvDashRight}
f\equiv Pre^\mathcal{A}\circ x_1(f)\vDash \bigvee_{y\in\to_a^\mathcal{B}\qty(y_0(f),\cdot)}Pre^\mathcal{B}(y).
\end{align}
By the Zag of $\mathcal{A}\leftrightarrows_G\mathcal{B}$, $\forall y\in\to_a^\mathcal{B}\qty(y_0(f),\cdot)$,
\begin{align*}
f_0\vDash \bigvee\sigma\qty(x_0(f), y_0(f))\vDash \square_a\qty(Pre^\mathcal{B}(y)\to Pre^\mathcal{A}\circ x_1(f)\land\bigvee\sigma\qty(x_1(f),y)).
\end{align*}
Thus,
\begin{align*}
\top\equiv\bigvee\mathcal{E}_{k-1}^P\vDash Pre^\mathcal{B}(y)\to Pre^\mathcal{A}\circ x_1(f)\land\bigvee\sigma\qty(x_1(f),y)),
\end{align*}
thereby
\begin{align*}
Pre^\mathcal{B}(y)\vDash Pre^\mathcal{A}\circ x_1(f)\equiv f.
\end{align*}
Since $y\in\to_a^\mathcal{B}\qty(y_0(f),\cdot)$ is arbitrary, we have by Eq. \eqref{EqEmuNPhardfvDashRight} that
\begin{align*}
f\equiv \bigvee_{y\in\to_a^\mathcal{B}\qty(y_0(f),\cdot)}Pre^\mathcal{B}(y).
\end{align*}
Since $f\in F$ is arbitary, we have that $Pre^\mathcal{B}\qty(E^\mathcal{B})$ is a formula basis of $F$ with $\abs{Pre^\mathcal{B}\qty(E^\mathcal{B})}=\abs{E^\mathcal{B}}\le K$.

In summary, we have transformed Problem \ref{BlemSet} to Problem \ref{BlemEmu} in polynomial time, thereby using polynomial addition space. Since Problem \ref{BlemSet} is PSPACE-complete by Proposition \ref{NPhardFormulaBasis}, we have that Problem \ref{BlemEmu} is PSPACE-hard.
\end{proof}

\section{Conclusion and Discussion}
\label{SecConc}

We study two problems in this paper. The first problem is to determine the action model equivalence. We propose the generalized action emulation and prove that it is sufficient and necessary for the action model equivalence. Previous strutural relationships including the bisimulation, the propositional action emulation, the action emulation, and the action emulation of canonical action models can all be described by restricting $\eta(\cdot,\cdot)$ in specific subsets of the modal language $\mathcal{L}$ in the generalized action emulation (see Remark \ref{RemResToGen}). Specifically, for the action emulation of canonical action models, $\eta(\cdot,\cdot)$ are the disjunctions of the subsets of $\Gamma\circ K\circ C\qty(\Phi)$, where $\Phi$ is the precondition set. We summarize four properties of the atom set which are critical for its ability to determine the action model equivalence. In fact, any formula set with these critical properties can be used to restrict the values of $\eta(\cdot,\cdot)$ in the generalized action emulation while keep its necessity for the action model equivalence (see Theorem \ref{TheGAEImpFourCriGAE}). We design an iteration algorithm such that if the initial finite formula set satisfies the four critical properties, then the algorithm returns in finite steps, and it returns true iff the two action models are equivalent. Interestingly, by carefully choosing the initial formula set, our iteration algorithm can also be used to determine the bisimulation, the propositional action emulation, and the action emulation. We also construct a new formula set with the four critical properties, which is generally more efficient for the iteration algorithm than the atom set.

The second problem is minimizing the event space of an action model under specific structural relationships. By applying the partition refinement algorithm \cite{EijckCWI2004} to the action model, one may easily minimizing the event space of an action model under the bisimulation. A more difficult problem is to minimizing the event space of under the propositional action emulation (remind that the propositional action emulation is weaker than the bisimulation, but still sufficient for the action model equivalence). We prove that this problem is PSPACE-complete, and propose a PSPACE algorithm for it. An even more difficult problem is to minimze the event space under the action model equivalence. By the technique of canonical formulas \cite{Moss2007JOPL}, we first prove that if the depth of the targeting action model is bounded by $k\ge 0$, then the search of the action model with the minimal event space can be restricted to those with depths no more than $k$ (see Theorem \ref{ThelkEquivTrans}). We further transform the problem to the search of a group of subsets of a finite target set satisfying certain properties (see Theorem \ref{TheCoverNonInterEquivCond}). Since the target set is finite, the search is at least computable. We also prove that minimzing the event space of an action model under the action model equivalence is actually PSPACE-hard.

Minizing the event space does not really simplify an action model in many cases. This is because the event formulas may become much more complex than the original ones. Finding more advanced concepts than minizing the event space is a valuable problem in future works.

%\centerline{\bf \bfseries  ---------------------------------------------------------}

\bibliographystyle{unsrt}
\bibliography{reference}

\end{document}